\documentclass{article}

\usepackage{microtype}
\usepackage{graphicx}
\usepackage{subfigure}
\usepackage{booktabs} 
\usepackage{algorithm,algorithmic}
\usepackage{amssymb, multirow, paralist, color,amsmath,amsthm}
\usepackage{textcomp}
\usepackage{siunitx}
\usepackage{mathtools}

\usepackage{hyperref}

\newcommand{\Norm}[1]{\left\|#1\right\|}
\newcommand*\tcircle[1]{%
  \raisebox{-0.5pt}{%
    \textcircled{\fontsize{7pt}{0}\fontfamily{phv}\selectfont #1}%
  }%
}
\def \S {\mathbf{S}}

\def \R {\mathbb{R}}

\def \w {\mathbf{w}}
\def \v {\mathbf{v}}

\def \x {\mathbf{x}}

\def \E {\mathrm{E}}

\def \x {\mathbf{x}}
\def \p {\mathbf{p}}
\def \a {\mathbf{a}}

\def \b {\mathbf{b}}

\def \1 {\mathbf{1}}

\def \z {\mathbf{z}}
\def \m {\mathbf{m}}

\def \s {\mathbf{s}}

\def \y {\mathbf{y}}
\def \u {\mathbf{u}}

\def \I {\mathbb{I}}
\def \P {\mathcal{P}}

\def \xh {\widehat{\x}}
\def \wh {\widehat{\w}}

\def \B {\mathcalB}

\def \y {\mathbf{y}}
\def \E {\mathrm{E}}
\def \x {\mathbf{x}}

\def \D {\mathcal{D}}
\def \z {\mathbf{z}}
\def \u {\mathbf{u}}

\def \w {\mathbf{w}}

\def \R {\mathbb{R}}
\def \S {\mathcal{S}}

\def \v {\mathbf{v}}

\def \p {\mathbf{p}}

\def \I {\mathbb{I}}

\def \hrho {\hat{\rho}}

\def \a {\mathbf{a}}
\def \b {\mathbf{b}}

\def \B {\mathcal{B}}

\def \wh {\widehat{\w}}

\def \s {\mathbf{s}}

\def \vh {\widehat{\v}}

\def \xh {\widehat{\x}}

\def \P {\mathbb{P}}

\newtheorem{thm}{Theorem}

\newtheorem{lemma}{Lemma}

\newtheorem{ass}{Assumption}

\def\inner#1#2{\left\langle #1, #2 \right\rangle}

\usepackage[accepted]{icml2022}

\icmltitlerunning{ Optimizing Partial AUC for Deep Learning with Convergence Guarantee}

\begin{document}

\twocolumn[
\icmltitle{When AUC meets DRO: Optimizing Partial AUC\\ for Deep Learning with Non-Convex Convergence Guarantee}



\icmlsetsymbol{equal}{*}

\begin{icmlauthorlist}
\icmlauthor{Dixian Zhu}{equal,cs}
\icmlauthor{Gang Li}{equal,cs}
\icmlauthor{Bokun Wang}{cs}
\icmlauthor{Xiaodong Wu}{ee}
\icmlauthor{Tianbao Yang}{cs}
\end{icmlauthorlist}

\icmlaffiliation{cs}{Department of Computer Science, University of Iowa, Iowa City, Iowa, USA}
\icmlaffiliation{ee}{Department of Electrical and Computer Engineering, University of Iowa, Iowa City, Iowa, USA}
\icmlcorrespondingauthor{Tianbao Yang}{tianbao-yang@uiowa.edu}
\icmlkeywords{Machine Learning, ICML}
\vskip 0.3in
]



\printAffiliationsAndNotice{\icmlEqualContribution} 

\begin{abstract}
In this paper, we propose systematic and efficient gradient-based methods for both one-way and two-way partial AUC (pAUC) maximization that are applicable to deep learning. We propose new formulations of pAUC surrogate objectives by using the distributionally robust optimization (DRO) to define the loss for each individual positive data. We consider two formulations of DRO, one of which is based on conditional-value-at-risk (CVaR) that yields a non-smooth but exact estimator for pAUC, and another one is based on a KL divergence regularized DRO that yields an inexact but smooth (soft) estimator for pAUC. For both one-way and two-way pAUC maximization, we propose two algorithms and prove their convergence for optimizing their two formulations, respectively. Experiments demonstrate the effectiveness of the proposed algorithms for pAUC maximization for deep learning on various datasets. The proposed methods are implemented with tutorials in
our open-sourced library LibAUC (\url{www.libauc.org}).
\end{abstract}

\section{Introduction}
\label{submission}
AUC, short for the area under the ROC curve, is a performance measure of a model, where the ROC curve is a curve of true positive rate (TPR) vs false positive rate (FPR) for all possible thresholds. AUC maximization in machine learning has a long history dating back to early 2000s~\cite{Herbrich1999d}. It has four ages in the twenty-years history, full-batch based methods in the first age, online methods in the second age, stochastic methods in the third age, and deep learning methods in the recent age. The first three ages focus on learning linear models or kernelized models. In each age, there have been seminal works in rigorous optimization algorithms that play important roles in the evolution of AUC maximization methods. Recent advances in non-convex optimization (in particular non-convex min-max optimization)~\cite{liu2019stochastic} has driven large-scale deep AUC maximization  to succeed in real-world tasks, e.g., medical image classification~\cite{DBLP:journals/corr/abs-2012-03173} and molecular properties prediction~\cite{wang2021advanced}. 

Nevertheless, the research on efficient optimization algorithms for partial AUC (pAUC) lag behind. In many applications,  there are large monetary costs due to high false positive rates (FPR) and low true positive rates (TPR), e.g., in medical diagnosis. Hence, a measure of primary interest is the region of the curve corresponding to low FPR and/or high TPR, i.e., pAUC.  There are two commonly used versions of pAUC, namely one-way pAUC (OPAUC)~\cite{dodd03} and two-way pAUC (TPAUC)~\cite{yang2019two}, where OPAUC puts a restriction on the range of FPR, i.e., FPR$\in[\alpha, \beta]$ (Figure~\ref{fig:auc} middle) and TPAUC puts a restriction on the lower bound of TPR and the upper bound of FPR, i.e., TPR$\geq \alpha$, FPR$\leq \beta$ (Figure~\ref{fig:auc} right). Compared with standard AUC maximization, pAUC maximization is more challenging since its estimator based on training examples involves selection of examples whose prediction scores are in certain ranks. 
\setlength{\abovedisplayskip}{4pt}
\setlength{\belowdisplayskip}{4pt}

To the best of our knowledge, there are few  rigorous and efficient algorithms developed for pAUC maximization for deep learning. Some earlier works have focused on pAUC maximization for learning linear models. For example, \citet{Narasimhan2017SupportVA} have proposed a structured SVM approach for one-way pAUC maximization, which is guaranteed to converge for optimizing the surrogate objective of pAUC. However, their approach is not efficient for big data and is not applicable to deep learning, which needs to evaluate the prediction scores of all examples and sort them at each iteration. There are some heuristic approaches, e.g., updating the model parameters according to the gradient of surrogate pAUC computed based on a mini-batch data~\cite{10.5555/2968826.2968904} or using an ad-hoc weighting function for each example for computing the stochastic gradient~\cite{pmlr-v139-yang21k}. However, such approaches are either not guaranteed to converge or could suffer a large approximation error. 

\begin{figure}[t]
\begin{center}
\centerline{\includegraphics[width=\columnwidth]{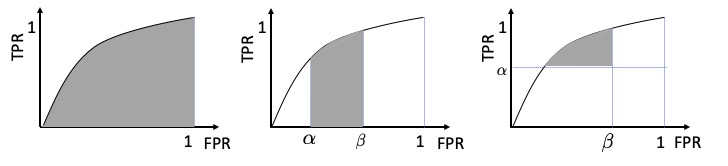}}

\caption{From left to right: AUC,  one-way pAUC, two-way pAUC}
\label{fig:auc}
\end{center}
\end{figure}

In this paper, we propose more systematic and rigorous optimization algorithms for pAUC maximization with convergence guarantee, which are applicable to deep learning. We consider both OPAUC maximization and TPAUC maximization, where for OPAUC we focus on maximizing pAUC in the region where $\text{FPR} \in[0,\beta]$ and for TPAUC we focus on maximizing pAUC in the region where $\text{FPR}\leq\beta$ and $\text{TPR}\geq\alpha$ for some $\alpha, \beta\in(0,1)$. In order to tackle the challenge of computing unbiased stochastic gradients of the surrogate objective of pAUC, we propose new formulations based on distributionally robust optimization (DRO), which allows us to formulate the problem into weakly convex optimization, and novel compositional optimization problems, and to develop efficient stochastic algorithms with convergence guarantee. We summarize our contributions below. 

\begin{itemize}
    \item
    For OPAUC maximization, for each positive example, we define a loss over all negative examples based on DRO. We consider two special formulations of DRO, with one based on the conditional-value-at-risk (CVaR) function that yields an exact estimator of the surrogate objective of OPAUC, and another one based on Kullback–Leibler (KL) divergence regularized DRO that yields a soft estimator of the surrogate objective.
    \item 
    We propose efficient stochastic algorithms for optimizing both formulations of OPAUC and establish their convergence guarantee and complexities for finding a (nearly) stationary solution. We also demonstrate that the algorithm for optimizing the soft estimator based on the KL divergence regularized DRO can enjoy parallel speed-up. 
    \item 
    For TPAUC maximization, we apply another level of DRO with respect to the positive examples on top of OPAUC formulations, yielding both exact and soft estimators for TPAUC. We also provide two rigorous stochastic algorithms with  provable convergence for optimizing both the exact and soft estimator of TPAUC, with the latter problem formulated as a novel three-level compositional stochastic optimization problem. 
    \item 
    We conduct extensive experiments for deep learning on image classification  and graph classification tasks with imbalanced data. We compare with heuristic and ad-hoc approaches for pAUC maximization and multiple baseline methods, and observe superior performance of the proposed algorithms. 
\end{itemize}

To the best our knowledge, this work is the first one that provides rigorous stochastic algorithms and convergence guarantee for pAUC maximization that are efficient and applicable to deep learning. We  expect the proposed novel formulations for OPAUC and TPAUC will  allow researchers to develop even faster algorithms than the proposed algorithms in this paper.

\section{Related Work}
In this section, we provide a brief overview of related work for pAUC maximization.

Earlier works have considered indirect methods for pAUC maximization~\cite{JMLR:v10:rudin09b,Agarwal2011TheIP,Rakotomamonjy2012SparseSV,Li2014TopRO,10.1145/1401890.1401980}. They did not directly optimize the surrogate objective of pAUC but instead some objectives that have some relationship to the right corner of ROC curve, e.g., p-norm push~\cite{JMLR:v10:rudin09b}, infinite-push~\cite{Agarwal2011TheIP,Rakotomamonjy2012SparseSV,Li2014TopRO}, and asymmetric SVM objective~\cite{10.1145/1401890.1401980}. Nevertheless, none of these studies propose algorithms that are scalable and applicable for deep learning. 

In~\cite{10.5555/2968826.2968904}, the authors proposed mini-batch based stochastic methods for pAUC maximization. At each iteration, a gradient estimator is simply computed based on the pAUC surrogate function of the mini-batch data. However, this heuristic approach is not guaranteed to converge for minimizing the pAUC objective and its error scales as $O(1/\sqrt{B})$, where $B$ is the mini-batch size. \citet{pmlr-v28-narasimhan13,10.1145/2487575.2487674,Narasimhan2017SupportVA} developed rigorous algorithms for optimizing pAUC with FPR restricted in a range $(\alpha, \beta)$ based on the structured SVM formulation.  However, their algorithms are only applicable to learning linear models and are not efficient for big data due to per-iteration costs proportional to the size of training data.  Recently, \citet{pmlr-v139-yang21k} considered optimizing two-way partial AUC with FPR less than $\beta$ and TPR larger than $\alpha$. Their paper focuses on simplifying the optimization problem that involves selection of top ranked negative examples and bottom ranked positive examples. They use ad-hoc weight functions for each positive and negative examples to relax the objective function into decomposable over pairs. {The weight function is designed such that the larger the scores of negative examples the higher are their weights, the smaller the scores of positive examples the higher are their weights.} Nevertheless, their objective function might have a large approximation error for the pAUC estimator. 

There are also some studies about partial AUC maximization without providing rigorous convergence guarantee on their methods, including greedy methods~\cite{Wang2011MarkerSV,Ricamato2011PartialAM} and boosting methods~\cite{Komori2010ABM,takashi12}. Some works also use pAUC maximization for learning non-linear neural networks~\cite{Ueda2018PartialAM,Iwata2020SemiSupervisedLF}. However, it is unclear how the optimization algorithms were designed as there were no discussion on the algorithm design and convergence analysis. Finally, it was brought to our attention that a recent work~\cite{paUCyao} also considered partial AUC maximization with a non-convex objective. The difference between this work and~\cite{paUCyao} is that: (i) they focus on optimizing one-way pAUC with FPR in a certain range $(\alpha, \beta)$ where $\alpha>0$; in contrast we consider optimizing both one-way pAUC and two-way pAUC, but for one-way pAUC we only consider FPR in a range of $(0, \beta)$; (ii) the second difference is that the proposed algorithms in this paper for one-way pAUC maximization has a better complexity than that established in~\cite{paUCyao}.

\section{Preliminaries}
In this section, we present some notations and preliminaries.  Let $\S =\{(\x_1, y_1), \ldots, (\x_n, y_n)\}$ denote a set of training data, where $\x_i$ represents an input training example (e.g., an image), and $\y_i\in\{1, -1\}$ denotes its corresponding label (e.g., the indicator of a certain disease). 
Let $h_\w(\x) = h(\w, \x)$ denote the score function of the neural network on an input data $\x$, where $\w\in \R^d$ denotes the parameters of the network. Denote by $\mathbb I(\cdot)$ an indicator function of a predicate, and by $[s]_+ = \max(s, 0)$.  For a set of given training examples $\S$, let $\S_+$ and $\S_-$ be the subsets of $\S$ with only positive  and negative examples, respectively, with $n_+=|\S_+|$ and $n_-=|\S_-|$. Let $\S^{\downarrow}[k_1,k_2]\subset\S$ be the subset of examples whose rank in terms of their prediction scores in the descending order are in the range of $[k_1, k_2]$, where $k_1\leq k_2$. Similarly, let $\S^{\uparrow}[k_1,k_2]\subset\S$ denote the subset of examples whose rank in terms of their prediction scores in the ascending order are in the range of $[k_1, k_2]$, where $k_1\leq k_2$.  
 We denote by $\E_{\x\sim\S}$  the average over $\x\in\S$. 
 Let $\x_+\sim\P_+$ denote a random positive example and $\x_-\sim\P_-$ denote a random negative example. We use $\Delta$ to denote a simplex of a proper dimension.

A function $F(\w)$ is weakly convex if there exists $C>0$ such that $F(\w) + \frac{C}{2}\|\w\|^2$ is a convex function. A function $F(\w)$ is $L$-smooth if its gradient is Lipchitz continuous, i.e., $\|\nabla F(\w) - \nabla F(\w')\|\leq L\|\w - \w'\|$. 

{\bf pAUC and its non-parametric estimator.}  For a given threshold $t$ and a score function $h(\cdot)$,  the TPR can be written as $\text{TPR}(t) = \Pr(h(\x)\geq t|y = 1)$, 
and the FPR  can be written as $\text{FPR}(t) = \Pr(h(\x)>t | y=-1)$. 
For a given $u\in[0,1]$, let $\text{FPR}^{-1}(u) = \inf\{t\in\R: \text{FPR}(t)\leq u\}$ and $\text{TPR}^{-1}(u) = \inf\{t\in\R: \text{TPR}(t)\leq u\}$.  The ROC curve defined as $\{u, \text{ROC}(u)\}$, where $u\in[0,1]$ and $\text{ROC}(u) = \text{TPR}(\text{FPR}^{-1}(u))$.   OPAUC (non-normalized) with FRP restricted in the range $(\alpha_0, \alpha_1)$ is equal to~\cite{dodd03} 
\begin{align}\label{eqn:paucd1}
    &\text{OPAUC}(h, \alpha_0, \alpha_1)  =\int^{\alpha_1}_{\alpha_0}\text{ROC}(u) du =\\
& \Pr(h(\x_+)> h(\x_-), h(\x_-)\in[\text{FPR}^{-1}(\alpha_1), \text{FPR}^{-1}(\alpha_0)])\notag,
\end{align}

where $h(\x_-)\in[\text{FPR}^{-1}(\alpha_1), \text{FPR}^{-1}(\alpha_0)]$ means that only negative examples whose prediction scores are in certain quantiles are considered. 
As a result, we have the following non-parametric estimator of OPAUC: \begin{align}\label{eqn:epaucd2}
&\widehat{\text{OPAUC}}(h, \alpha_0, \alpha_1) =\\
&\frac{1}{n_+}\frac{1}{n_-}\sum_{\x_i\in\S_+}\sum_{\x_j\in\S^\downarrow_-[k_1+1, k_2]}\I(h(\x_i)>h(\x_j)),\notag \end{align}
where $k_1=\lceil n_-\alpha_0\rceil, k_2 = \lfloor n_-\alpha_1 \rfloor$. In this work, we will focus on optimizing $\widehat{\text{OPAUC}}(h, 0, \beta)$ for some $\beta\in(0,1)$. 

Similarly, a non-parametric estimator of TPAUC with $\text{FPR}\leq\beta, \text{TPR}\geq  \alpha$ is given by 

\begin{align}\label{eqn:etpaucd2}
\widehat{\text{TPAUC}}&(h, \alpha, \beta) = \\
&\frac{1}{n_+}\frac{1}{n_-}\sum_{\x_i\in\S_+^{\uparrow}[1, k_1]}\sum_{\x_j\in\S^\downarrow_-[1, k_2]}\I(h(\x_i)>h(\x_j))\notag,
\end{align}
where  $k_1=\lfloor n_+\alpha\rfloor, k_2 = \lfloor n_-\beta \rfloor$.

{\bf Distributionally Robust Optimization (DRO).} For a set of random loss functions $\ell_1(\cdot), \ldots, \ell_n(\cdot)$, a DRO loss can be written as 
\begin{align}\label{eqn:dro}
\hat L_\phi(\cdot)= \max_{\p\in\Delta} \sum_j p_j \ell_j(\cdot)  - \lambda D_\phi(\p, 1/n),
\end{align}
where $D_\phi(\p, 1/n) = \frac{1}{n}\sum_i \phi(np_i)$ is a divergence measure, and $\lambda>0$ is a parameter.  The idea of the DRO loss is to assign an importance weight $p_i$ to each individual loss and take the uncertainty into account by maximization over $\p\in\Delta$ with a proper constraint/regularization on $\p$. In the literature,  several divergence measures have been considered~\cite{levy2020large}. In this paper, we will consider two special divergence measures that are of most interest for our purpose, i.e., the KL divergence $\phi_{kl}(t) = t\log t - t + 1$, which gives $D_\phi(\p, 1/n) =\sum_ip_i\log (np_i)$,  and the CVaR divergence $\phi_{c}(t) =\mathbb I(0<t\leq 1/\gamma)$ with a parameter $\gamma\in(0,1)$, which gives $D_\phi(\p, 1/n) = 0$ if $p_i\leq 1/(n\gamma)$ and infinity otherwise. 
The following lemma gives the closed form of $\hat L_\phi$ for $\phi_c$ and $\phi_{kl}$. 
\begin{lemma}\label{lem:0}By using KL divergence measure, we have
\begin{align}\label{eqn:kldro}
\hat L_{kl}(\cdot; \lambda)= \lambda \log \left(\frac{1}{n}\sum_{i=1}^n\exp\left(\frac{\ell_i(\cdot)}{\lambda}\right)\right).
\end{align}
By using the CVaR divergence $\phi_{c}(t)$ for some $\gamma$ such that $n\gamma$ is an integer, we have,
\begin{align}\label{eqn:cvar}
\hat L_{cvar}(\cdot; \gamma)= \frac{1}{ n\gamma}\sum_{i=1}^{n\gamma}\ell_{[i]}(\cdot),
\end{align}
where $\ell_{[i]}(\cdot)$ denotes the $i$-th largest value in $\{\ell_1, \ldots, \ell_n\}$.
\end{lemma}
The estimator in~(\ref{eqn:cvar}) is also known as the estimator of conditional-value-at-risk~\cite{rockafellar2000optimization}.


\section{AUC meets DRO for OPAUC Maximization}
Since the non-parametric estimator of OPAUC in~(\ref{eqn:epaucd2}) is non-continuous and non-differentiable, a continuous surrogate objective for OPAUC($h_\w, 0, \beta$) is usually defined by using a continuous pairwise surrogate loss $L(\w; \x_i, \x_j)= \ell(h_\w(\x_i) - h_\w(\x_j))$, resulting in the following  problem: 
\begin{align}\label{eqn:pemzeroone} 
\hspace*{-0.1in}\min_{\w} \frac{1}{n_+}\sum_{\x_i\in\S_+}\frac{1}{n_-\beta}\sum_{\x_j\in \S^\downarrow_-{[1,n_-\beta]}}L(\w; \x_i, \x_j),
\end{align}
where we assume $n_-\beta$ is an positive integer for simplicity of presentation. For the surrogate loss $\ell(\cdot)$, we assume it satisfies the following properties. 
\begin{ass}We assume $\ell(\cdot)$ is a convex, differentiable and monotonically decreasing function when $\ell(\cdot)>0$, and $\ell'(0)<0$.
\end{ass}
It is notable that the above condition is a sufficient condition to ensure that the surrogate $\ell(\cdot)$ is consistent for AUC maximization~\cite{10.5555/2832249.2832379}. There are many surrogate loss functions that have the above properties, e.g., squared hinge loss $\ell(s)=(c-s)_+^2$, logistic loss $\ell(s) = \log(1+\exp(-s/c))$ where $c>0$ is a parameter.


The challenge of optimizing a surrogate objective of pAUC in~(\ref{eqn:pemzeroone}) lies at tackling the selection of top ranked negative examples from $\S_-$, i.e., $\S_-^\downarrow[1, k]$ for some fixed $k$. It is impossible to compute  an unbiased stochastic gradient of the objective in~(\ref{eqn:pemzeroone}) based on a mini-batch of examples that include only a part of negative examples.

\subsection{AUC meets DRO for OPAUC}
To address the above challenge, we define new formulations for OPAUC maximization by leveraging the DRO. 
In particular, we define a robust loss for each positive data by 
\begin{align*}
&\hat L_\phi(\w; \x_i)=\max_{\p\in\Delta} \sum_{\x_j\in\S_-} p_jL(\w; \x_i, \x_j) - \lambda D_\phi(\p, 1/n_-).
\end{align*}
Then we define the following objective for OPAUC maximization:
\begin{align}\label{eqn:pauces}
\min_{\w}\frac{1}{n_+}\sum_{\x_i\in\S_+}\hat L_\phi(\w; \x_i).
\end{align}
When $\phi(\cdot)=\phi_c(\cdot)$, we refer to the above estimator (i.e., the objective function) as CVaR-based OPAUC estimator; and when $\phi(\cdot)= \phi_{kl}(\cdot)$, we refer to the above estimator as KLDRO-based OPAUC estimator. Below, we present two theorems to state the equivalent form of the objective, and the relationship between the two estimators and the surrogate objective in~(\ref{eqn:pemzeroone}) of OPAUC.

\begin{thm}\label{thm:cvar}
By choosing $\phi(\cdot)=\phi_c(\cdot)=\I(\cdot\in(0, 1/\beta])$, then the problem~(\ref{eqn:pauces}) is equivalent to 
\begin{align}\label{eqn:opauccvar} 
\hspace*{-0.1in}\min_{\w}\min_{\s\in\R^{n_+}} F(\w, \s)= \frac{1}{n_+}\sum_{\x_i\in\S_+} \left(s_i  +  \frac{1}{\beta} \psi_i(\w, s_i)\right),
\end{align}
where $\psi_i(\w, s_i) = \frac{1}{n_-} \sum_{\x_j\in \S_-}(L(\w; \x_i, \x_j) - s_i)_+$. If $\ell$ is a monotonically decreasing function for $\ell(\cdot)>0$, then the objective in~(\ref{eqn:pauces}) is equivalent to~(\ref{eqn:pemzeroone}) of OPAUC.
\end{thm}
{\bf Remark:} The above theorem indicates that CVaR-based OPAUC estimator is an exact estimator of OPAUC, which is consistent for OPAUC maximization. The variable $s_i$ can be considered as the threshold variable to select the top-ranked negative examples for each positive data.

\begin{thm}\label{thm:klform}
By choosing $\phi(\cdot)=\phi_{kl}(\cdot)$,  then the problem~(\ref{eqn:pauces}) becomes 
\begin{align}\label{eqn:pauceskl}
\min_{\w}\frac{1}{n_+}\sum_{\x_i\sim\S_+}\lambda\log \E_{\x_j\in\S_-}\exp(\frac{L(\w; \x_i, \x_j)}{\lambda}).
\end{align}
If $\ell(\cdot)$ is a monotonically decreasing function for $\ell(\cdot)>0$, when $\lambda=0$, the above objective is a surrogate of  $\widehat{\text{OPAUC}}(h_\w, 0, \frac{1}{n_-}$); and when $\lambda=+\infty$, the above objective is a surrogate of  $\widehat{\text{OPAUC}}(h_\w, 0, 1$), i.e., the AUC. 
\end{thm}

{\bf Remark:} Theorem~\ref{thm:klform} indicates that KLDRO-based OPAUC estimator is a soft estimator, which interpolates between OPAUC($h_\w, 0, 1/n_-$) and OPAUC($h_\w, 0, 1$) by varying $\lambda$. 

It is also notable that  when $\beta=1/n_-$ in CVaR-based estimator, the objective in~(\ref{eqn:pauces}) becomes the  infinite-push (or top-push) objective considered in the literature~\cite{Agarwal2011TheIP,Rakotomamonjy2012SparseSV,Li2014TopRO}, and hence our algorithm for solving~(\ref{eqn:opauccvar}) can be also used for solving the infinite-push objective for deep learning. In contrast, the previous works for the infinite-push objective focus on learning linear models.   Similarly, when $\lambda=0$ in KLDRO-based estimator, the objective in~(\ref{eqn:pauceskl}) becomes the infinite-push objective. Nevertheless, our algorithm for optimizing KLDRO-based estimator is not exactly applicable to optimizing the infinite-push objective as we focus on the cases $\lambda>0$, which yields a smooth objective function under proper conditions of $\ell(\cdot)$ and $h(\cdot; \x)$. As a result, we could have stronger convergence by optimizing the KLDRO-based estimator as indicated by our convergence results in next subsection.  

\subsection{Optimization Algorithms and Convergence Results}\label{sec:alg}
In this subsection, we present the optimization algorithms for solving both~(\ref{eqn:opauccvar}) and~(\ref{eqn:pauceskl}), and then present their convergence results for finding a nearly stationary solution. The key to our development is to formulate the two optimization problems into known non-convex optimization problems that have been studied in the literature, and then to develop stochastic algorithms by borrowing the existing techniques. 

{\bf Optimizing CVaR-based estimator.} We first consider  optimizing the CVaR-based estimator, which is equivalent to~(\ref{eqn:opauccvar}). A benefit for solving ~(\ref{eqn:opauccvar}) is that an unbiased stochastic subgradient can be computed in terms of $(\w, \s)$. However, this problem is still challenging because the objective function $F(\w, \s)$ is non-smooth non-convex. In order to develop a stochastic algorithm with convergence guarantee,  we prove that $F(\w, \s)$ is weakly convex in terms of $(\w, \s)$, which allows us to borrow the techniques of optimizing weakly convex function~\cite{sgdweakly18} for solving our problem and to establish the convergence.  We first establish the weak convexity of $F(\w, \s)$. 
\begin{lemma}\label{lem:weak}
If $L(\cdot; \x_i, \x_j)$ is a $L_s$-smooth function for any $\x_i, \x_j$, then $F(\w, \s)$ is $\rho$-weakly convex with $\rho=L_s/\beta$. 
\end{lemma}

Another challenge for optimizing $F(\w, \s)$ is that $\s$ is of high dimensionality and computing the gradient for all entries in $\s$ at each iteration is expensive. Therefore, we develop a tailored stochastic algorithm for solving~(\ref{eqn:opauccvar}), which is  shown in Algorithm~\ref{alg:SOPA}. This algorithm uses stochastic gradient descent (SGD) updates for updating $\w$ and stochastic coordinate gradient descent (SCGD) updates for updating $\s$. We refer to the algorithm as SOPA.  A key feature of SOPA is that the stochastic gradient estimator for $\w$ is a weighted average gradient of the pairwise losses for all pairs in the mini-batch, i.e., step 6. The hard weights $p_{ij}$ (either 0 or 1) are dynamically computed by step 4, which compares the pairwise loss $(\ell(h(\w_t, \x_i) - h(\w_t, \x_j))$ with the threshold variable $s_i^t$, which is also updated by a SGD step. 


\begin{algorithm}[t]{\hspace*{-0.5in}}
    \centering
    \caption{SOPA}
    \label{alg:SOPA}
    \begin{algorithmic}[1]  
    \STATE Set $\s^1=0$ and initialize $\w$
    \FOR {$t = 1,\ldots, T$}
    \STATE Sample two mini-batches $\B_+\subset\S_+, \B_-\subset\S_-$ 
   
                   \STATE Let $p_{ij} =\I(\ell(h(\w_t, \x_i) - h(\w_t, \x_j)) - s^t_i> 0)$
          \STATE Update $s^{t+1}_i   =s^t_i - \frac{\eta_2}{n_+} (1  - \frac{\sum_j p_{ij}}{\beta |\B_-|} ) $ for $\x_i\in\B_+$
               \STATE Compute a gradient estimator $\nabla_t$ by 
               $$ \hspace*{-0.2in}\nabla_t =  \frac{1}{\beta |\B_+||\B_-|}\sum_{\x_i\in\B_+}   \sum_{\x_j\in \B_-}p_{ij}\nabla_\w L(\w_t; \x_i, \x_j)$$
     \STATE Update $\w_{t+1}   =\w_t - \eta_1  \nabla_t$
    \ENDFOR
    \end{algorithmic}
\end{algorithm}

{\bf Optimizing KLDRO-based estimator of OPAUC.} Next, we consider optimizing the KLDRO-based estimator, which is equivalent to~(\ref{eqn:pauceskl}). A nice property of the objective function is that it is smooth under a proper condition as stated in Assumption~\ref{ass:2}. However, the challenge for solving~(\ref{eqn:pauceskl}) is that  an unbiased stochastic gradient is not readily computed. To highlight the issue, the problem (\ref{eqn:pauceskl}) can be written as 
\begin{align}\label{eqn:op2}
    \min_{\w}F(\w) = \frac{1}{n_+}\sum\nolimits_{\x_i\in\S_+}f(g_i(\w)),
\end{align}
where  $g_i(\w)=\E_{\x_j\sim\S_-}\exp(\frac{L(\w; x_i,\x_j)}{\lambda})$ and $f(\cdot)= \lambda \log (\cdot)$.  A similar optimization problem has been studied in~\cite{DBLP:journals/corr/abs-2104-08736} for maximizing average precision,  which is referred to as finite-sum coupled compositional stochastic optimization, where $f(g_i(\w))$ is a compositional function and $g_i(\w)$ depends on the index $i$ for the outer summation. A full gradient of  $f(g_i(\w))$ is given by $f'(g_i(\w))\nabla g_i(\w)$. With a mini-batch of samples, $g_i(\w)$ can be estimated by an unbiased estimator $\hat g_i(\w)$. However, $f'(\hat g_i(\w))\nabla \hat g_i(\w)$ is a biased estimator due to the compositional form. To address this challenge, \citet{DBLP:journals/corr/abs-2104-08736} proposed a novel stochastic algorithm that maintains a moving average estimator for $g_i(\w)$ denoted by $u_i$. Recently, \citet{wangsox} has also studied the finite-sum coupled compositional optimization problem comprehensively and proposed a similar algorithm (SOX) and derived better convergence results than that in~\cite{DBLP:journals/corr/abs-2104-08736}. Hence,  we employ the same algorithm in~\cite{wangsox} for solving~(\ref{eqn:pauceskl}), which is in shown in Algorithm~\ref{alg:sopa-s} and is referred to as SOPA-s.

There are two key differences between SOPA-s and SOPA. First, the pairwise weights $p_{ij}$ in SOPA-s (step 5) are soft weights between 0 and 1, in contrast to the hard weights $p_{ij}\in\{0,1\}$ in SOPA. Second, the update for $\w_{t+1}$ is a momentum-based update where $\gamma_1\in(0,1)$. We can also use an Adam-style update, which shares similar convergence as the momentum-based update~\cite{guo2022stochastic}. 

\begin{algorithm}[t]{\hspace*{-0.5in}}
    \centering
    \caption{SOPA-s}
    \label{alg:sopa-s}
    \begin{algorithmic}[1]  
    \STATE Set $\u^1=0$ and initialize $\w$
    \FOR {$t = 1,\ldots, T$}
    \STATE Sample two mini-batches $\B_{+}\subset\S_+, \B_-\subset\S_-$ 
        \STATE For each $\x_i\in\B_{+}$, update $u^{t+1}_i =(1-\gamma_0)u^t_{i} + \gamma_0 \frac{1}{|\B_-|}  \sum_{\x_j\in \B_-}\exp\left(\frac{L(\w_t; \x_i, \x_j)}{\lambda}\right) $
                             \STATE Let $p_{ij} = \exp (L(\w_t; \x_i, \x_j)/\lambda)/u^{t}_{i}$
                        \STATE Compute a gradient estimator $\nabla_t$ by 
          \[
         \hspace*{-0.2in}\nabla_t=\frac{1}{|\B_{+}|}\frac{1}{|\B_-|}\sum_{\x_i\in\B_{+}}   \sum_{\x_j\in \B_-}p_{ij}\nabla L(\w_t; \x_i, \x_j)\]
     \STATE Update $\v_{t}=(1-\gamma_1)\v_{t-1} + \gamma_1 \nabla_t$
     \STATE Update $\w_{t+1}   =\w_t - \eta  \v_t$ (or Adam-style)
    \ENDFOR
    \end{algorithmic}
\end{algorithm}
\subsection{Convergence Analysis}
For convergence analysis, we make the following assumption about $h$ and $\ell(\cdot)$. 
\vspace*{-0.03in}\begin{ass}\label{ass:2}
Assume $h(\cdot; \x)$ is Lipschitz continuous, smooth and bounded, $\ell(\cdot)$ is a smooth function and has a bounded gradient for a bounded argument. 
\end{ass}
A bounded smooth score function $h(\cdot; \x)$ is ensured if the activation function of the neural network is smooth and the output layer uses a bounded and smooth  activation function. For example, let $\hat h(\w; \x)$ denote the plain output of the neural network, then the score function $h(\w; \x) = 1/(1+\exp(-\hat h(\w; \x))$ is bounded and smooth. The Lipschitz continuity of $h(\w; \x)$ can be guaranteed if $\w$ is bounded. 

We first consider the analysis of SOPA. Since $F(\w, \s)$ is non-smooth, for presenting the convergence result, {we need to introduce a convergence measure based on the Moreau envelope of $F(\w, \s)$ given below for some $\hrho>\rho$}: 
\begin{align*}
    F_{\hrho}(\w, \s) = \min\nolimits_{\w, \s} F(\w, \s) + \frac{\hrho}{2}(\|\w\|^2 +\|\s\|^2). 
\end{align*}
It is guaranteed that $F_{\hrho}(\w, \s)$ is a smooth function~\cite{Drusvyatskiy2019EfficiencyOM}. A point $(\w, \s)$ is called an $\epsilon$-nearly stationary solution to $F(\w, \s)$ if $\|\nabla F_{\hrho}(\w, \s)\|\leq \epsilon $ for some $\hrho>\rho$, where $\rho$ is the weak convexity parameter of $F$.  This convergence measure has been widely used for weakly convex optimization problems~\cite{sgdweakly18,rafique2018non,chen2018universal}. Then we establish the following convergence guarantee for SOPA. 

\begin{thm}\label{thm:sopa}
Under Assumption~\ref{ass:2},  Algorithm 1 ensures that after $T=O(1/(\beta\epsilon^4))$ iterations we can find an $\epsilon$ nearly stationary solution of $F(\w, \s)$, i.e., $\E\|\nabla F_{\hrho}(\w_\tau, \s_\tau)\|^2\leq \epsilon^2$ for a randomly selected $\tau\in\{1, \ldots, T\}$ and $\hrho=1.5\rho$. 
\end{thm}
Next, we establish the convergence of SOPA-s. Under Assumption~\ref{ass:2}, we can show that $F(\w)$ in~(\ref{eqn:op2}) is smooth. Hence, we use the standard convergence measure in terms of gradient norm of $F(\w)$. 
\begin{thm}\label{thm:2}
Under Assumption~\ref{ass:2}, Algorithm 2 with $\gamma_0=O(B_-\epsilon^2)$, $\gamma_1=O(\min\{B_-, B_+\}\epsilon^2)$, $\eta=O(\min\{\gamma_0B_1/n_+,\gamma_1\})$ ensures that after $T=O(\frac{1}{\min(B_+, B_-)\epsilon^4}+\frac{n_+}{B_+B_-\epsilon^4})$ iterations we can find an $\epsilon$-stationary solution of $F(\w)$, i.e., $\E[\|\nabla F(\w_\tau)\|^2]\leq \epsilon^2$ for a randomly selected $\tau\in\{1,\ldots,T\}$, where $B_+=|\B_{+}|$ and $B_-=|\B_{-}|$. 
\end{thm}
{\bf Remark:} The convergence analysis of Algorithm 2 follows directly from that in~\cite{wangsox}.  
Compared with that in Theorem~\ref{thm:sopa} for SOPA, the convergence of SOPA-s is stronger than that of SOPA in several aspects: (i) the convergence measure of SOPA-s is stronger than that of SOPA due to that  Theorem~\ref{thm:2} guarantees the convergence in terms of  gradient norm of the objective, while Theorem~\ref{thm:sopa} guarantees  the convergence on a weaker convergence measure namely the gradient norm of the Moreau envelope of the objective; (ii) the complexity of SOPA-s enjoys a parallel speed-up by using a mini-batch of data. However, it is also notable that the complexity of SOPA does not depend on the number of positive data as that of SOPA-s.

\section{AUC meets DRO for TPAUC Maximization}
In this section, we propose estimators for the surrogate objective of TPAUC and stochastic algorithms with convergence guarantee for optimizing  the estimators. 
To this end, we apply another level of DRO on top of $\hat L_\phi(\x_i, \w), \x_i\in\S_+$ and define the following estimator of TPAUC:
\begin{align*}
F(\w; \phi, \phi')=\max_{\p\in\Delta}\sum_{\x_i\in\S_+} p_i \hat L_\phi(\x_i, \w)  -  \lambda' D_{\phi'}(\p, \frac{1}{n_+}).
\end{align*}


Next, we focus on optimizing the soft estimator of TPAUC defined by using $\phi=\phi'=\phi_{kl}$. First, we have the following form for the estimator. 
\vspace*{-0.05in}
\begin{lemma}
When $\phi=\phi'=\phi_{kl}$, we have 
\begin{align*}
&F(\w; \phi_{kl}, \phi_{kl})\\
&=  \lambda'\log \E_{\x_i\sim\S_+}\left(\E_{\x_j\sim\S_-}\exp(\frac{\ell(\w; \x_i,\x_j)}{\lambda})\right)^{\frac{\lambda}{\lambda'}}. 
\end{align*}
\end{lemma}
\vspace*{-0.1in}For minimizing this function, we formulate the problem as a novel three-level compositional stochastic optimization:  
\begin{align*}
\min_{\w} f_1 (\frac{1}{n_+}\sum_{\x_i\in\S_+}f_2(g_i(\w))),
\end{align*}
where $f_1(s) = \lambda'\log (s)$, $f_2(g) = g^{\lambda/\lambda'}$ and $g_i(\w) = \E_{\x_j\sim\S_-}\exp(L(\w; \x_i, \x_j)/\lambda) $. We propose a novel stochastic algorithm for solving the above problem, which is shown in Algorithm~\ref{alg:sota-s}, to which we refer as SOTA-s. Note that the problem is similar to  multi-level compositional optimization~\cite{balasubramanian2020stochastic} but also has subtle difference.  The function inside $f_1$ has a form similar to that in~(\ref{eqn:op2}). Hence, we use similar technique to SOPA-s by maintaining and updating $u^i$ to track $g_i(\w)$ in step 4. Besides, we need to maintain and update $v_{t+1}$ to track $\frac{1}{n_+}\sum_{\x_i\in\S_+}f_2(g_i(\w_t))$ in step 5. Then the gradient estimator in step 7 is computed by $\nabla f_1(v_{t+1})\frac{1}{|\B_+|}\sum_{\x_i\in\B_+}\nabla 
\hat g_i(\w_t)\nabla f_2(u^i_{t})$, where $\hat g_i(\w)=\E_{\x_j\sim\B_-}\exp(L(\w_t; \x_i, \x_j)/\lambda)$. Then we update the model parameter by the momentum-style or Adam-style update.

\begin{algorithm}[t]{\hspace*{-0.5in}}
    \centering
    \caption{SOTA-s}
    \label{alg:sota-s}
    \begin{algorithmic}[1]  
    \STATE Set $\u_0=0, v_0=0, \m_0=0$ and initialize $\w$
    \FOR {$t = 1,\ldots, T$}
    \STATE  Sample two mini-batches  $\B_{+}\subset\S_+$, $\B_-\subset\S_-$ 
     \STATE For each $\x_i\in\B_{+}$ compute $u^i_{t} =(1-\gamma_0)u^i_{t-1} + \gamma_0 \frac{1}{|B_-|}  \sum_{\x_j\in \B_-}L(\w_t; \x_i, \x_j) $
     \STATE Let $v_{t} = (1-\gamma_1)v_{t-1} + \gamma_1\frac{1}{|\B_{+}|}\sum_{\x_i\in \B_{+}} f_2(u^i_{t-1})$
         \STATE Let $p_{ij} = (u^i_{t-1})^{\lambda/\lambda' - 1}\exp (L(\w_t, \x_i, \x_j)/\lambda)/v_{t}$
         \STATE Compute a gradient estimator $\nabla_t$ by
          \[
         \hspace*{-0.15in}\nabla_t=\frac{1}{|\B_{+}}\frac{1}{|\B_-|}\sum_{\x_i\in\B_{+}}   \sum_{\x_j\in \B_-}p_{ij}\nabla L(\w_t; \x_i, \x_j)\]
     \STATE Update $\m_{t}=(1-\gamma_2)\m_{t-1} + \gamma_2 \nabla_t$
     \STATE Update $\w_{t+1}   =\w_t - \eta  \m_t$   (or Adam-style)
     \ENDFOR
    \end{algorithmic}
\end{algorithm}
\begin{thm}\label{thm:alg3}
Under Assumption~\ref{ass:2}, Algorithm 3 with $\gamma_0=O(B_-\epsilon^2)$, $\gamma_1=O(B_+\epsilon^2)$, $\gamma_2=O(\min\{B_-, B_+\}\epsilon^2)$, $\eta=O(\min\{\gamma_0B_1/n_+,\gamma_1,\gamma_2\})$ ensures that after $T=O(\frac{1}{\min(B_+, B_-)\epsilon^4} + \frac{n_+}{B_+B_-\epsilon^4})$ iterations we can find an $\epsilon$ nearly stationary solution of $F(\w)$, where $B_+=|\B_{+}|$ and $B_-=|\B_-|$. 
\end{thm}
\vspace*{-0.1in}
{\bf Remark:} It is notable that SOTA-s has an iteration complexity in the same order of SOPA-s for OPAUC maximization. 

Finally, we discuss how to optimize  the exact estimator of TPAUC defined by $F(\w; \phi_c, \phi_c')$, 
where $\phi_c(t)=\I(0\leq t\leq 1/\beta)$ and $\phi'_c(t)=\I(0\leq t\leq 1/\alpha)$ with $K_2=n_-\beta$ and $K_1=n_+\alpha$ being integers. Lemma~\ref{lem:tpauc} in the supplement shows that  if $\ell(\cdot)$ is monotonically decreasing for $\ell(\cdot)>0$
\begin{align*}
F(\w; \phi_c,& \phi_c')=\\
&\frac{1}{K_1K_2}\sum_{\x_i\in\S^\uparrow_+[1,K_1]}\sum_{\x_j\in\S^\downarrow_-[1,K_2]}L(\w; \x_i, \x_j),
\end{align*}
is a consistent surrogate function of TPAUC for $\text{TPR}\geq \alpha$ and $\text{FPR}\leq \beta$ in view of the estimator $\widehat{\text{TPAUC}}$ given in~(\ref{eqn:etpaucd2}).

Similar to Theorem~\ref{thm:cvar}, we can show that $F(\w; \phi_c, \phi_c')$ is equivalent to: 
\begin{align*}
    &\min_{s'\in\R, \s\in\R^{n_+}}s' +\frac{1}{n_+\alpha}\sum_{\x_i\in\S_+}(s_i+\frac{1}{\beta}\psi_i(\w; s_i) - s')_+.
\end{align*}
Like $F(\w, \s)$ in~(\ref{eqn:opauccvar}), we can prove the inner function is weakly convex in terms of $(\w, \s, s')$. However, computing an unbiased stochastic gradient in terms of $\w$ and $s_i$ is also impossible due to that  $\psi_i(\w; s_i)$ is inside a hinge function. To solve this problem, we can use the conjugate form of the hinge function to convert the minimization of $F(\w; \phi_c, \phi'_c)$ into a weakly-convex concave min-max  problem~\cite{rafique2018non}  and we can develop a stochastic algorithm but only with $O(1/\epsilon^6)$ iteration complexity for finding a nearly stationary solution. 
We present the algorithm and analysis in the supplement for interested readers. 
\begin{figure*}[h]
\centering
    \includegraphics[width=0.24\linewidth]{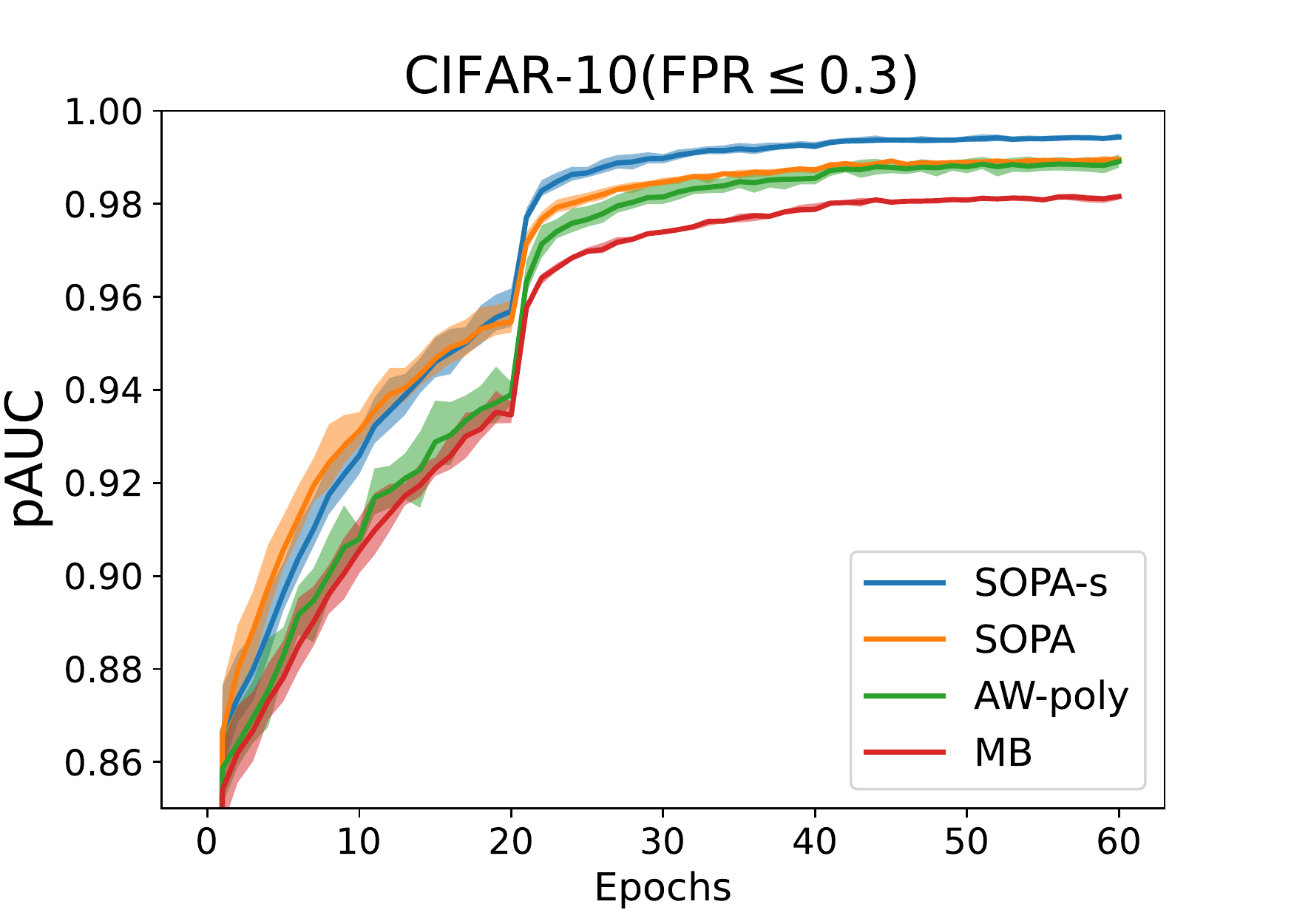}
    \includegraphics[width=0.24\linewidth]{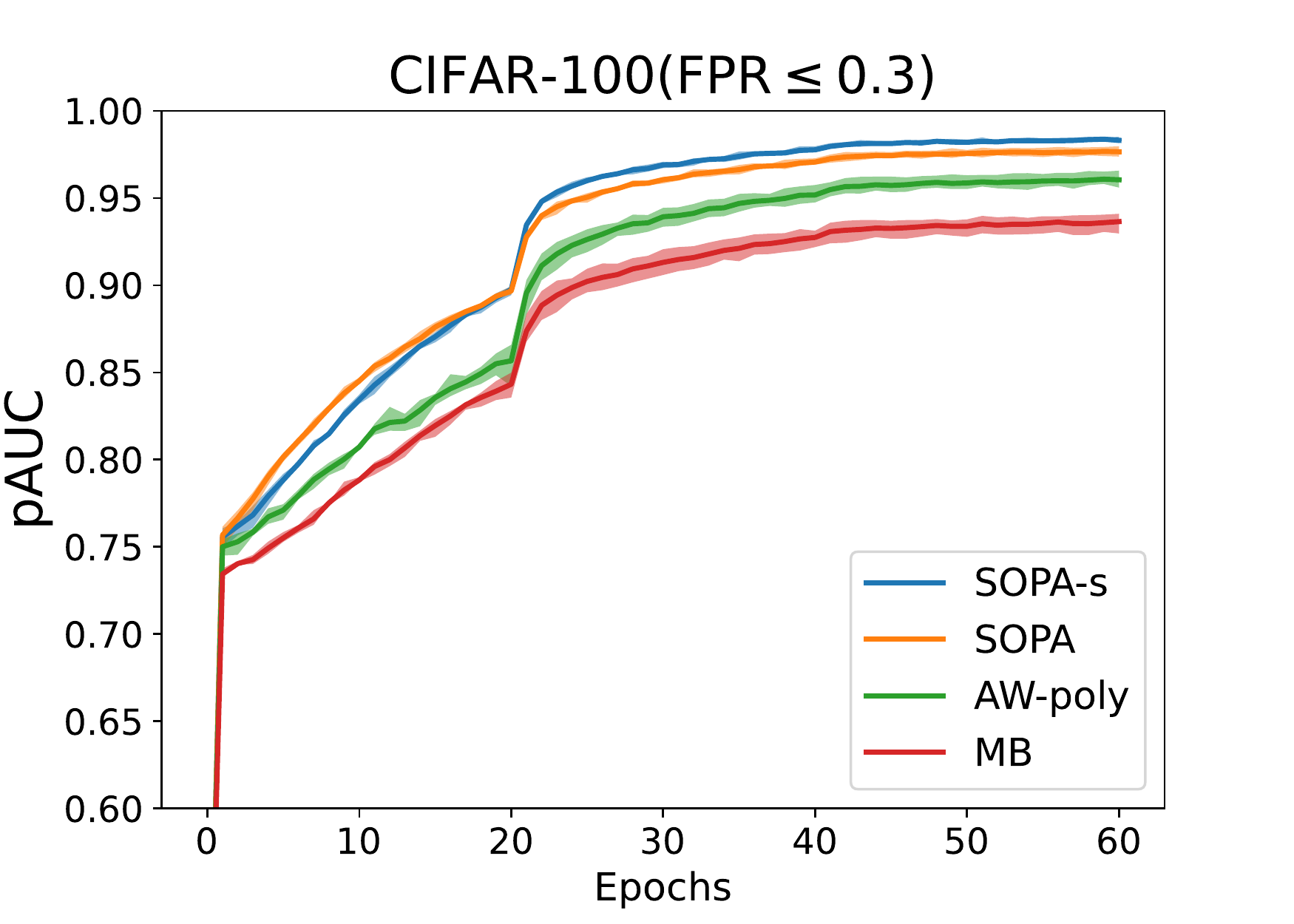}
    \includegraphics[width=0.24\linewidth]{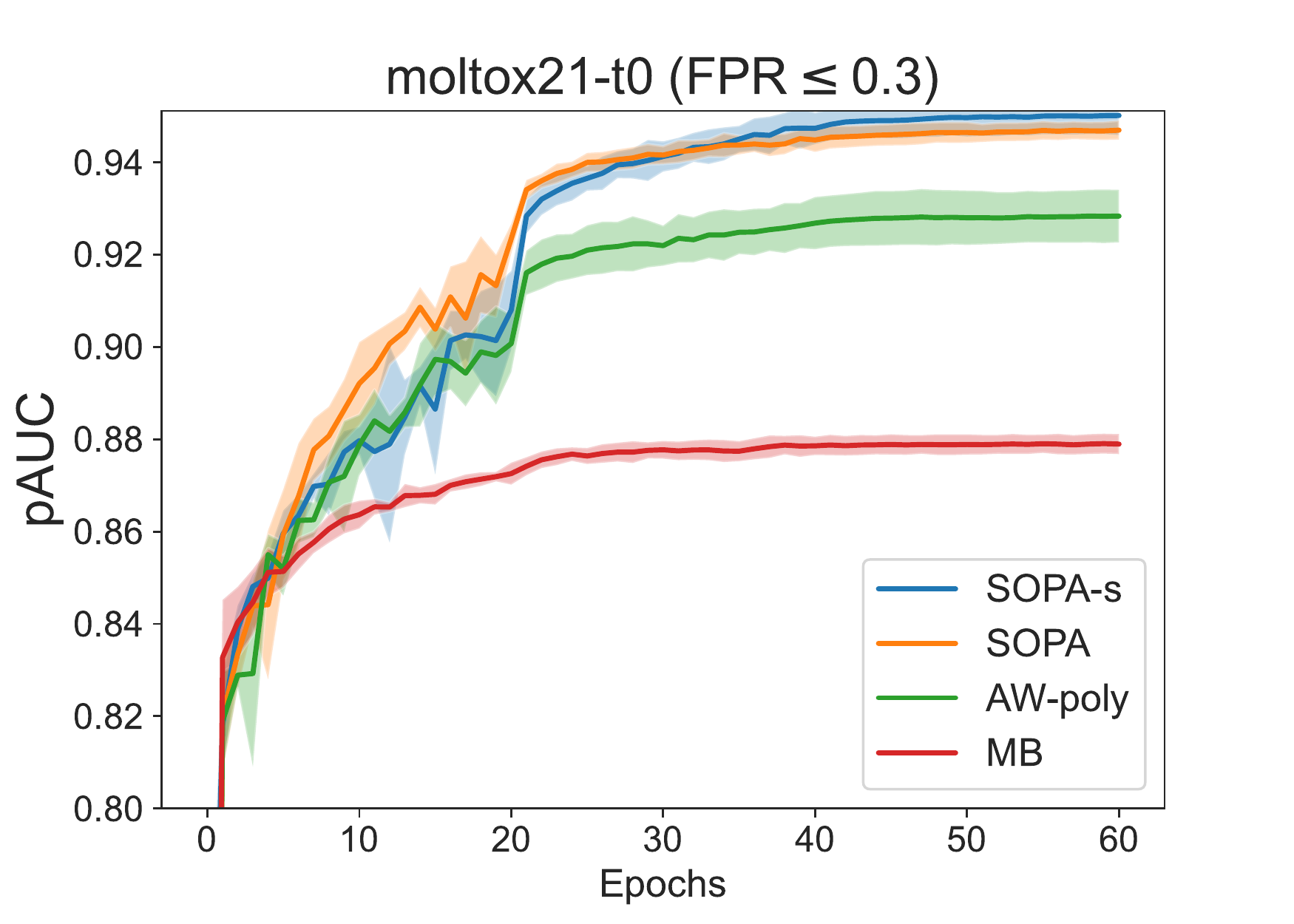}
     \includegraphics[width=0.24\textwidth]{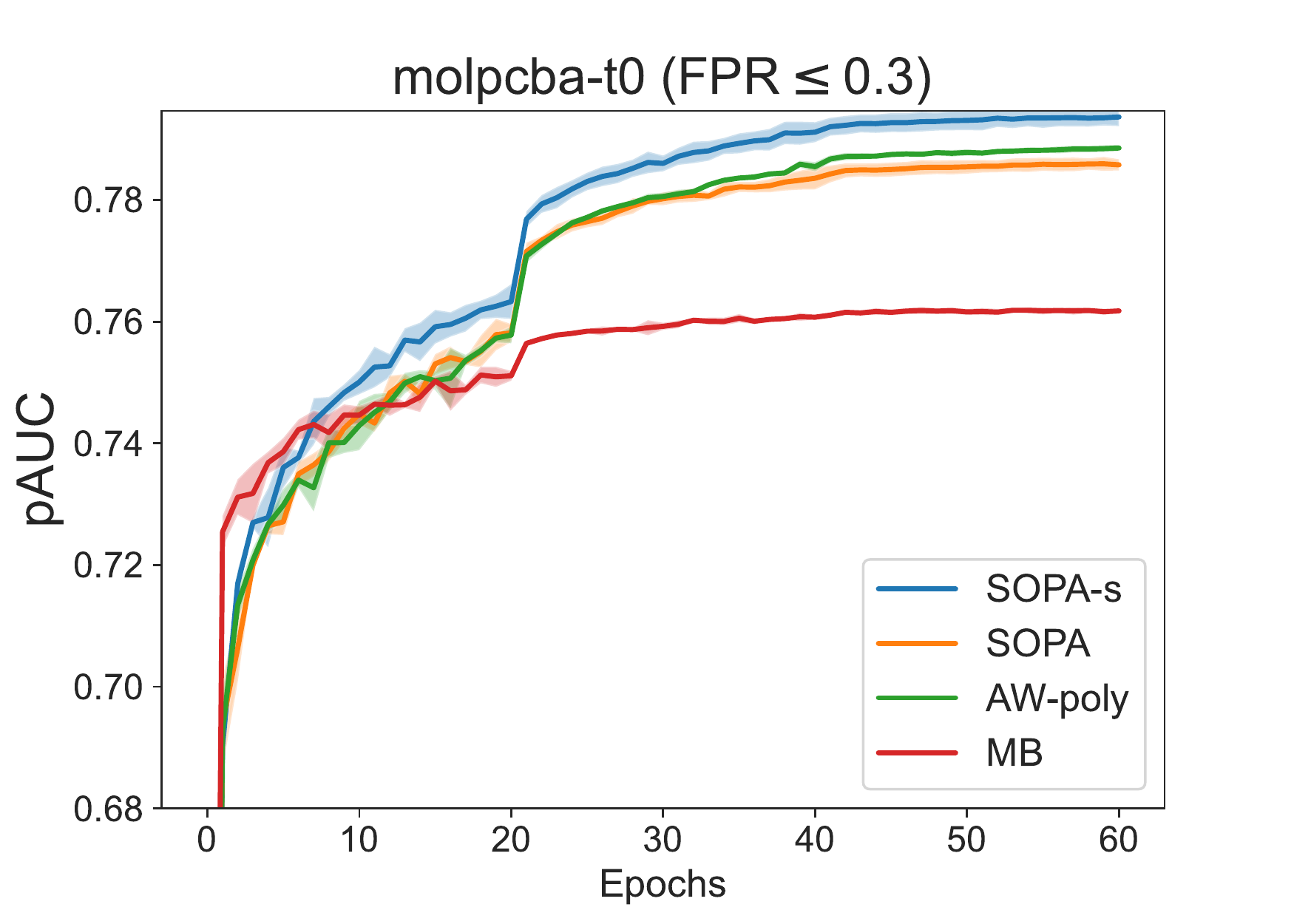}
         \label{fig:molpcba-03}
    
   \includegraphics[width=0.24\linewidth]{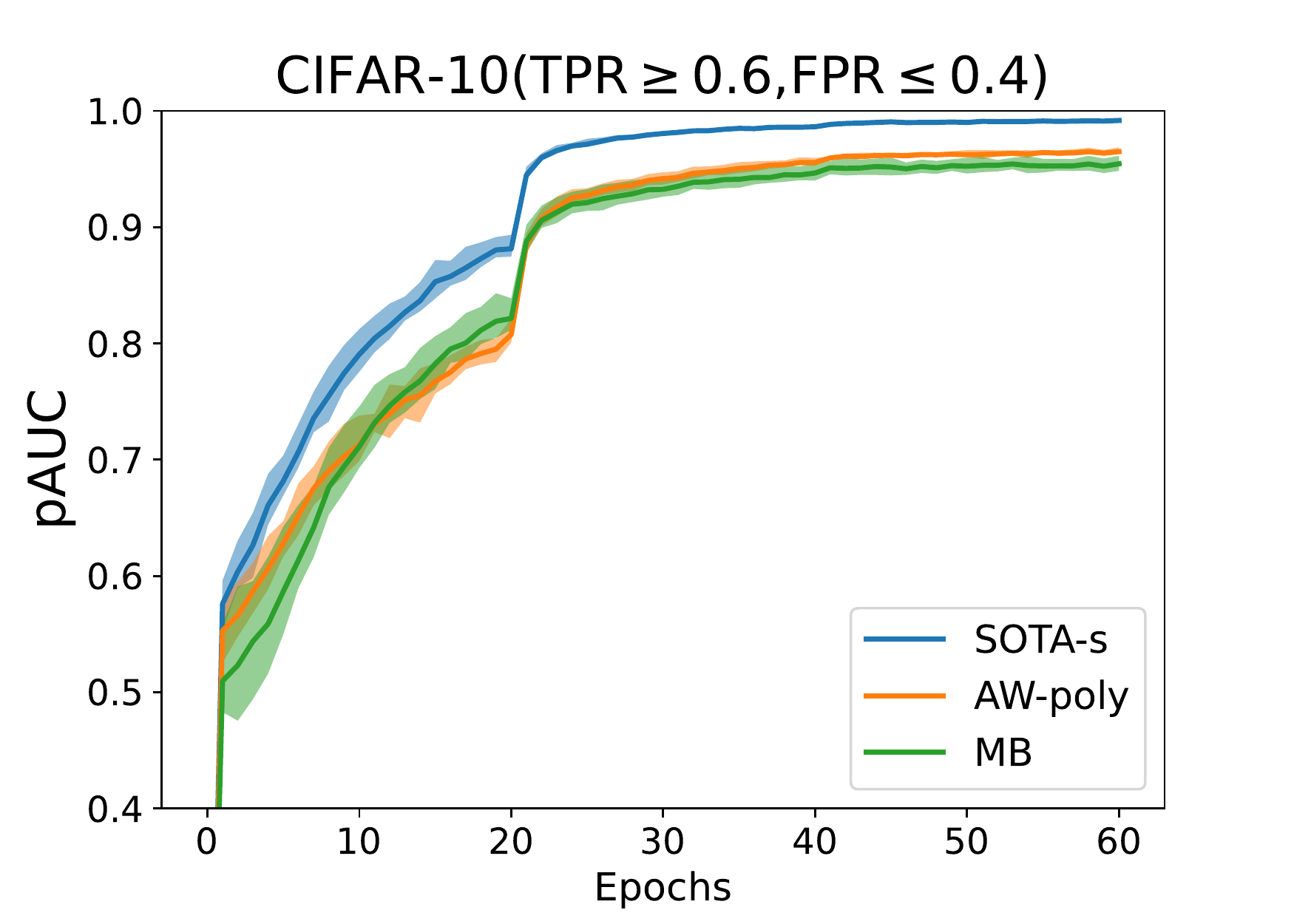}
    \includegraphics[width=0.24\linewidth]{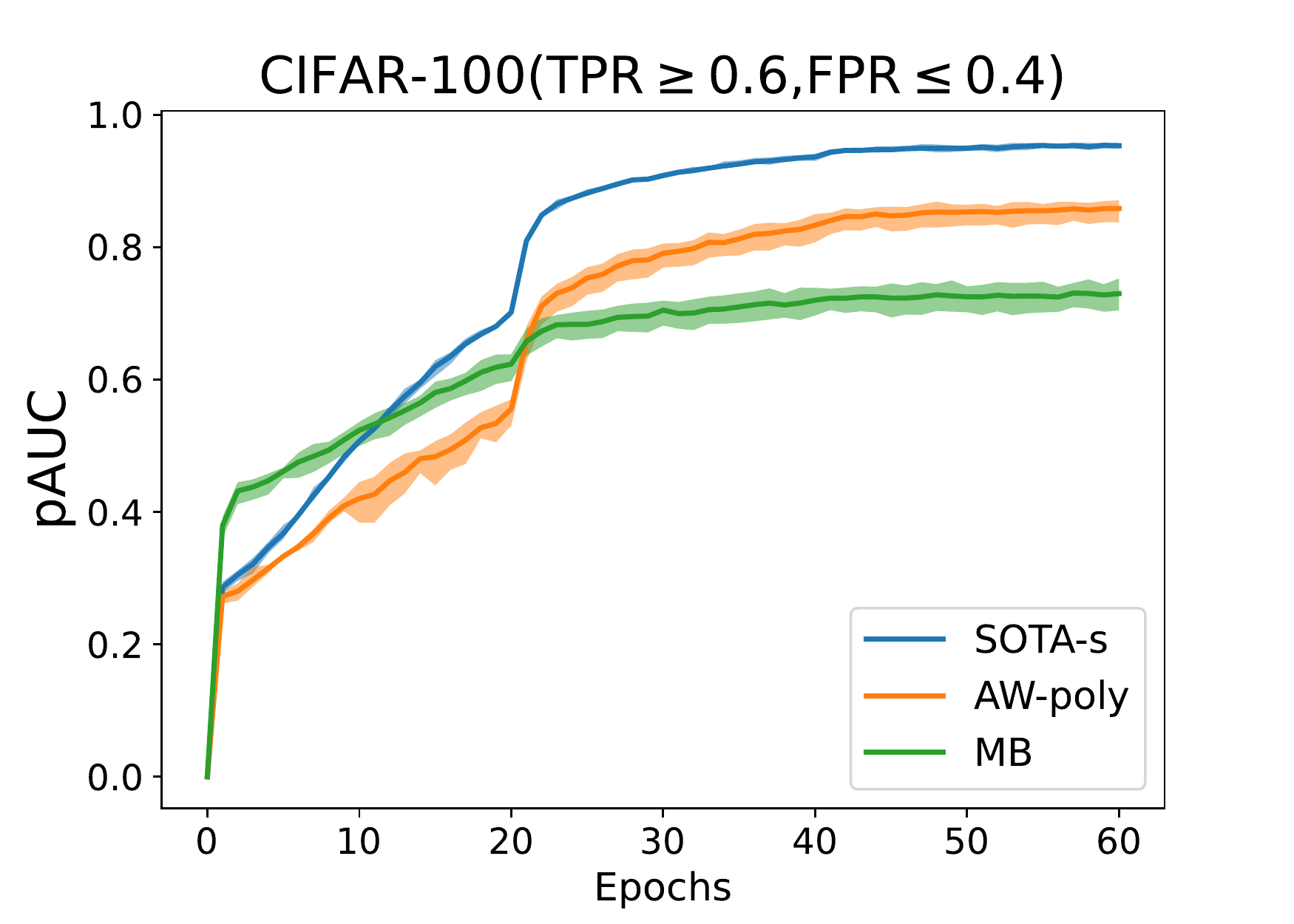}
    \includegraphics[width=0.24\linewidth]{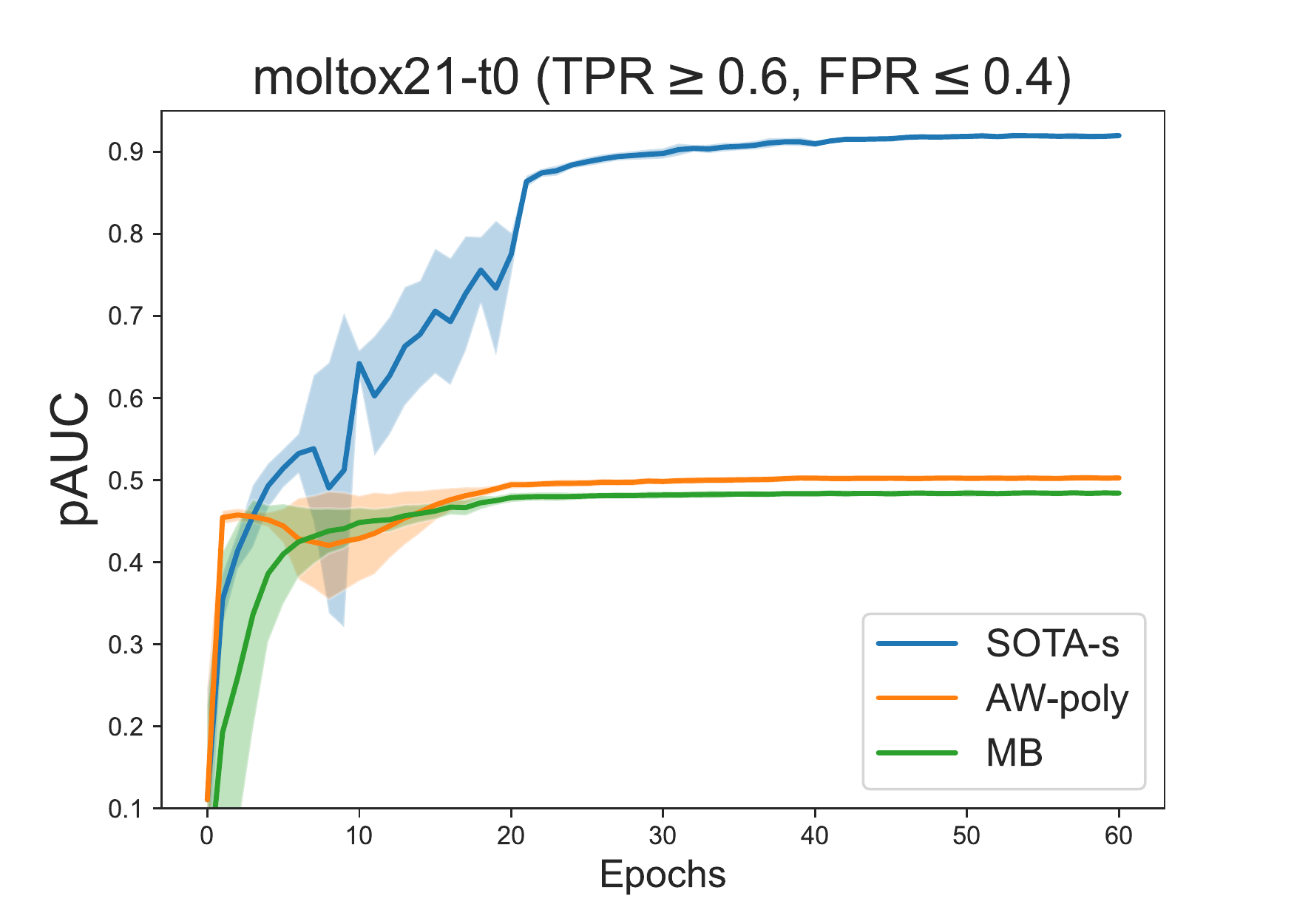}
    \includegraphics[width=0.24\linewidth]{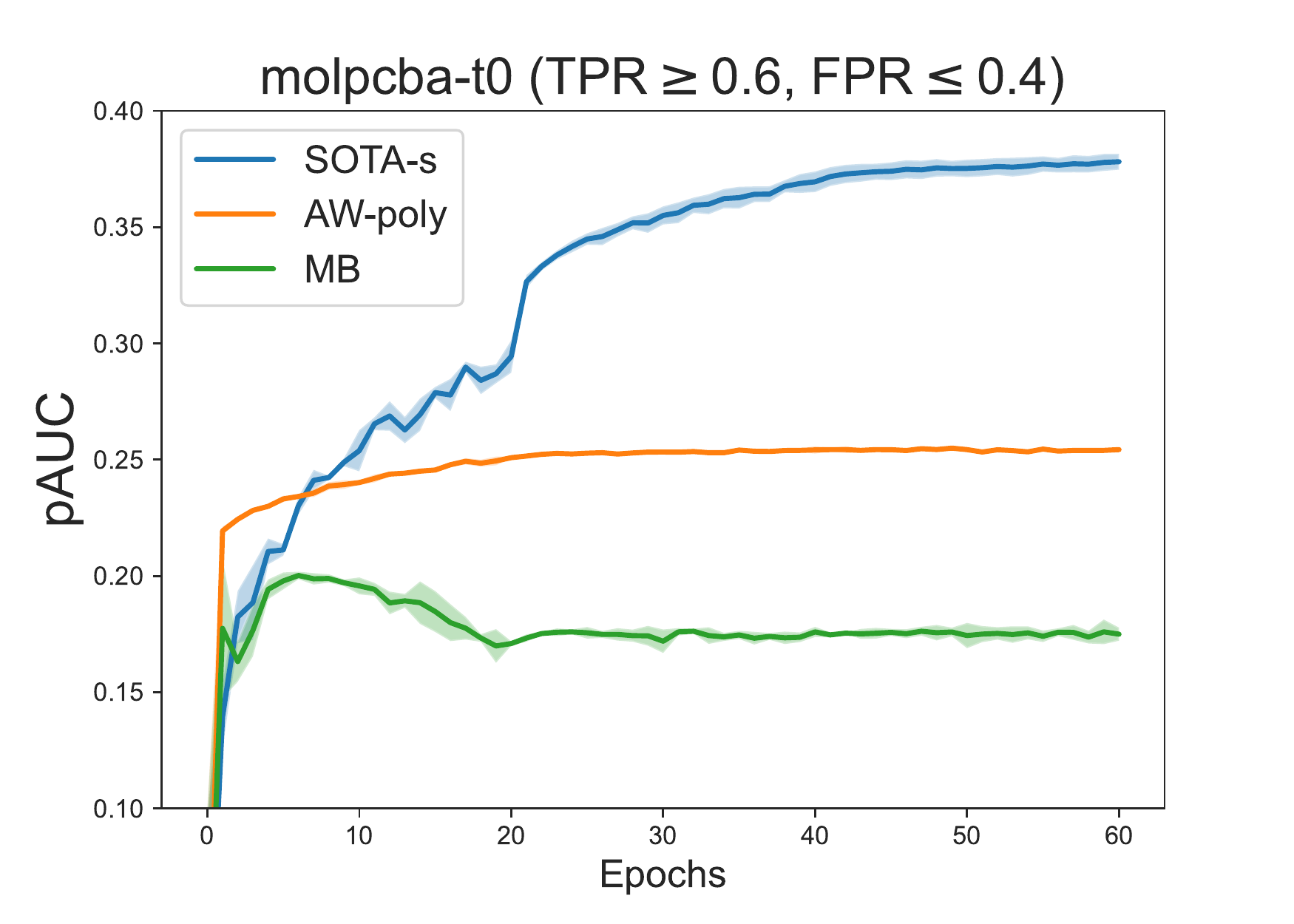}
\vspace{-0.1in}
\centering
\caption{Training Convergence Curves on image and molecular datasets; Top for OPAUC maximization, bottom for TPAUC maximization.}
\label{fig:1}
\vspace{-0.2in}
\end{figure*}

\section{Experiments}

{\bf Datasets.} We consider binary classification tasks on two types of datasets, namely image datasets and molecular datasets. For image datasets, we use CIFAR-10, CIFAR-100, Melanoma for experiments. For CIFAR-10 and CIFAR-100~\cite{krizhevsky2009learning}, we construct imbalanced versions of the datasets by randomly removing some positive samples following~\cite{DBLP:journals/corr/abs-2012-03173}. Specifically, we take first half of classes as the negative class and last half of classes as the positive class, and then remove 80\% samples from the positive class to make it imbalanced.
The Melanoma dataset is a naturally imbalanced medical dataset which is released on Kaggle~\cite{rotemberg2021patient}. 
For molecular datasets, we use  ogbg-moltox21 (the No.0 target), ogbg-molmuv (the No.1 target) and ogbg-molpbca (the No.0 target) for experiments, which are from the Stanford Open Graph Benchmark (OGB) website~\cite{hu2020open}. The task on these molecular datasets is to predict certain property of molecules. 
The statistics for the datasets are presented in Table~\ref{tab:data_stat} in the supplement.

{\bf Deep Models.} For image datasets,  we learn convolutional neural network (CNN) and use ResNet18~\cite{he2016deep} for CIFAR-10, CIFAR100 and Melanoma. For molecular datasets, we learn graph neural network (GNN) and use Graph Isomorphism Network (GIN) as the backbone model on all datasets~\cite{xu2018powerful}, which has 5 mean-pooling layers with 64 number of hidden units and dropout rate  0.5.

{\bf Baselines.} We will compare our methods with different baselines for both training performance and testing performance. For comparison of training convergence, we consider different methods for optimizing the same objective, i.e., partial AUC.  We compare with 2 baselines, i.e., the naive mini-batch based method~\cite{10.5555/2968826.2968904}, to which we refer as MB,  and a recently proposed ad-hoc weight based method~\cite{pmlr-v139-yang21k},  to which we refer as AW-poly. MB that optimizes OPAUC  only considers the top negative samples in the mini-batch; and MB that optimizes TPAUC considers the top negative samples and bottom positive samples in the mini-batch.  For AW-poly, we use the polynomial weight function according to their paper. It is notable that AW-poly was originally proposed for optimizing TPAUC. But it can be easily modified for optimizing OPAUC with FPR in $(0,\beta)$.  For comparison of testing performance, we compare different methods for optimizing different objectives, including the cross-entropy loss (CE), the pair-wise squared hinge loss for AUC maximization (AUC-SH), the AUC min-max margin loss (AUC-M)~\cite{DBLP:conf/icml/YuanGXYY21}, p-norm push (P-push)~\cite{JMLR:v10:rudin09b}. For optimizing CE and AUC-SH, we use the standard Adam optimizer. For optimizing AUC-M, we use their proposed optimizer PESG. For P-push, we use a stochastic algorithm with an Adam-style update similar to that proposed in~\cite{DBLP:journals/corr/abs-2104-08736}. For our methods, we use  $\ell(t)=(1-t)_+^2$ and also use the Adam-style update unless specified explicitly.  Similar to~\cite{DBLP:conf/icml/YuanGXYY21,DBLP:journals/corr/abs-2104-08736}, we use a pre-training step that optimizes the base model by optimizing CE loss with an Adam optimizer, and then {re-initialize the classifier layer and fine-tune all layers by different methods.} 


\begin{table*}[h!]
    \caption{One way partial AUC on testing data of three image datasets}
    \label{tab:op-image-test}
    \centering
\resizebox{0.8\linewidth}{!}{
    \begin{tabular}{ l|cc|cc|cc }
    \toprule
    &  \multicolumn{2}{c}{CIFAR-10} &\multicolumn{2}{|c}{CIFAR-100}& \multicolumn{2}{|c}{Melanoma} \\
         \hline
          Methods& FPR$\leq$0.3&FPR$\leq$0.5 &FPR$\leq$0.3&FPR$\leq$0.5 &FPR$\leq$0.3&FPR$\leq$0.5\\
         \hline
         CE  & 0.8446(0.0018)&0.8777(0.0014)& 0.7338(0.0047)&0.7787(0.0044) & 0.7651(0.0135)  &  0.8151(0.0028)\\
         AUC-SH     &0.8657(0.0056)&0.8948(0.0036)&0.7467(0.0047)&0.7930(0.0027) &  0.7824(0.0138) & 0.8176(0.0160) \\
         AUC-M    & 0.8678(0.0016)& 0.8934(0.0022)&0.7371(0.0031)& 0.7828(0.0005) &  0.7788(0.0068) & 0.8249(0.0141)\\
         P-push  & 0.8610(0.0007)& 0.8889(0.0021)&0.7445(0.0025)&0.7930(0.0029)&  0.7440(0.0130) &  0.8028(0.0170)\\
         MB  & 0.8690(0.0016)&0.8931(0.0015)&0.7487(0.0017)&0.7930(0.0014)&  0.7683(0.0303) & 0.8184(0.0278)\\
         AW-poly & 0.8664(0.0052)& 0.8915(0.0075)&0.7490(0.0058)& 0.7909(0.0068)&  0.7936(0.0238) & 0.8355(0.0067)  \\
         SOPA & \textbf{0.8766(0.0034)}& \textbf{0.9028(0.0031)} & \textbf{0.7551(0.0044)}& \textbf{0.7999(0.0028)}& \textbf{0.8093(0.0248)} & \textbf{0.8585(0.0210)} \\
         SOPA-s  &0.8691(0.0036) &0.8961(0.0036)&0.7468(0.0056)&  0.7877(0.0053) &  0.7775(0.0076) &  0.8401(0.0206)\\
         \bottomrule
    \end{tabular}
    }
    \caption{Two way partial AUC on testing data of three image datasets; ($\alpha$,$\beta$) represents TPR$\ge\alpha$ and FPR $\le\beta$.}
    \label{tab:tp-image-test}
\resizebox{0.8\linewidth}{!}{
    \begin{tabular}{ l|cc|cc|cc }
    \toprule
        &  \multicolumn{2}{c}{CIFAR-10} &\multicolumn{2}{|c}{CIFAR-100}& \multicolumn{2}{|c}{Melanoma} \\
         \hline
          Methods&(0.6,0.4)&(0.5,0.5) & (0.6,0.4)&(0.5,0.5) &(0.6,0.4)&(0.5,0.5)\\
        \hline
        CE   &0.4981(0.0078) &0.6414(0.0080) &0.2178(0.0136)  &0.4011(0.0118) &0.3399(0.0135)  &0.5150(0.0038) \\
             AUC-SH  &0.5622(0.0064) &0.6923(0.0071)&0.2599(0.0061)  &0.4397(0.0062) &0.3640(0.0354)  &0.5291(0.0312)  \\
             AUC-M  &0.5691(0.0021) &0.6907(0.0125) &0.2336(0.0041) &0.4153(0.0022) &0.3665(0.0646)  &0.5404(0.0545)  \\
             P-push   &0.5477(0.0077) &0.6781(0.0055) &0.2623(0.0042) &0.4417(0.0092) &0.3317(0.0304)  &0.4870(0.0443)  \\
             MB    &0.5404(0.0041) &0.6724(0.0011) &0.2207(0.0033) &0.4017(0.0149) &0.3330(0.0258)  &0.4981(0.0252)  \\
             AW-poly  &0.5536(0.0196) &0.6814(0.0203) &0.2489(0.0166) &0.4342(0.0112) &0.3878(0.0292)  &0.5216(0.0288)  \\
            SOTA-s    &\textbf{0.5799(0.0202)} &\textbf{0.7074(0.0145)} &\textbf{0.2708(0.0055)} &\textbf{0.4528(0.0069)} &\textbf{0.4198(0.0825)}  &\textbf{0.5865(0.0664)} \\
     \bottomrule
  \end{tabular}}\vspace*{-0.15in}
\end{table*}
\begin{table*}[h!]
    \caption{One way partial AUC on testing data of three  molecular datasets}
    \label{tab:op-molecule-test}
    \centering
\resizebox{0.8\linewidth}{!}{
    \begin{tabular}{ l|cc|cc|cc }
    \toprule
        &  \multicolumn{2}{c}{moltox21(t0)} &\multicolumn{2}{|c}{molmuv(t1)}& \multicolumn{2}{|c}{molpcba(t0)} \\
         \hline
          Methods&FPR$\leq$0.3&FPR$\leq$0.5&FPR$\leq$0.3&FPR$\leq$0.5 &FPR$\leq$0.3&FPR$\leq$0.5\\
         \hline
         CE  &0.6671(0.0009) &0.6954(0.005)& 0.8008(0.0090)& 0.8201(0.0061)  & 0.6802(0.0002)  &  0.7169(0.0002)\\
         AUC-SH     & 0.7161(0.0043)& 0.7295(0.0036)& 0.7880(0.0382)&0.8025(0.0437)&    0.6939(0.0009) & 0.7350(0.0015) \\
         AUC-M    & 0.6866(0.0048)& 0.7080(0.0020)& 0.7960(0.0123)&0.8076(0.0175)   &  0.6985(0.0016) & 0.7399(0.0005)\\
         P-push  & 0.6946(0.0107)& 0.7160(0.0073)& 0.7832(0.0220)& 0.7940(0.0321) &    0.6841(0.0007) &  0.7293(0.0043)\\
         MB  & \textbf{0.7398(0.0131)}&0.7329(0.0099)& 0.7672(0.0563)&0.7772(0.0547) &   0.6899(0.0002) & 0.7253(0.0006)\\
         AW-poly & 0.7227(0.0024)& 0.7271(0.0112)& 0.7754(0.0372)&0.7883(0.0431) &   0.6975(0.0006) & 0.7350(0.0015)  \\
         SOPA  & 0.7209(0.0063)& 0.7318(0.0084) &0.8187(0.0319) &0.8245(0.0312) &  0.6989(0.0022) &  0.7371(0.0011) \\
         SOPA-s  &0.7309(0.0151) &\textbf{0.7330(0.0073)}& \textbf{0.8449(0.0399)}&\textbf{0.8412(0.0447)}&  \textbf{0.7027(0.0018)} &  \textbf{0.7416(0.0006)}\\
         \bottomrule
    \end{tabular}
    }


    \caption{Two way partial AUC on testing data of three molecular datasets; ($\alpha$,$\beta$) represents TPR$\ge\alpha$ and FPR $\le\beta$.}
    \label{tab:tp-molecule-test}
\resizebox{0.8\textwidth}{!}{
    \begin{tabular}{ l|cc|cc|cc }
    \toprule
        &  \multicolumn{2}{c}{moltox21(t0)} &\multicolumn{2}{|c}{molmuv(t1)}& \multicolumn{2}{|c}{molpcba(t0)} \\
         \hline
          Methods &(0.6,0.4)&(0.5,0.5) & (0.6,0.4)&(0.5,0.5) &(0.6,0.4)&(0.5,0.5)\\
         \hline
         CE    & 0.0674(0.0014) &  0.2082(0.0011)& 0.1613(0.0337) &  0.4691(0.0183)&    0.0949(0.0006)  &  0.2639(0.0006)\\
         AUC-SH      & 0.0640(0.0080) & 0.2170(0.0140) & 0.2600(0.1300) &  0.4440(0.1280)  &  0.1400(0.0030)  &  0.3120(0.0030) \\
         AUC-M    &  0.0660(0.0090)&  0.2090(0.0100)& 0.1140(0.0790) &  0.4330(0.0530) &  0.1420(0.0090)  &   0.3130(0.0030)\\
         P-push    & 0.0610(0.0180) & 0.2070(0.0120) & 0.1860(0.1520) & 0.4170(0.1080)&  0.1350(0.0020)  & 0.3000(0.0120)  \\
         MB  & 0.0670(0.0150) & 0.2150(0.0230) & 0.1730(0.1530) & 0.4260(0.1180)  &  0.0950(0.0020)  &  0.2620(0.0030)\\
         AW-poly  & 0.0640(0.0100) & 0.2060(0.0250) & 0.1720(0.1440) & 0.3930(0.1230) &  0.1100(0.0010)  & 0.2810(0.0020) \\
         SOTA-s   & \textbf{0.0680(0.0180)} & \textbf{0.2300(0.0210)} & \textbf{0.3270(0.1640)} & \textbf{0.5260(0.1220)}  &  \textbf{0.1430(0.0010)}  & \textbf{0.3140(0.0020)} \\

     \bottomrule
    \end{tabular}
    }
    \vspace*{-0.2in}
\end{table*}
{\bf Target Measures.} For OPAUC maximization, we evaluate OPAUC with two FPR upper bounds, i.e., FPR $\leq 0.3$ and FPR$\leq 0.5$ separately. For TPAUC maximization, we evaluate TPAUC with two settings, i.e, FPR$\leq 0.4$ and TPR$\geq 0.6$, and FPR$\leq 0.5$ and TPR$\geq 0.5$. 

{\bf Parameter Tuning.}\label{meta-exp-info}
The learning rate of all methods is tuned in \{1e-3, 1e-4, 1e-5\}, except for PESG which is tuned at \{1e-1, 1e-2, 1e-3\} because it favors a larger learning rate. Weight decay is fixed as 2e-4. Each method is run 60 epochs in total and learning rate decays 10-fold after every 20 epochs. The mini-batch size is 64. 
For AUC-M, we tune the hyperparameter $\gamma$ that controls consecutive epoch-regularization in \{100, 500, 1000\}. For P-push, we tune the polynomial power hyper-parameter in \{2, 4, 6\}. For MB that optimizes OPAUC, we tune the top proportion of negative samples in $\{10\%,30\%,50\%\}$, and for MB that optimizes TPAUC we tune the top proportion of negative samples in $\{30\%,40\%,50\%\}$, and tune the bottom proportion of positive samples in the range $\{30\%, 40\%, 50\%\}$. For AW-poly, we follow~\cite{pmlr-v139-yang21k} and tune its parameter $\gamma$ in \{101, 34, 11\}. For SOPA, we tune the truncated FPR i.e. $\beta$ in \{0.1, 0.3, 0.5\}. For SOPA-s, we fix $\gamma_0=0.9$ and tune the KL-regularization parameter $\lambda$ in \{0.1, 1.0, 10\}, and for SOTA-s,  we fix $
\gamma_0=\gamma_1=0.9$,  and tune  both $\lambda$ and $\lambda'$ in \{0.1, 1.0, 10\}. The momentum parameter for updating $\v_t$ in SOPA-s (i.e., $1-\gamma_1$) and SOTA-s (i.e., $1-\gamma_2$) is set to the default value as in the Adam optimizer, i.e., 0.1.  
For comparison of training convergence, the parameters are tuned according to the training performance. For comparison of testing performance, the parameters are tuned according to the validation performance. For each experiment, we repeat multiple times {with different train/validation splits and  random seeds, then report average and standard deviation over multiple runs}.

{\bf Results.} We show the plots of training convergence in Figure~\ref{fig:1} on two image datasets (CIFAR-10, -100) and on two molecular datasets (moltox21, molpcba). From the results, we can see that  SOPA-s (SOTA-s) converge always faster than MB and AW-poly for OPAUC (TPAUC) maximization. For OPAUC maximization, SOPA-s is usually faster than SOPA. More results are included in the supplement on other datasets with similar observations.  The testing performance on all six datasets are shown in Table~\ref{tab:op-image-test},~\ref{tab:tp-image-test},~\ref{tab:op-molecule-test} and~\ref{tab:tp-molecule-test}. In most cases, the proposed methods are better than the baselines.  In particular, dramatic improvements have been observed on Melanoma and ogbg-molmuv datasets, which are two datasets with the highest imbalance ratios.  In addition, we see that AUC maximization methods (AUC-M, AUC-SH) are not necessarily good for pAUC maximization. 

{\bf Accuracy of KLDRO-based estimator.} Of independent interest, we conduct simple experiments to verify the accuracy of KLDRO-based estimator of OPAUC. To this end, we compute the relative error (RE) of KLDRO-based estimator compared with the exact estimator (i.e., CVaR-based estimator). For a given upper bound of FRP we vary $\lambda$ for 100 independently randomly generated model parameters $\mathbf w$, and the results are shown in the following figure on moltox21-t0 data (please refer to the experiments section for more information of the dataset), which demonstrates that for a given FPR there exists $\lambda$ such that KLDRO estimator is close to the exact estimator. 
\begin{figure}[h]
    \centering
    \includegraphics[scale=0.17]{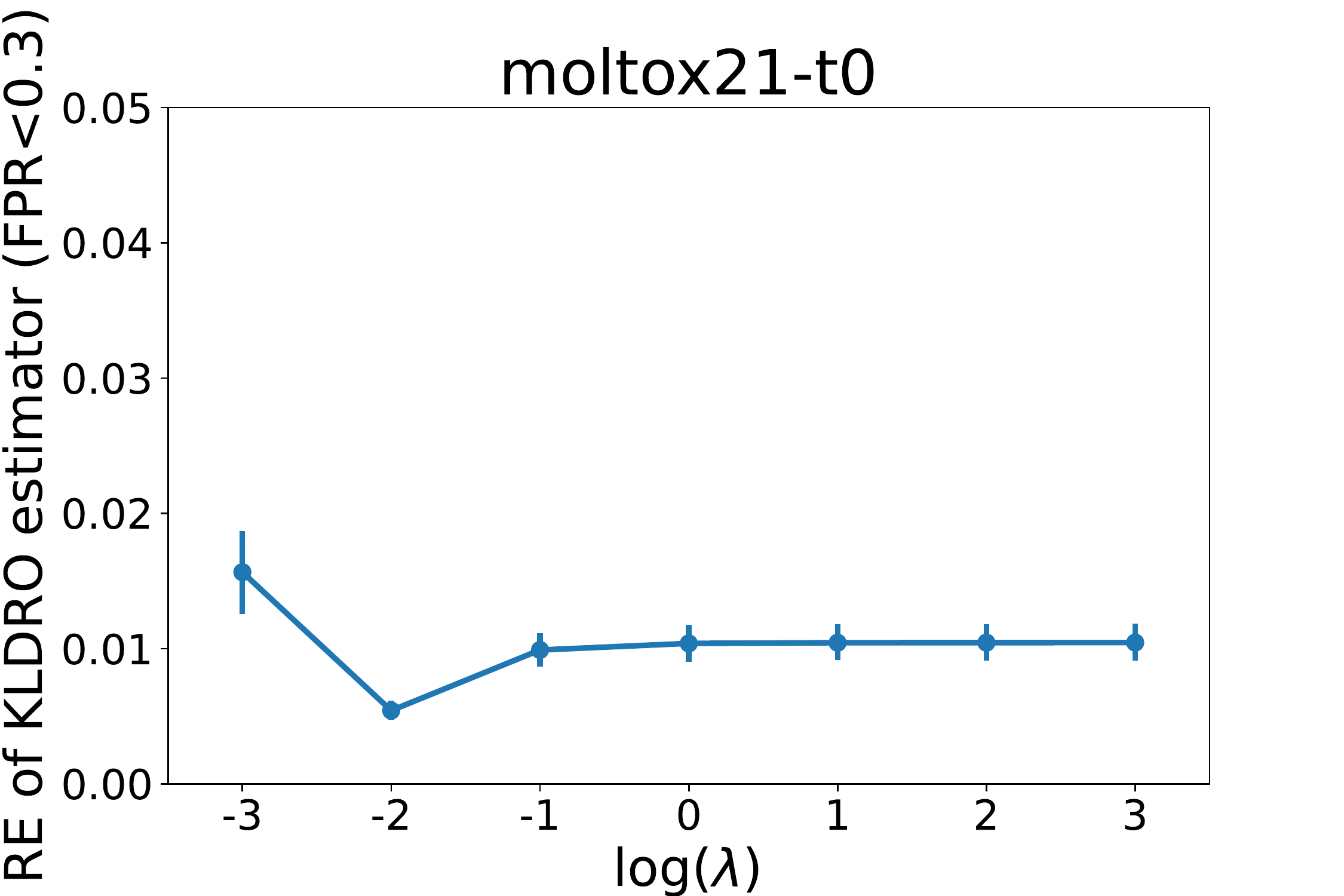}
    \includegraphics[scale=0.17]{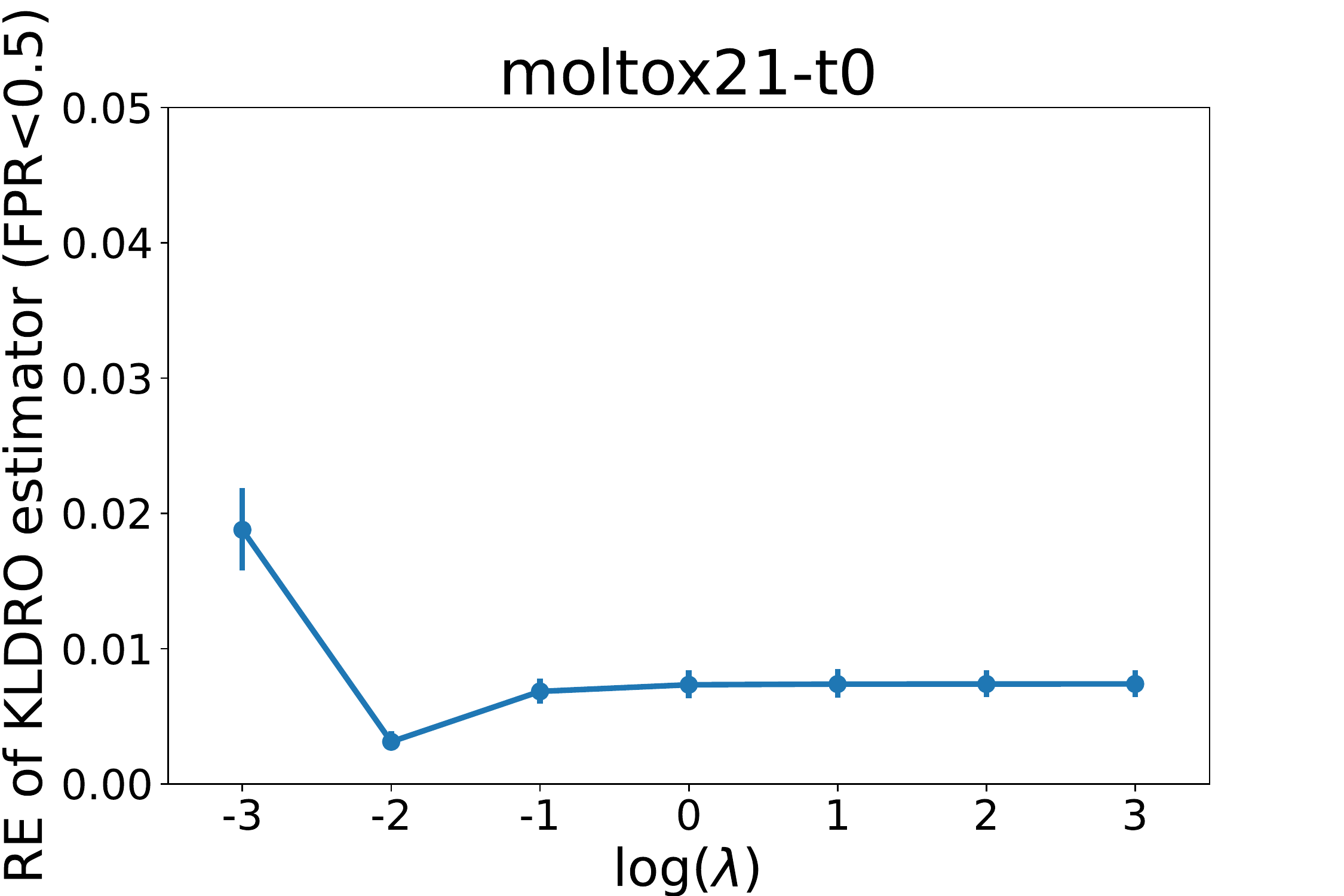}
    \caption{Relative error (RE) for KLDRO-based estimator for OPAUC on moltox21-t0 dataset with FPR=\{0.3, 0.5\}.}\label{estimator}\vspace*{0.2in}
\end{figure}

{\bf Ablation Study.} We also conduct some ablation study to understand the proposed algorithm SOPA-s and SOTA-s. In particular for both algorithms, we verify that tuning $\gamma_0$ in SOPA-s and $\gamma_0, \gamma_1$ in SOTA-s can help further improve the performance. The results are included in the supplement.  

\section{Conclusions}
In this paper, we have proposed new formulations for partial AUC maximization by using distributionally robust optimization. We propose two formulations for both one-way and two-way partial AUC, and develop stochastic algorithms with convergence guarantee for solving the two formulations, respectively. Extensive experiments on image and molecular graph datasets verify the effectiveness of the proposed algorithms. 

\section{Acknowledgements}
This work is partially supported by NSF Grant 2110545, NSF Career Award 1844403, and NSF Grant 1933212. D. Zhu and X. Wu was partially supported by NSF grant CCF-1733742. We also thank anonymous reviewers for constructive comments. 

\nocite{langley00}

\bibliography{AUC}
\bibliographystyle{icml2022}

\newpage
\onecolumn
\appendix

\section{More Experimental Results}
\begin{table}[h]
    \caption{Datasets Statistics. The percentage in parenthesis represents the proportion of positive samples.}
    \label{tab:data_stat}
    \centering
    \resizebox{0.6\textwidth}{!}{
    \begin{tabular}{lllll}
       Dataset & Train& Validation & Test \\
        \hline
        CIFAR-10  & 24000 (16.67\%)  & 6000 (16.67\%) & 6000 (16.67\%)\\
        CIFAR-100  & 24000 (16.67\%)  & 6000 (16.67\%) & 6000 (16.67\%) \\
        Melanoma & 26500 (1.76\%) & 3313 (1.78\%)  & 3313 (1.75\%) \\
        moltox21(t0)  & 5834 (4.25\%)  & 722 (4.01\%) & 709 (4.51\%)\\
        molmuv(t1)  & 11466 (0.18\%)  & 1559 (0.13\%) & 1709 (0.35\%) \\
        molpcba(t0) & 120762 (9.32\%) & 19865 (11.74\%)  & 20397 (11.61\%) \\
    \end{tabular}}
\end{table}

\subsection{Additional plots for training convergence}
We present more training convergence plots on  Melanoma dataset and molmuv dataset at Figure ~\ref{fig:addi}. For OPAUC maximization, We can observe that both our proposed SOPA-s and SOPA converge much better than AW-poly and MB method under different settings, i.e., FPR$\leq0.3$ and FPR$\leq 0.5$. And our proposed SOTA-s converge higher by a noticeable margin than AW-poly and MB method for TPAUC maximization all the time.
\begin{figure*}[h]
\centering
    \includegraphics[width=0.24\linewidth]{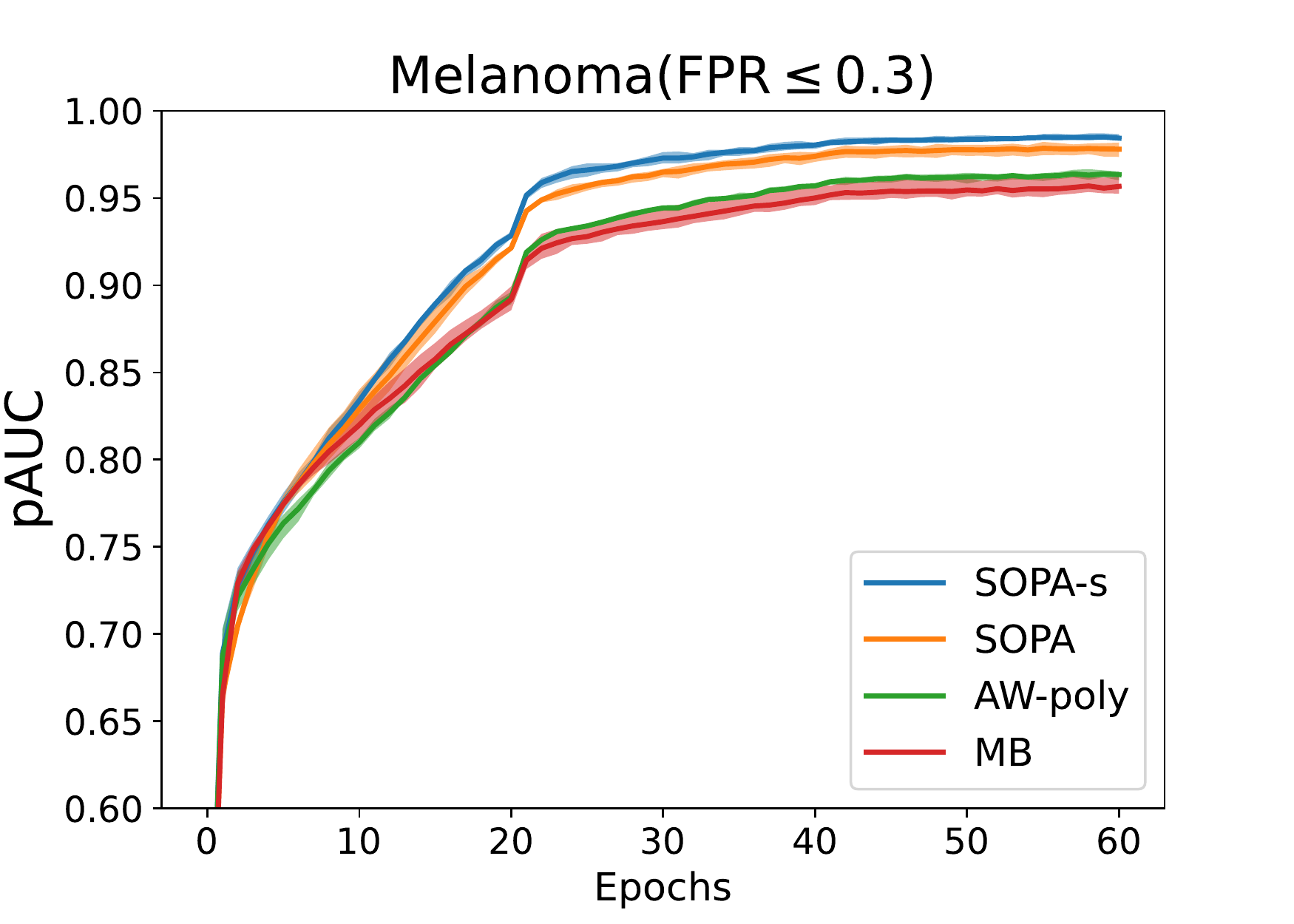}
    \includegraphics[width=0.24\linewidth]{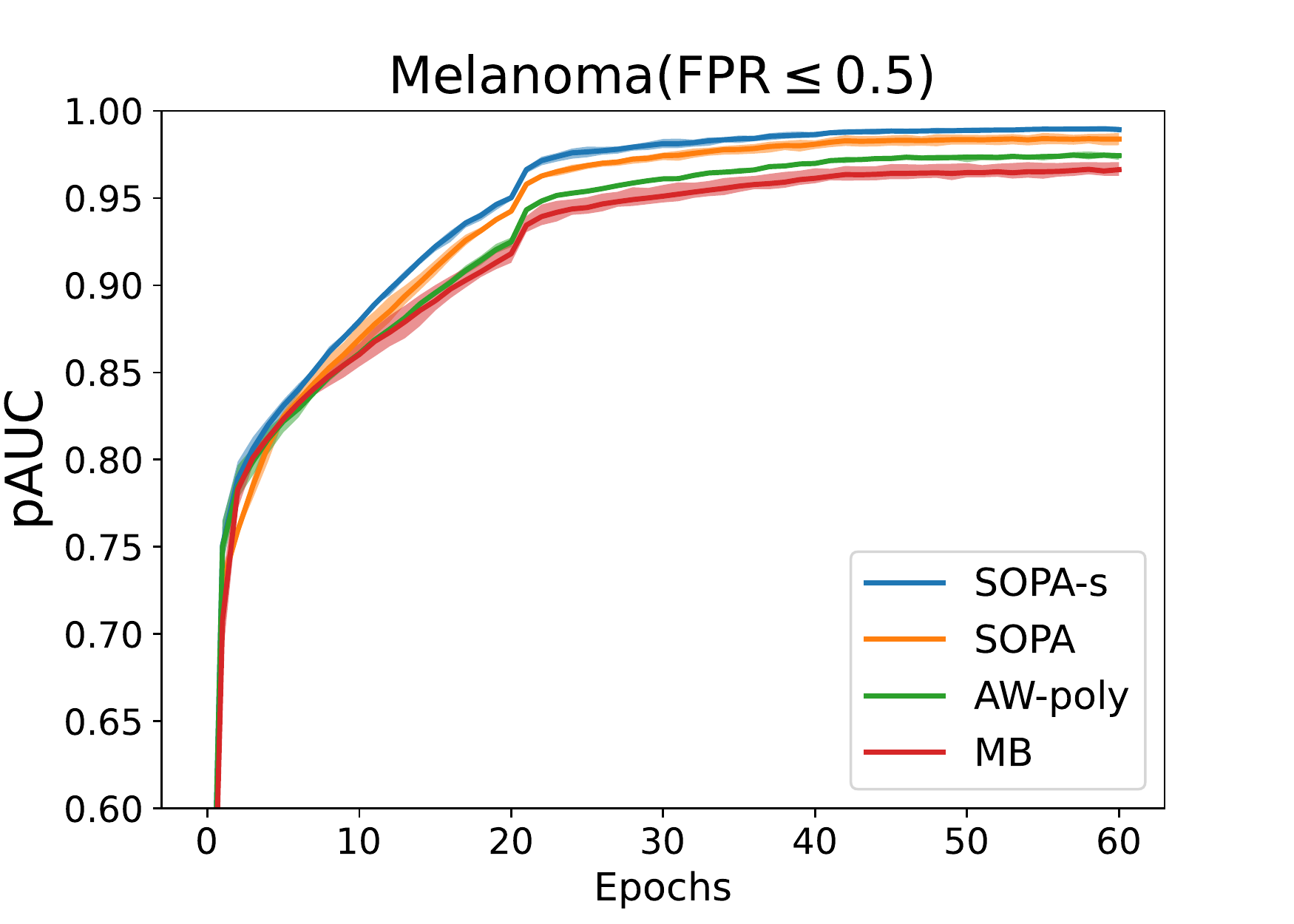}
    \includegraphics[width=0.24\linewidth]{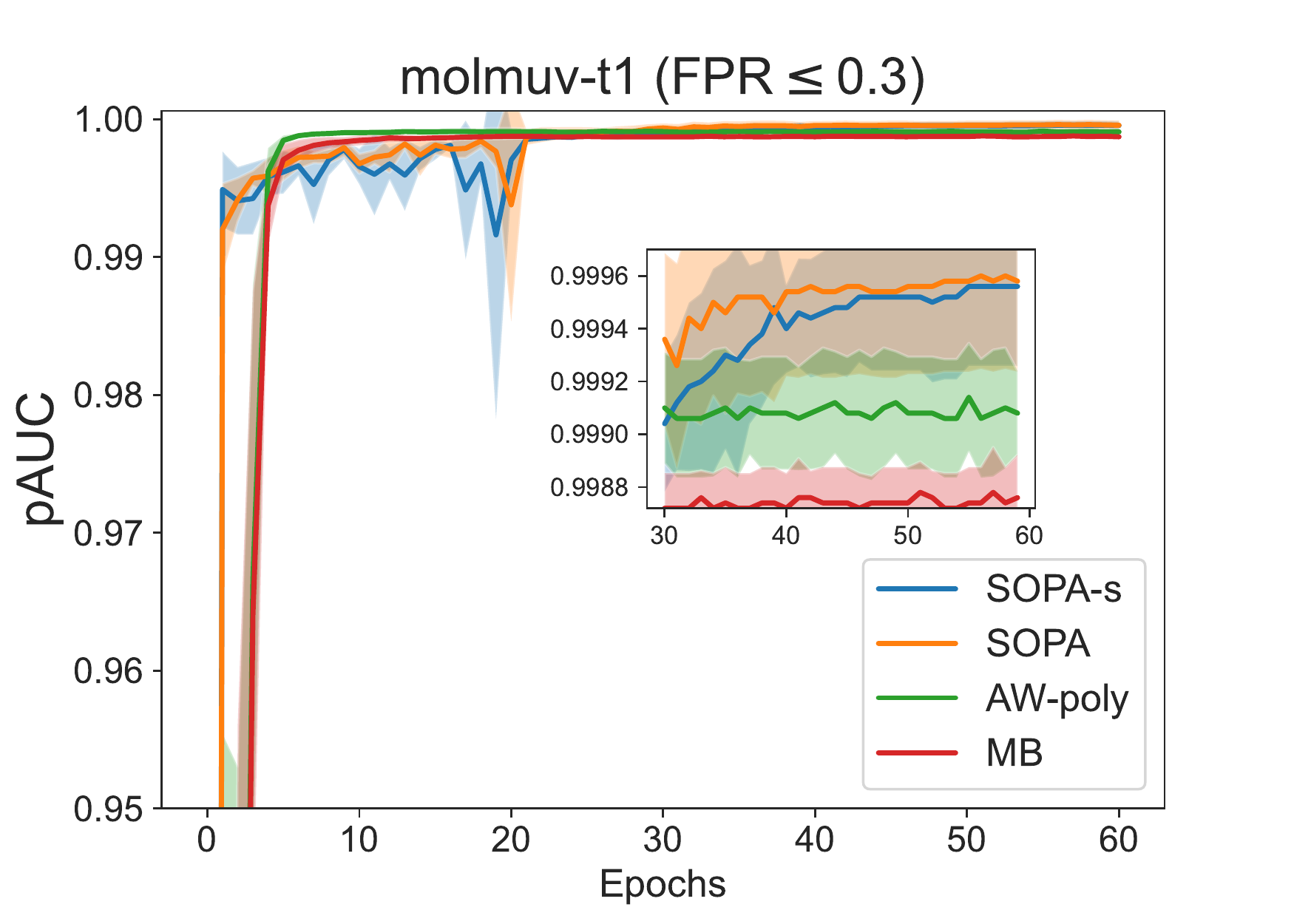}
     \includegraphics[width=0.24\textwidth]{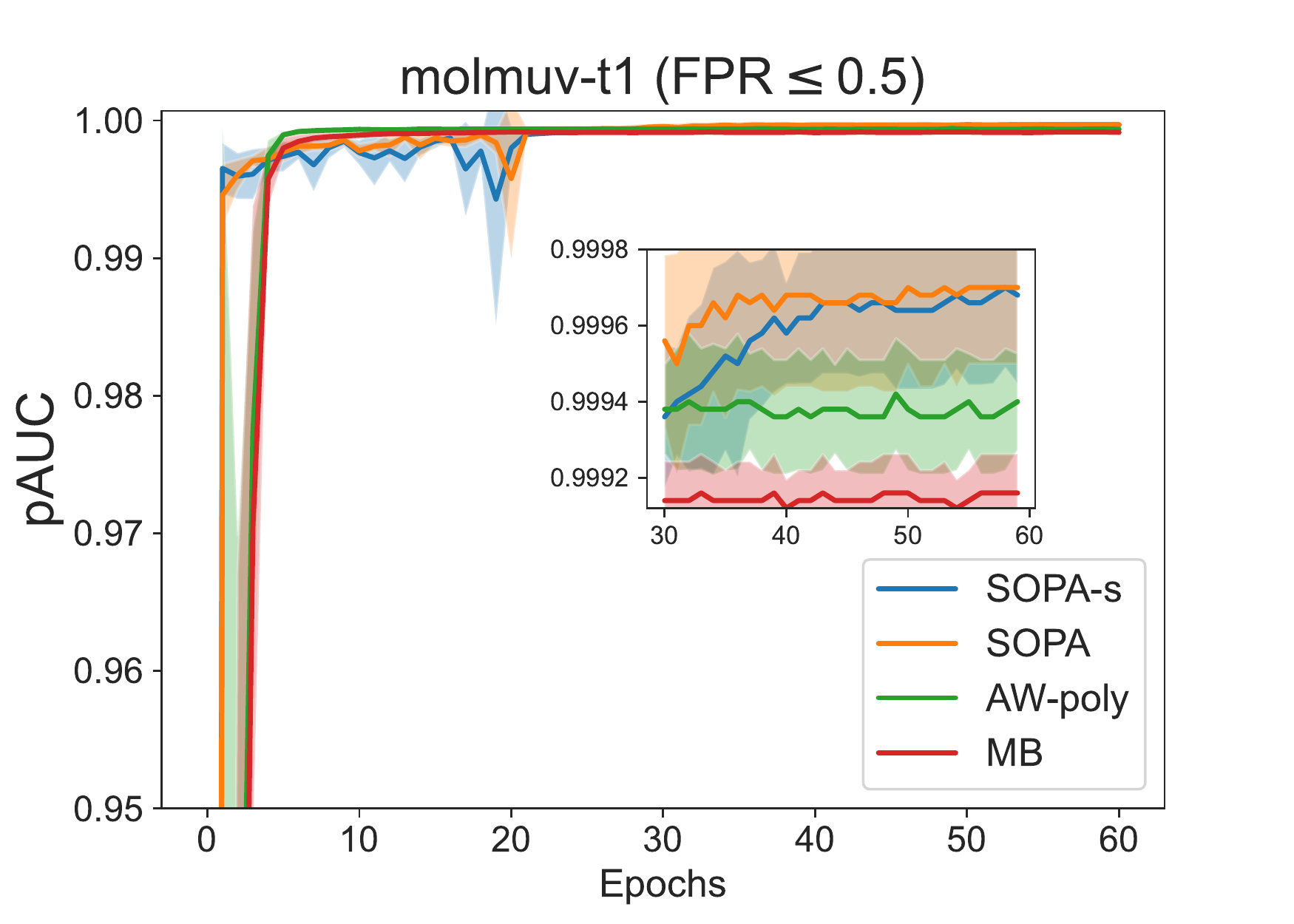}
         \label{fig:molmuv-03}
    
   \includegraphics[width=0.24\linewidth]{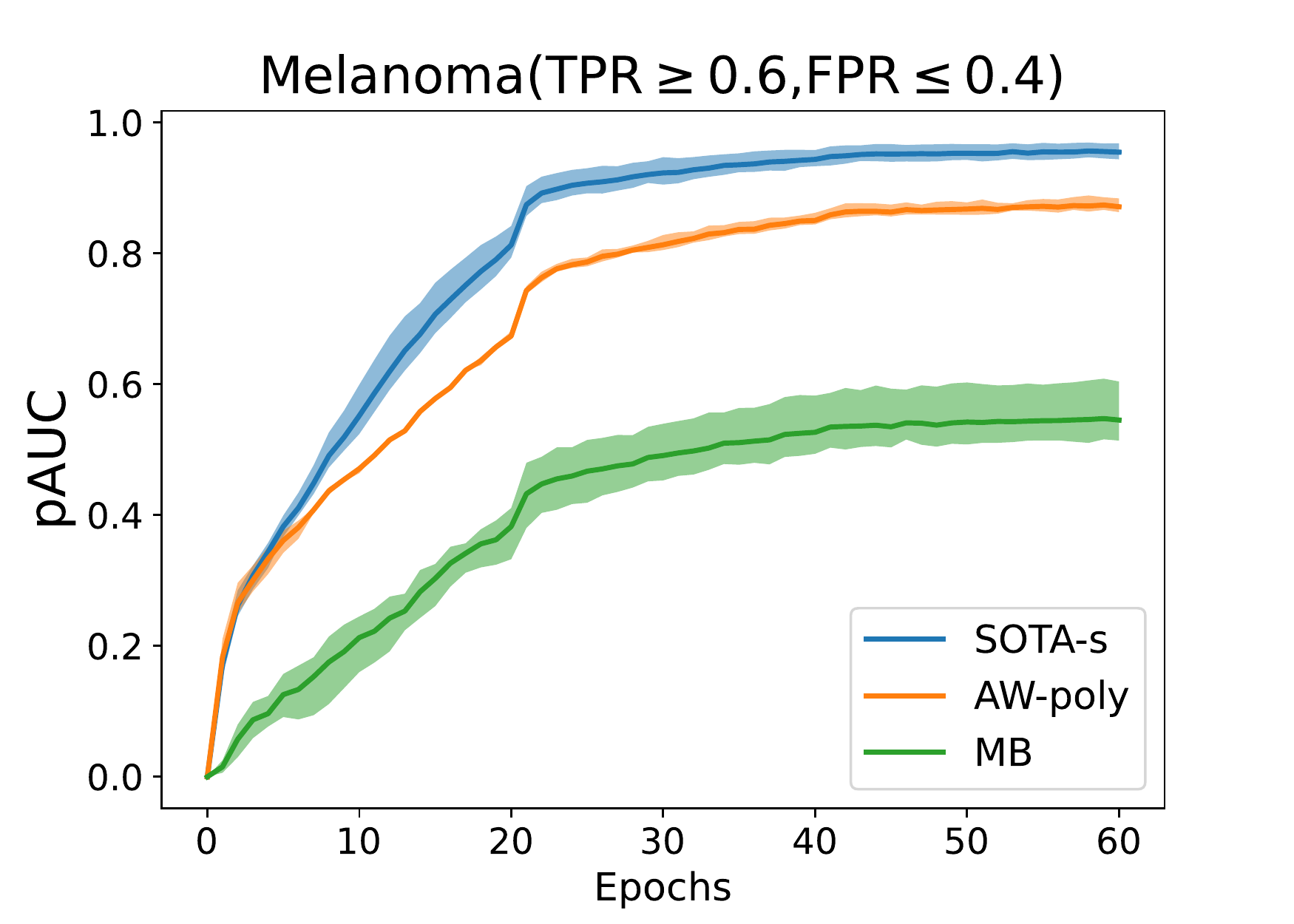}
    \includegraphics[width=0.24\linewidth]{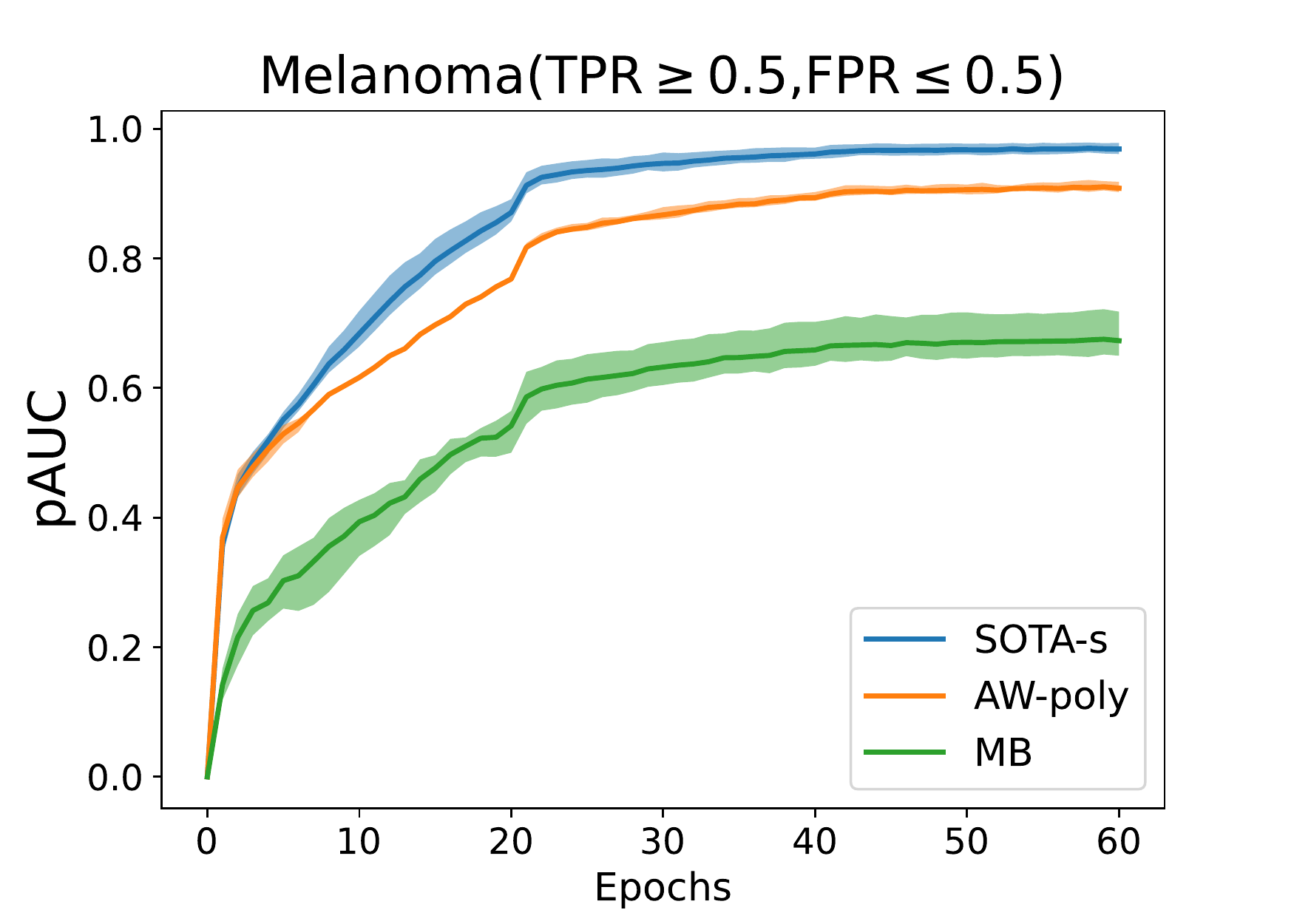}
    \includegraphics[width=0.24\linewidth]{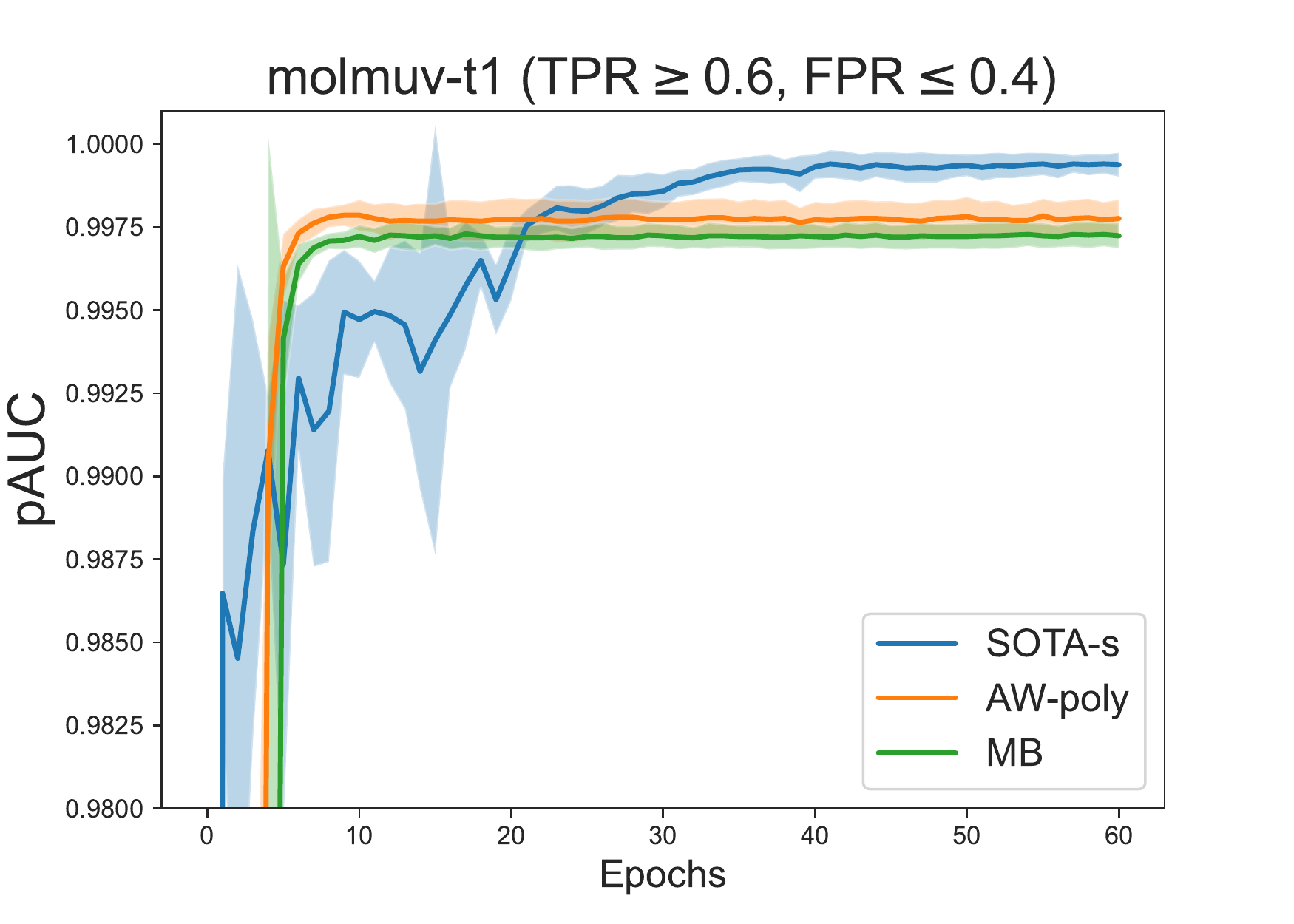}
    \includegraphics[width=0.24\linewidth]{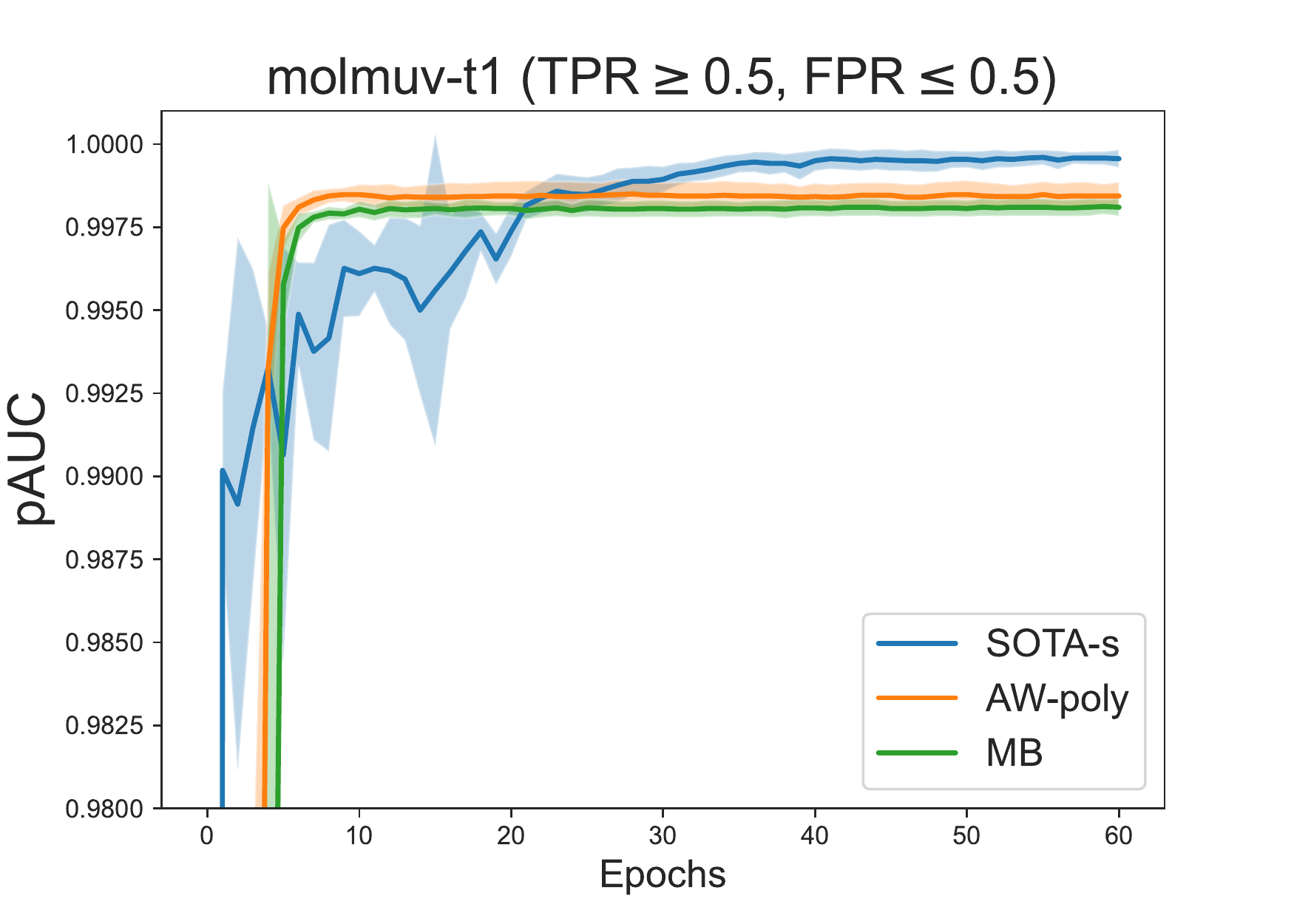}
\vspace{-0.1in}
\centering
\caption{Training Convergence Curves on Melanoma and molmuv datasets; Top for OPAUC maximization, bottom for TPAUC maximization.}
\label{fig:addi}
\vspace{-0.2in}
\end{figure*}

\subsection{Ablation study for $\gamma_0$ in SOPA-s and $\gamma_0, \gamma_1$ in SOTA-s}
We conduct extensive ablation study for understanding the extra hyper-parameters $\gamma_0$ in SOPA-s and $\gamma_0, \gamma_1$ in SOTA-s algorithms. We fix it as 0.9 for all of our experiments in the main content. But in practice, the performance would be further improved if we tune those hyper-parameters as well.

For image datasets, we conduct experiments on CIFAR-10 and CIFAR-100; for molecule datasets, we conduct experiments on ogbg-moltox21 and ogbg-molmuv. For each dataset, we first fix the best learning rate and other hyper-parameters based on our previous results in the paper. Then, for SOPA-s, we investigate $\gamma_0$ at \{1.0, 0.9, 0.7, 0.5, 0.3, 0.1\}; for SOTA-s, we investigate both $\gamma_0$ and $\gamma_1$ at \{1.0, 0.9, 0.7, 0.5, 0.3, 0.1\}

For training aspect, we include the comparisons for SOPA-s at Figure~\ref{fig:sopa-s-beta}; we include the comparisons for SOTA-s at Figure~\ref{fig:sota-s-beta}. From Figure~\ref{fig:sopa-s-beta}, we can see that better training performance could be achieved by tuning the parameter $\gamma_0$ in SOPA-s, compared with fixing it as 0.9; the similar result for SOTA-s can be also observed from Figure~\ref{fig:sota-s-beta}.

\begin{figure*}[h]
\centering
    \includegraphics[width=0.24\linewidth]{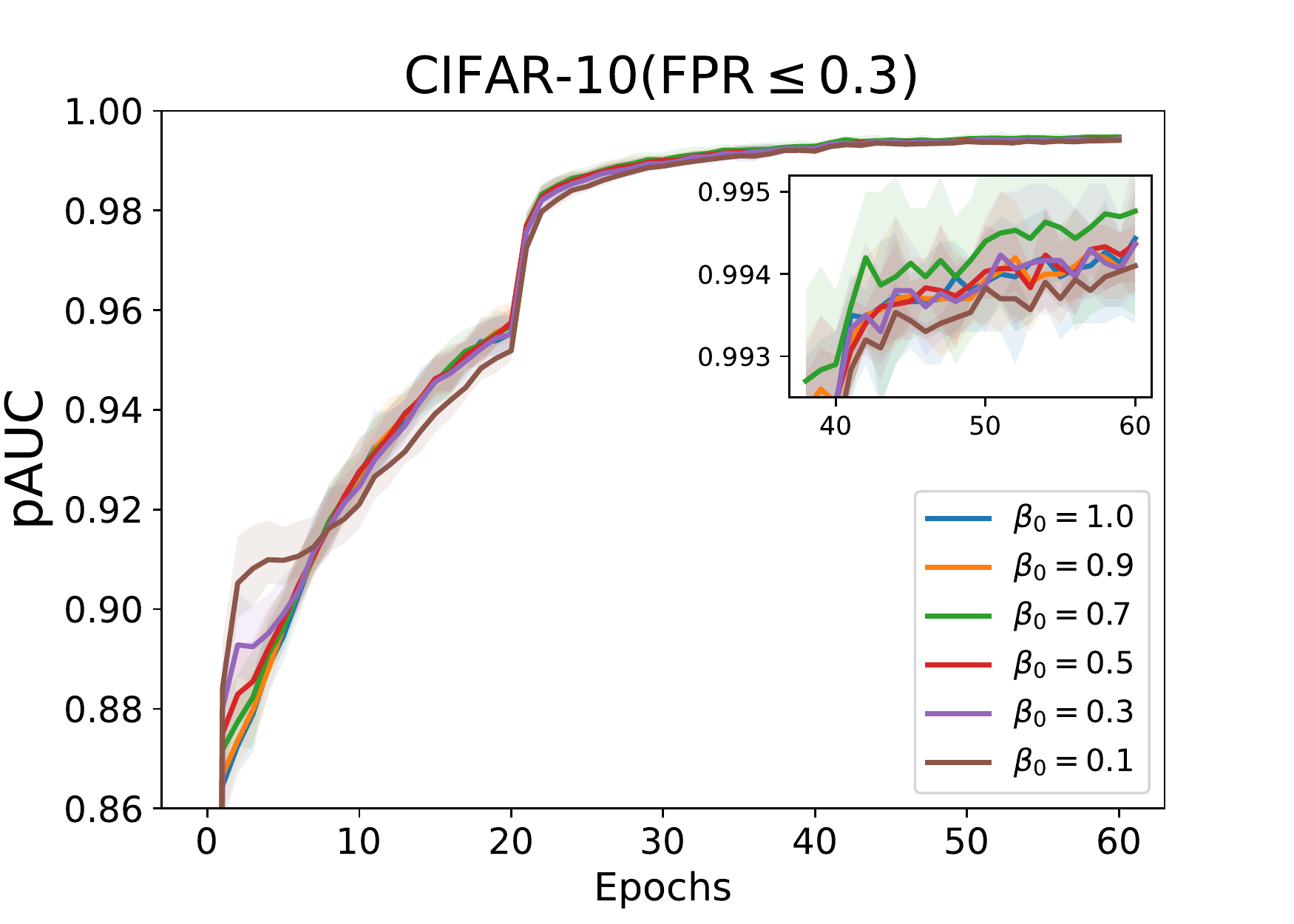}
    \includegraphics[width=0.24\linewidth]{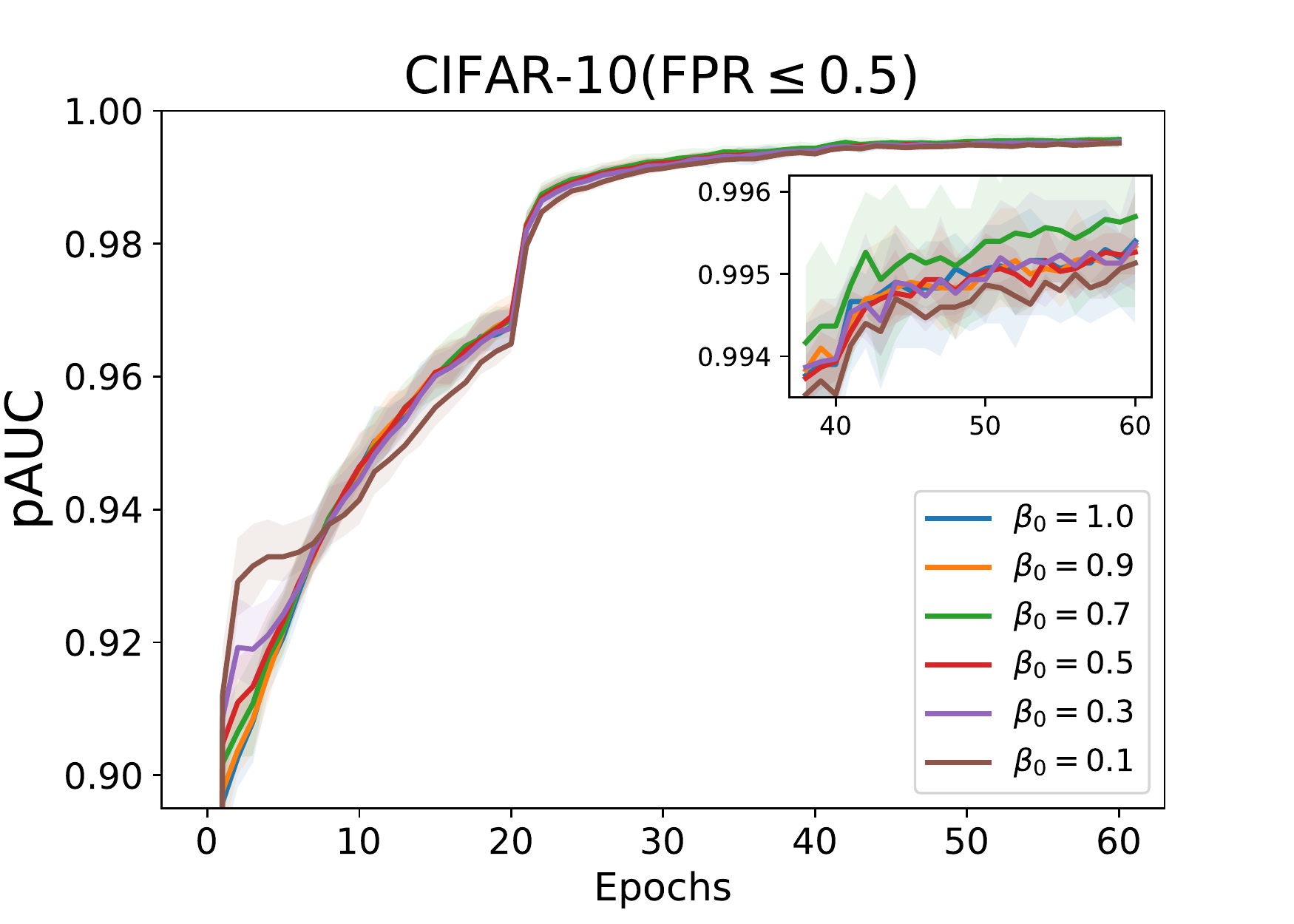}
    \includegraphics[width=0.24\linewidth]{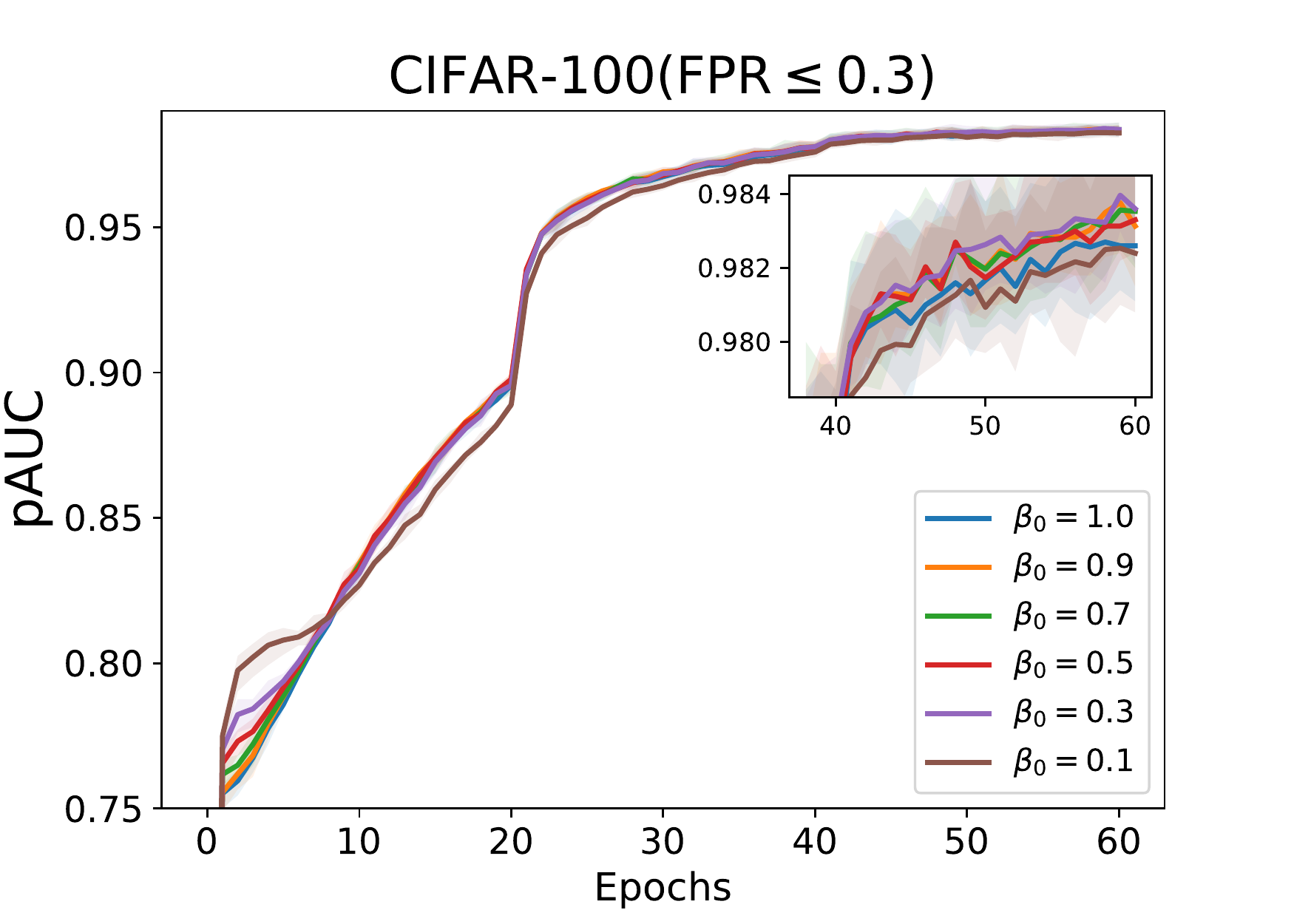}
    \includegraphics[width=0.24\linewidth]{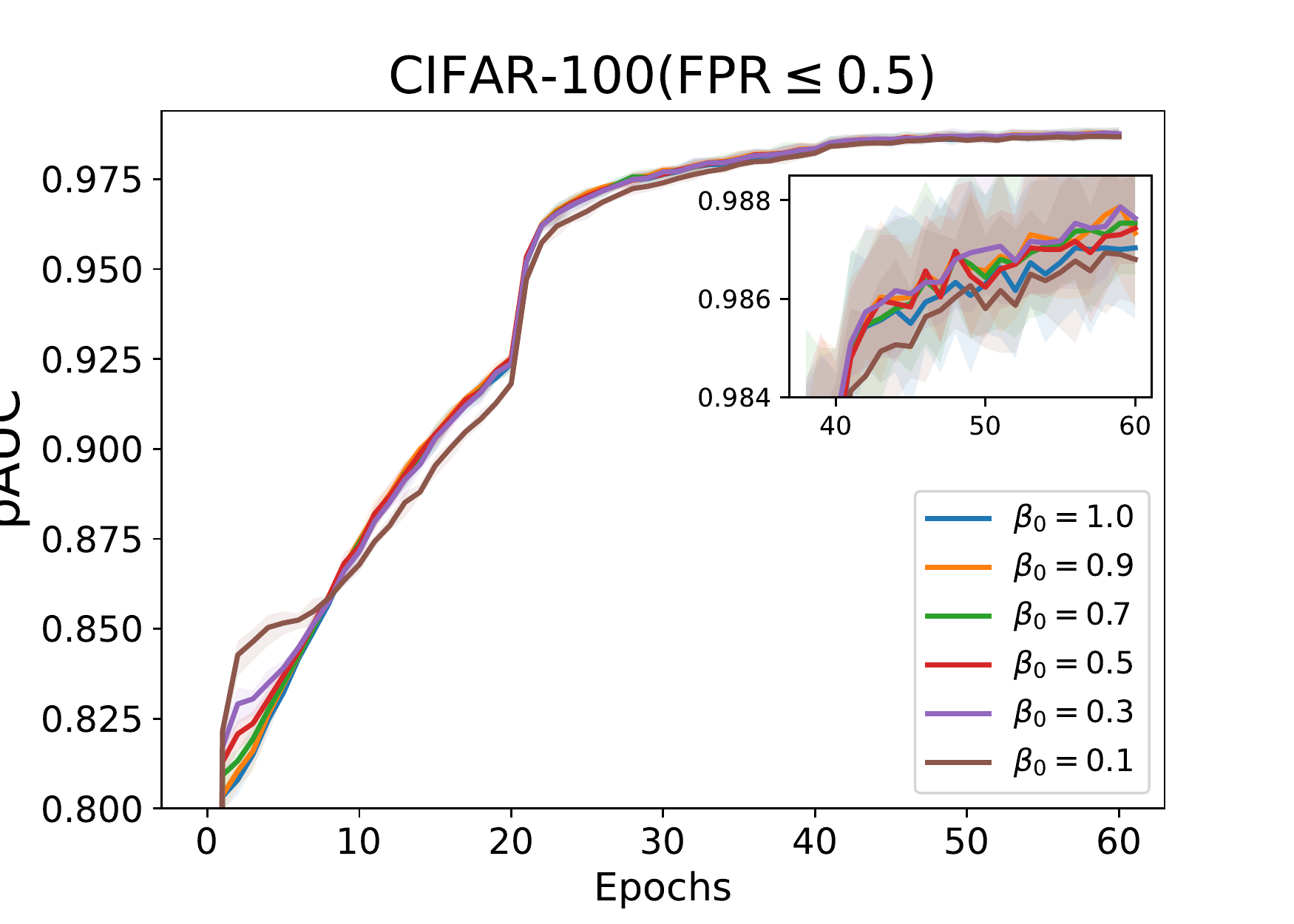}

     \includegraphics[width=0.24\linewidth]{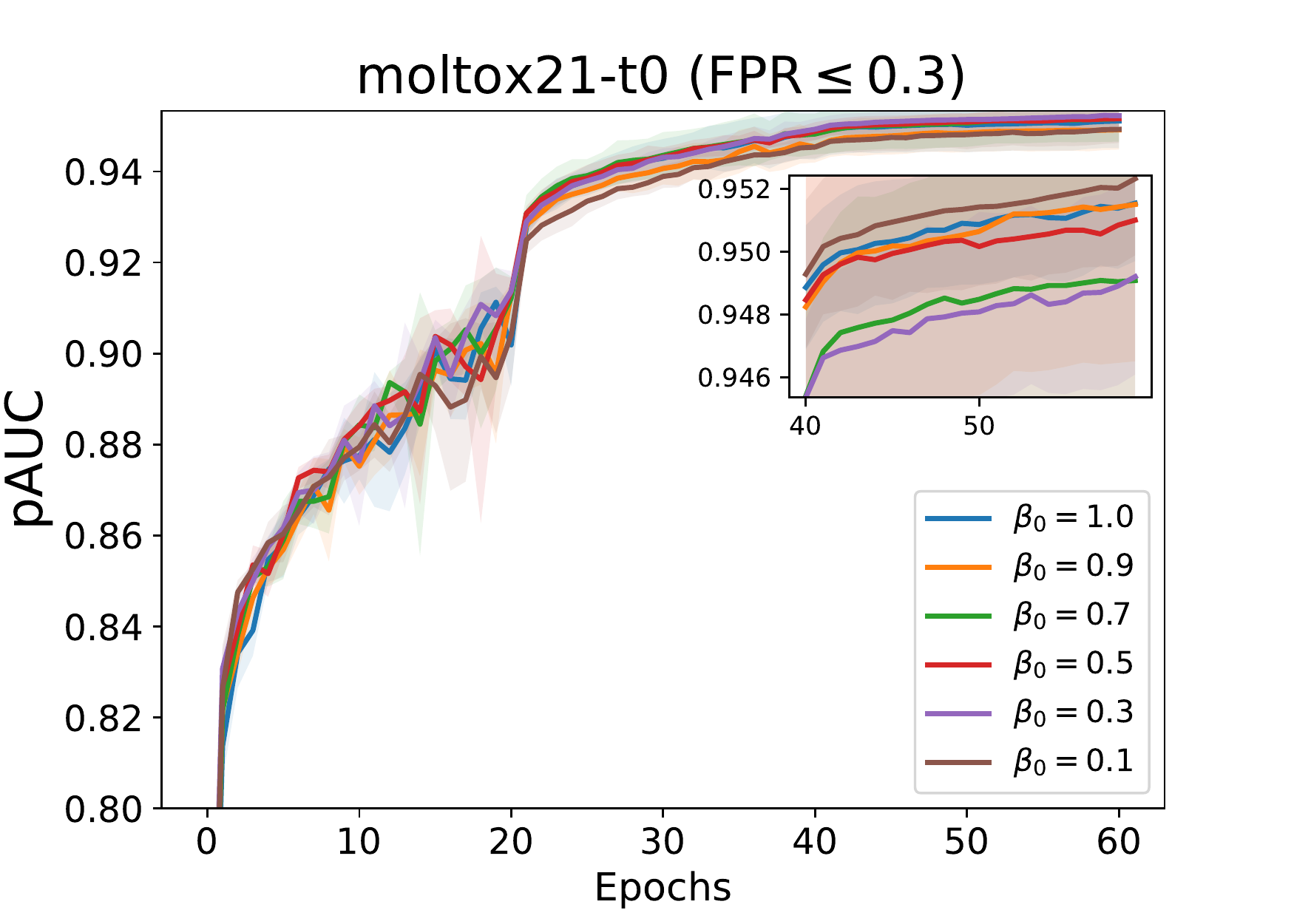}
    \includegraphics[width=0.24\linewidth]{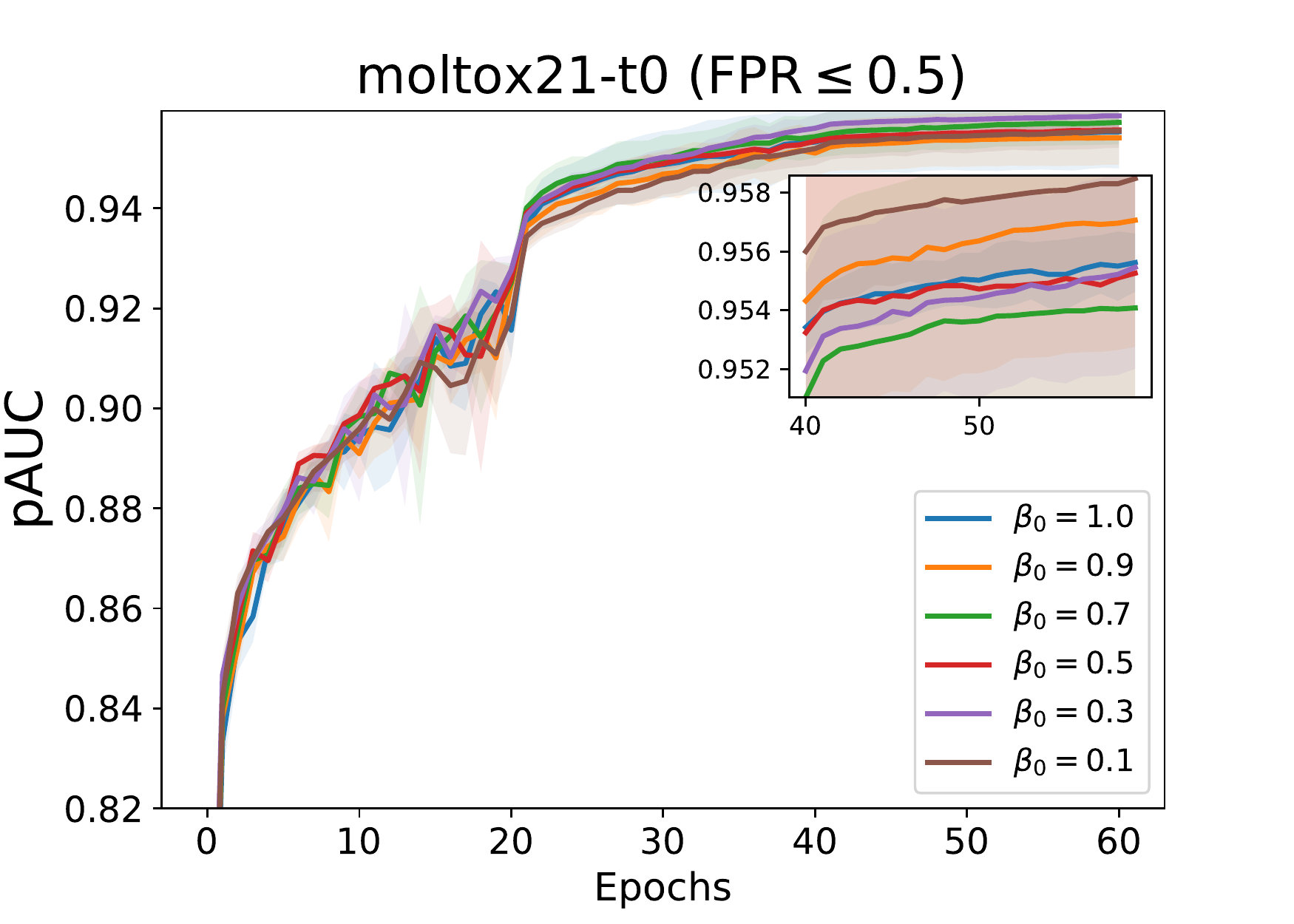}
    \includegraphics[width=0.24\linewidth]{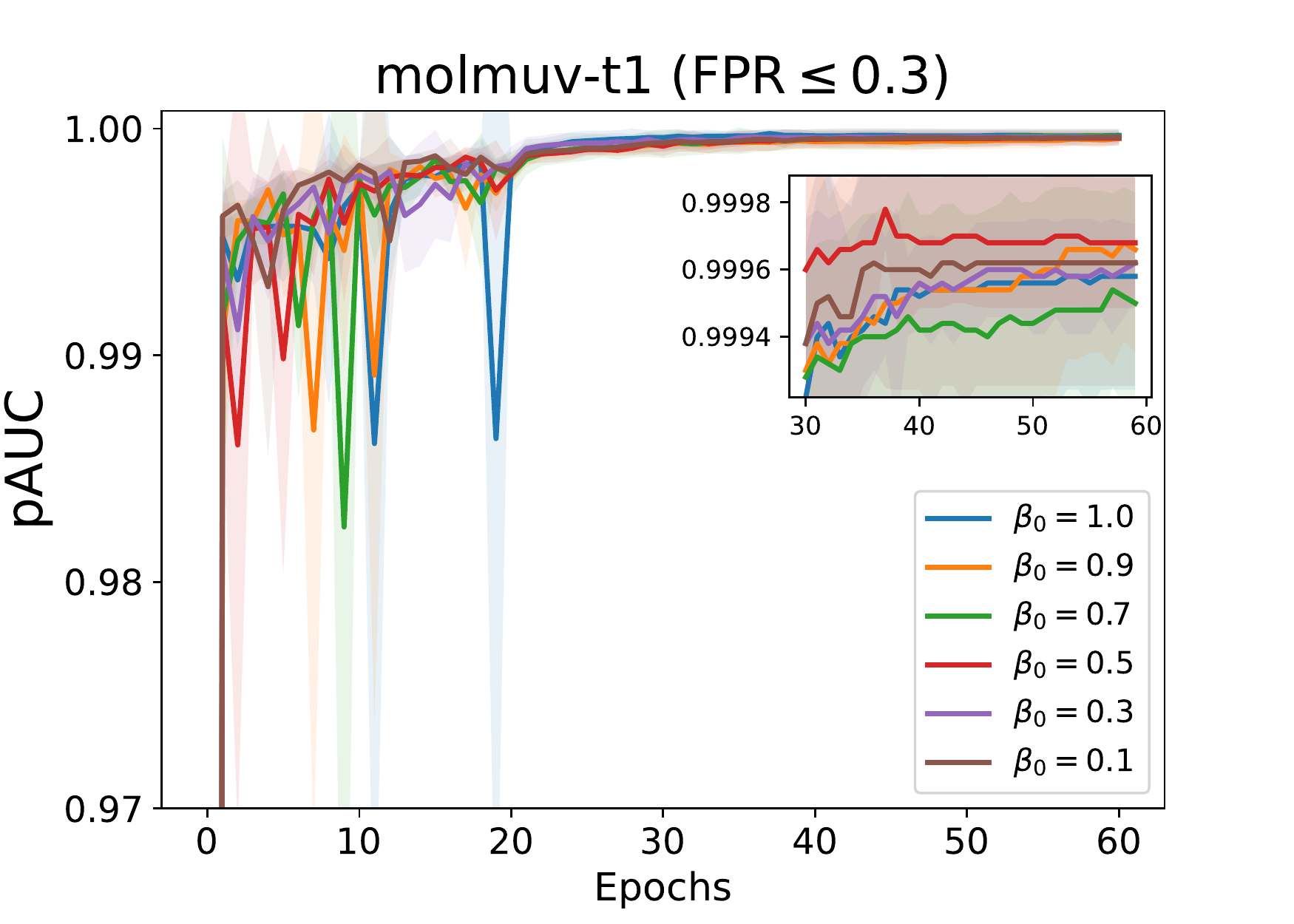}
    \includegraphics[width=0.24\linewidth]{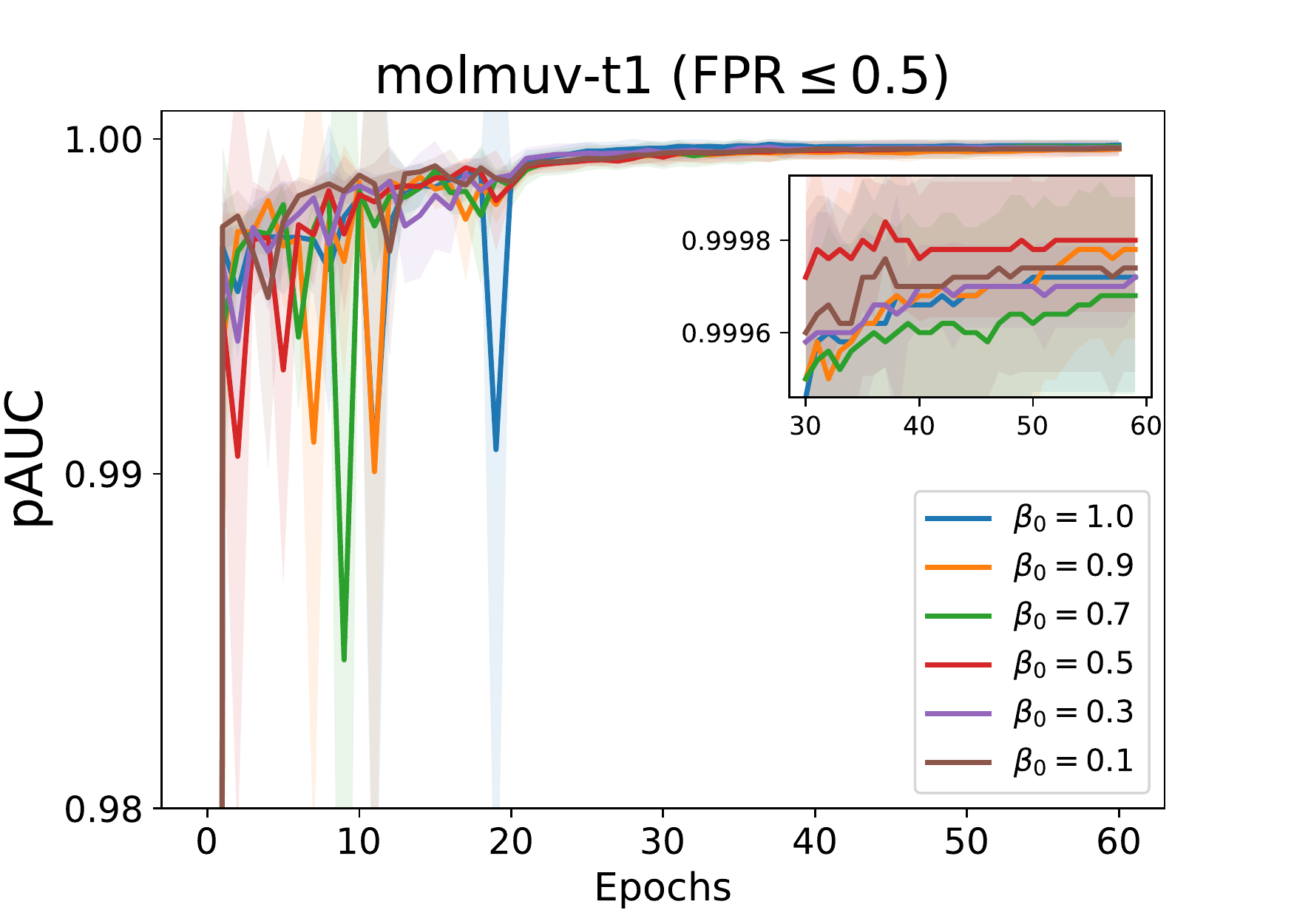}
    
\vspace{-0.1in}
\centering
\caption{Training convergence for SOPA-s on different $\gamma_0$.}
\label{fig:sopa-s-beta}
\vspace{-0.15in}
\end{figure*}

\begin{figure*}[h]
\centering
    \includegraphics[width=0.24\linewidth]{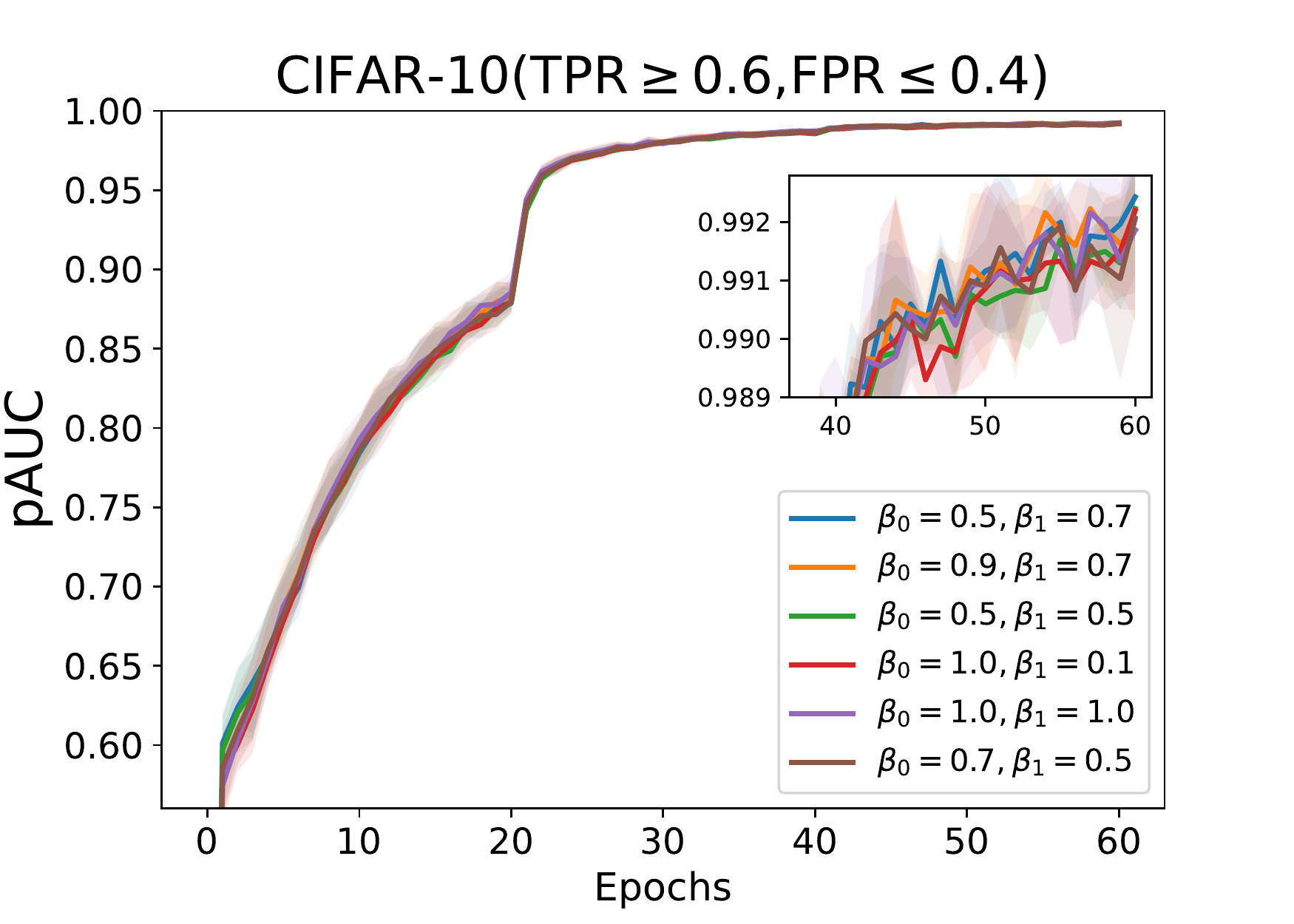}
    \includegraphics[width=0.24\linewidth]{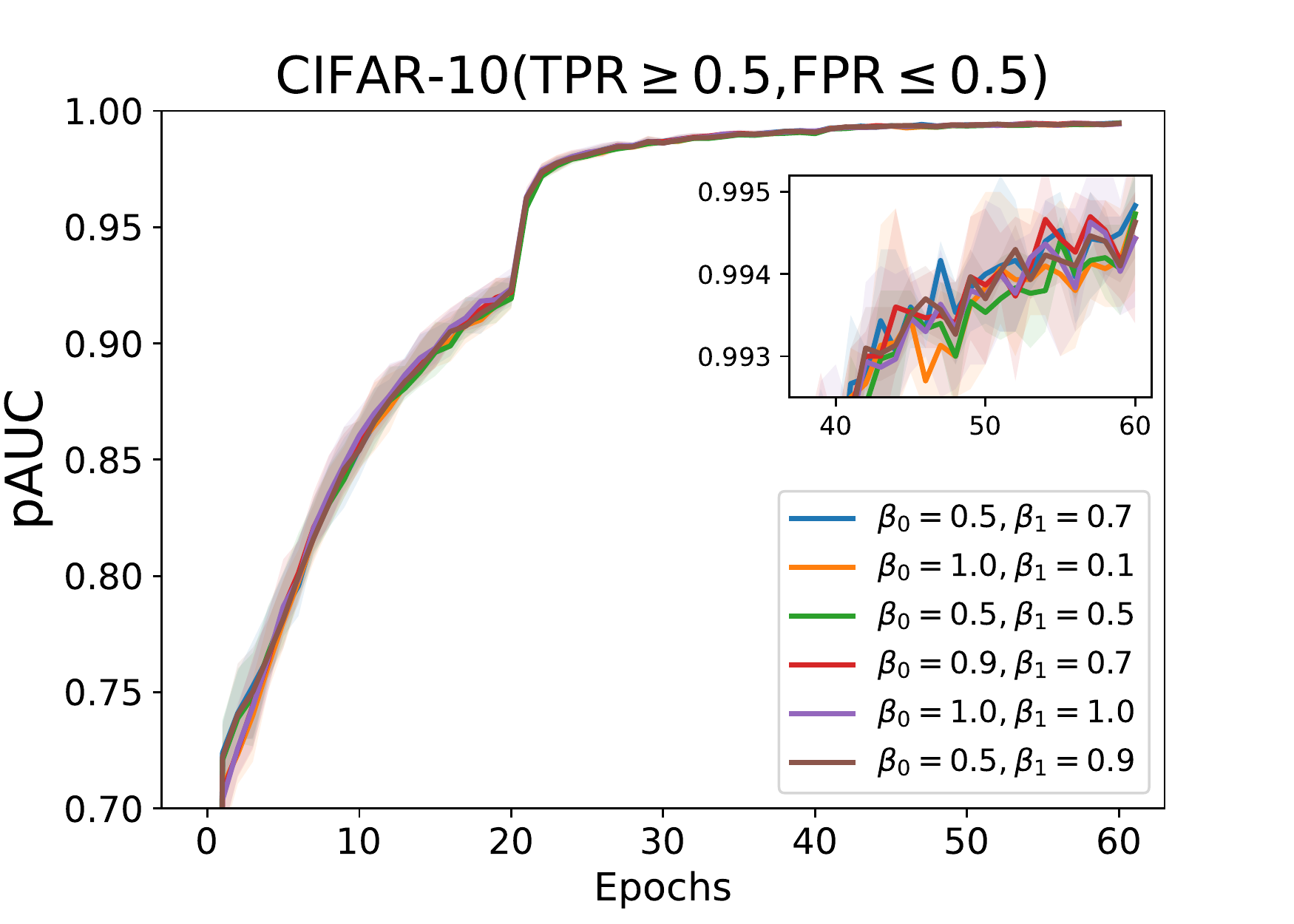}
    \includegraphics[width=0.24\linewidth]{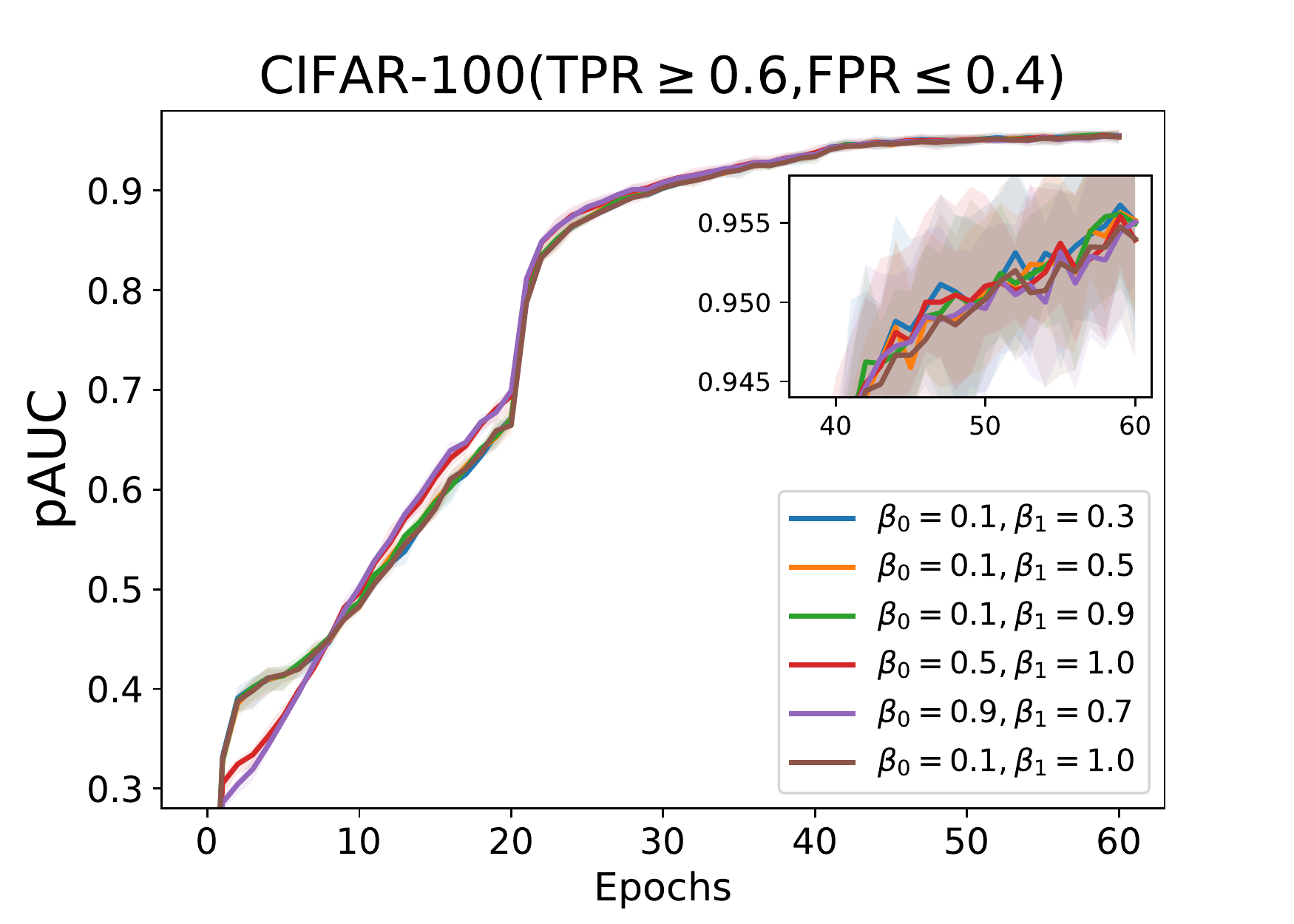}
    \includegraphics[width=0.24\linewidth]{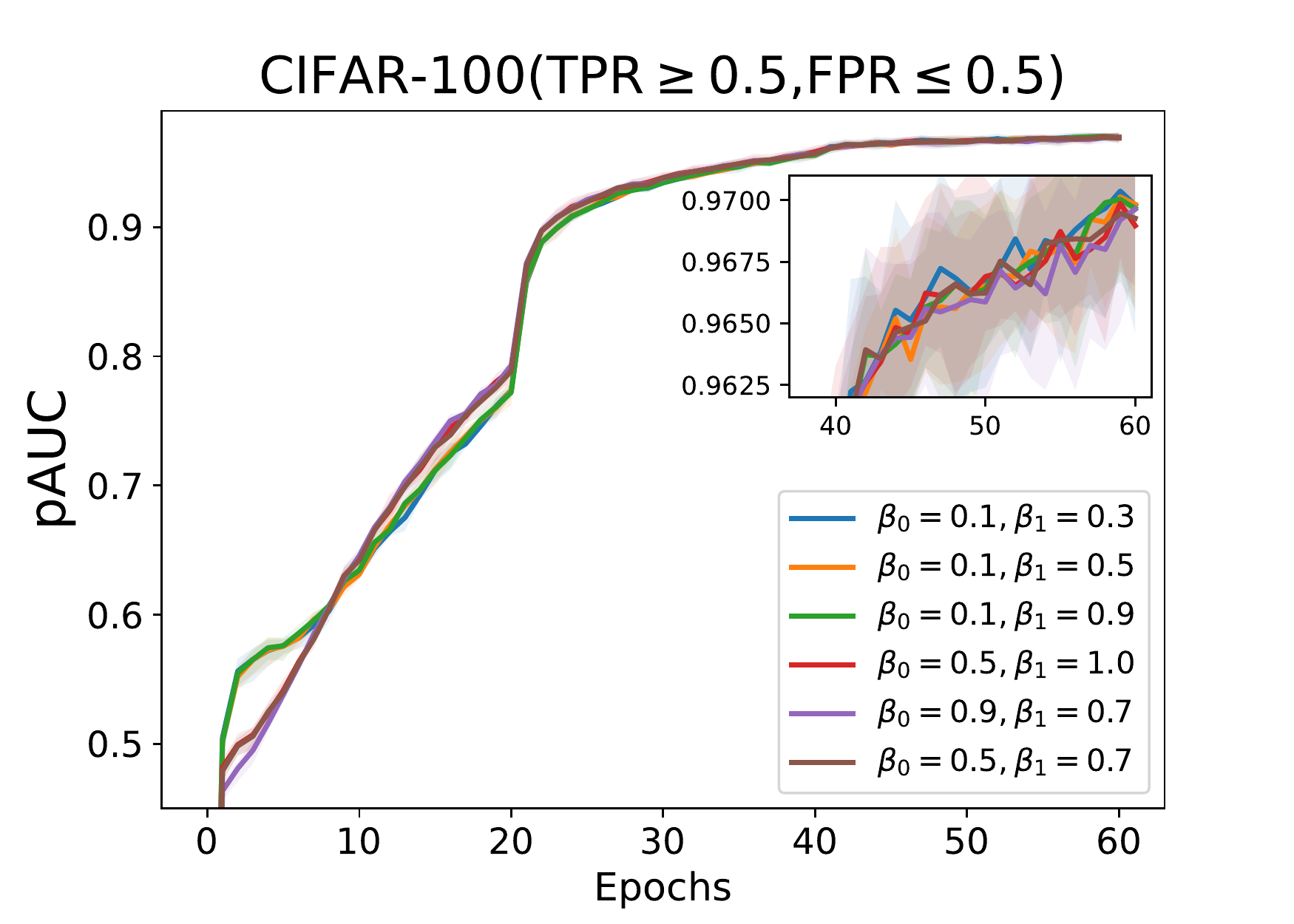}

    \includegraphics[width=0.24\linewidth]{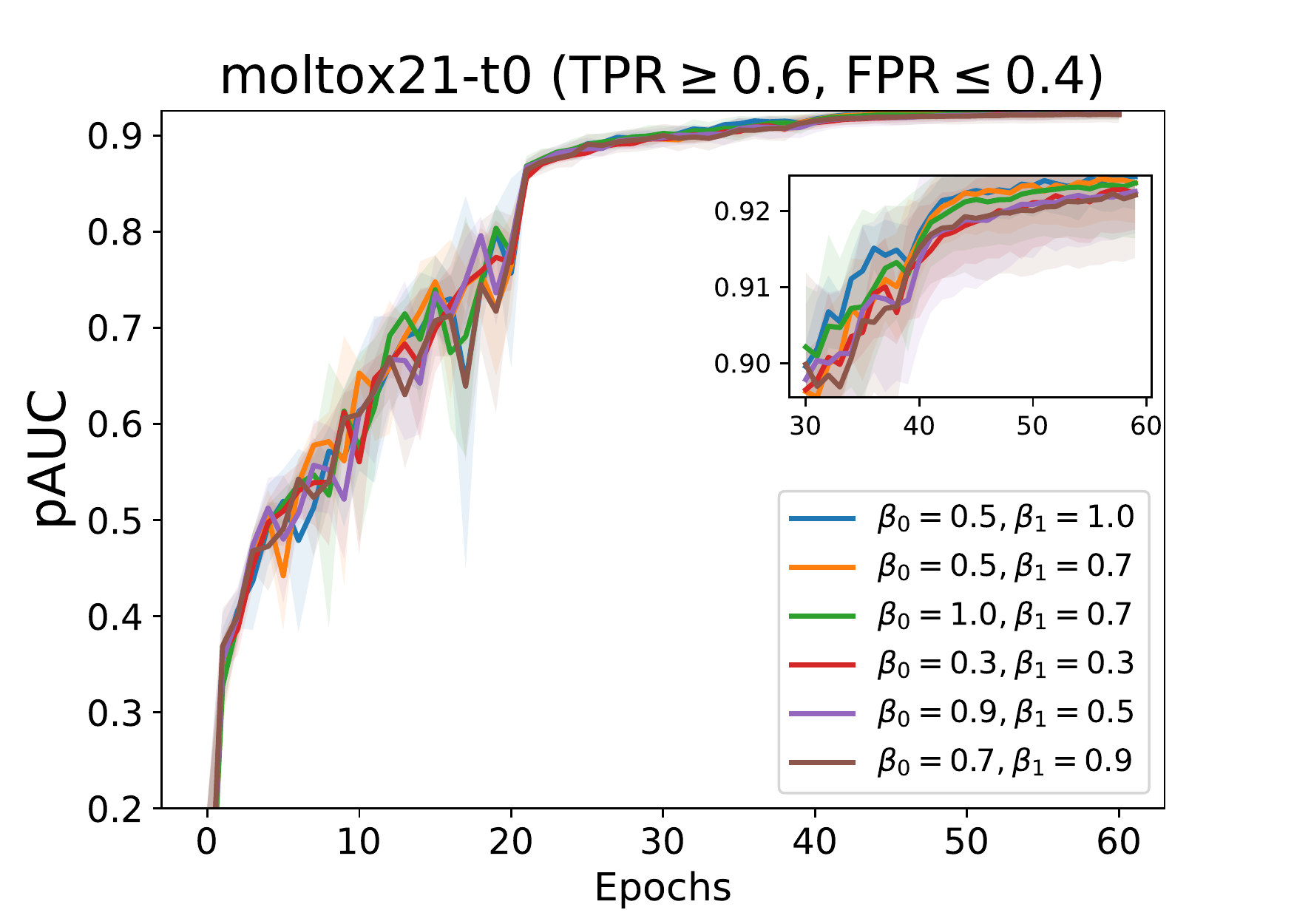}
    \includegraphics[width=0.24\linewidth]{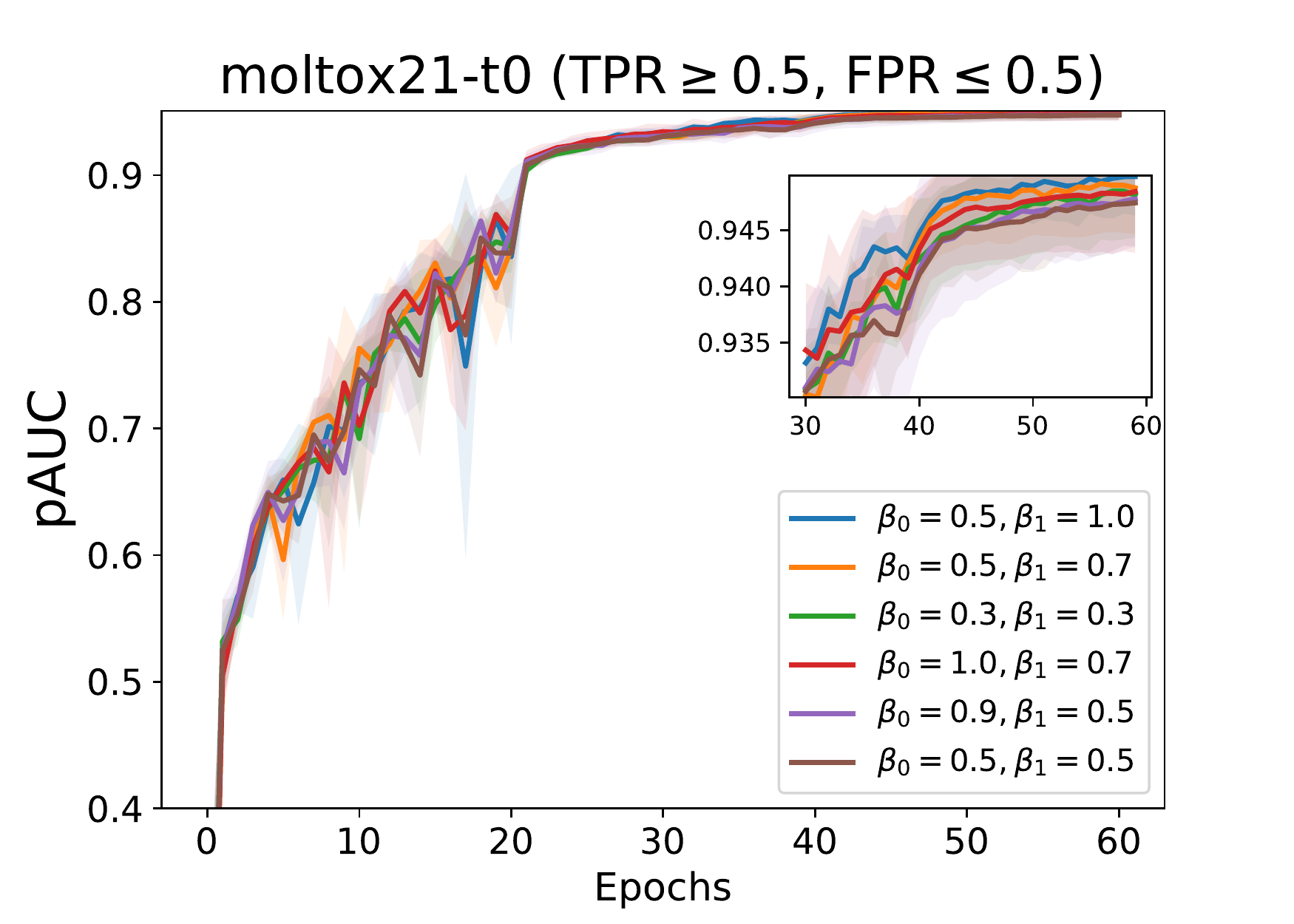}
    \includegraphics[width=0.24\linewidth]{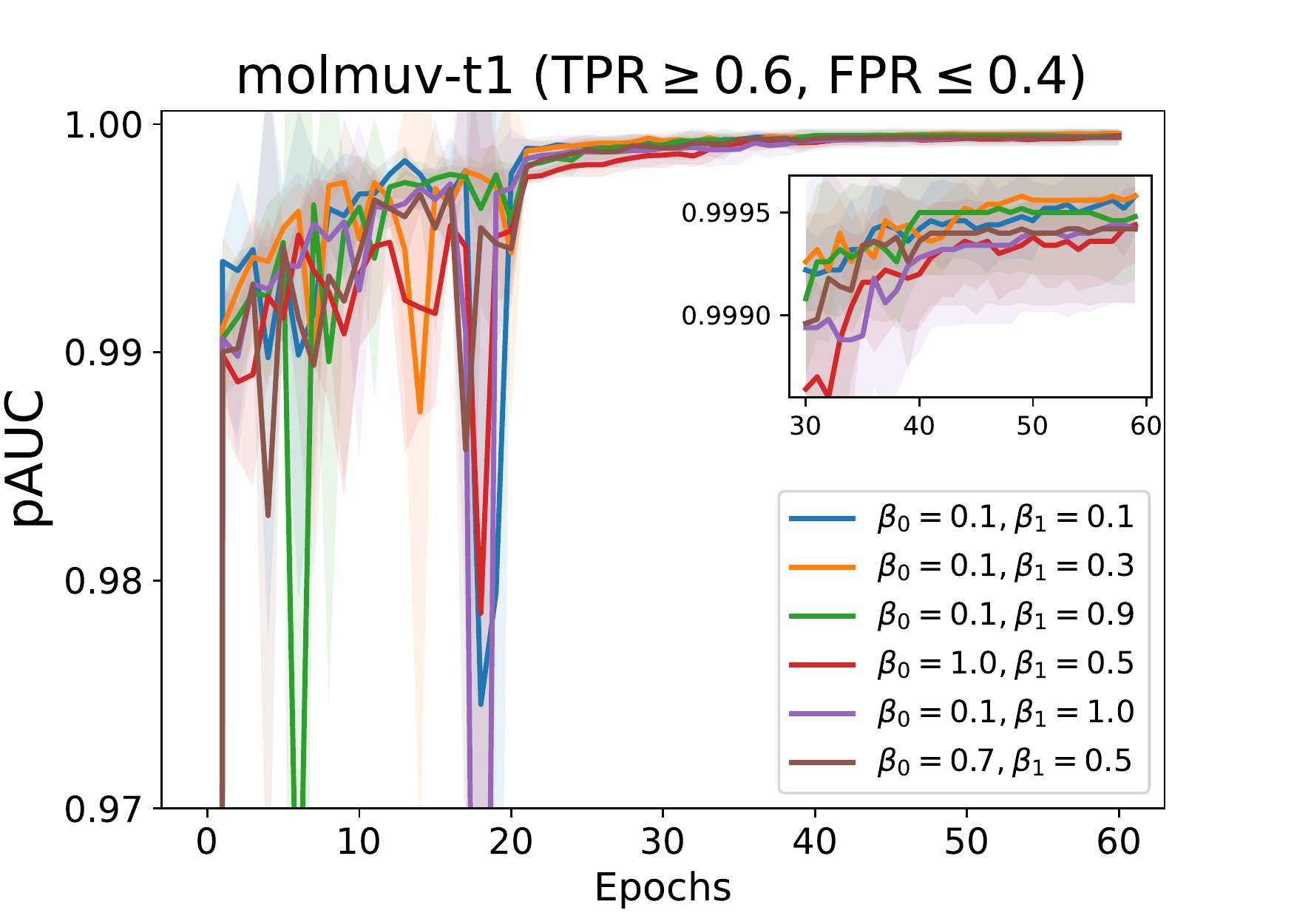}
    \includegraphics[width=0.24\linewidth]{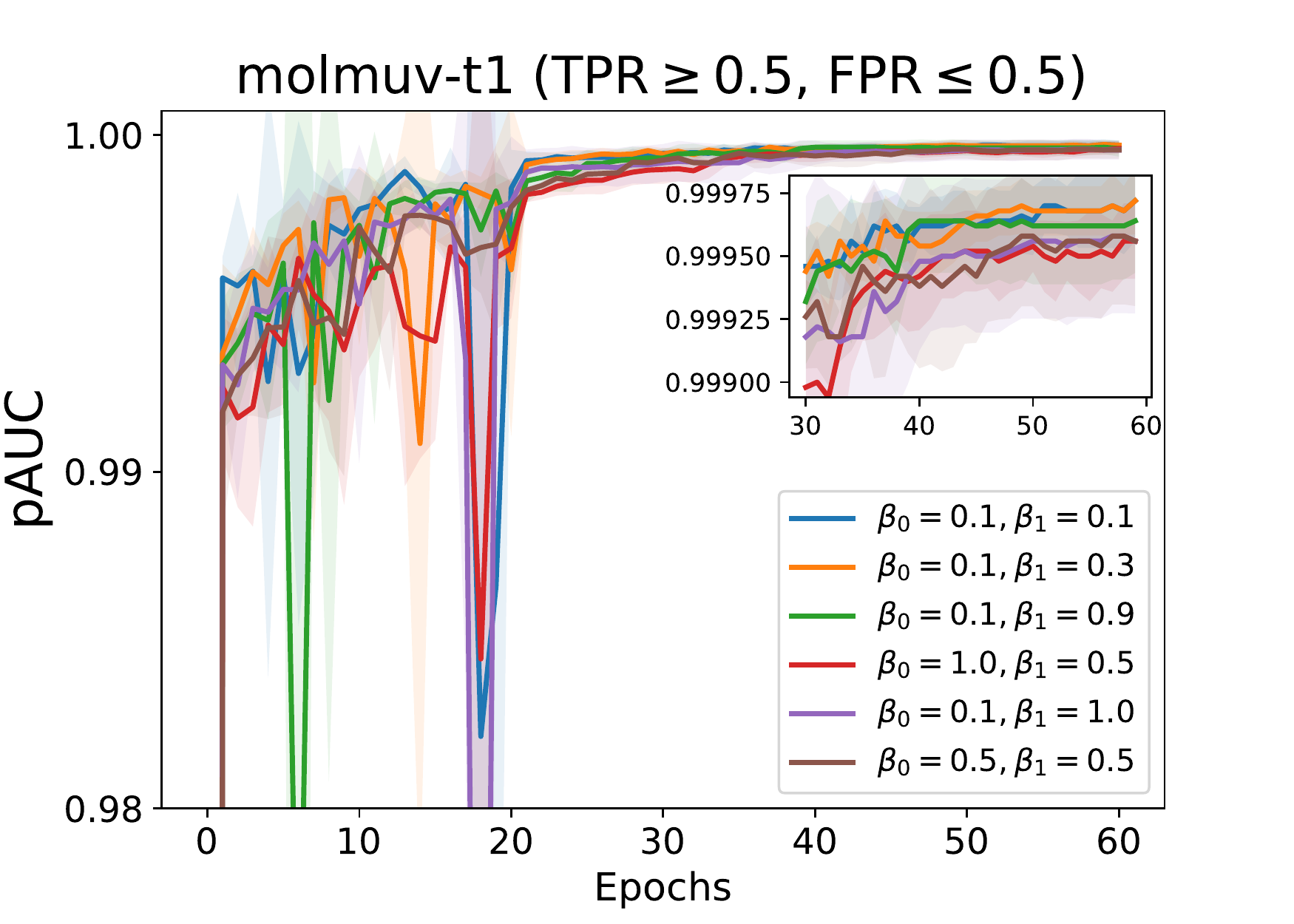}
    
\vspace{-0.1in}
\centering
\caption{Top-6 choices of $\gamma_0$ and $\gamma_1$ for SOTA-s from training perspective.}
\label{fig:sota-s-beta}
\end{figure*}

For testing aspect, we include the testing pAUC results at Table~\ref{tab:sopa-s-beta} for SOPA-s; Table~\ref{tab:sota-s-beta} for SOTA-s. Both verify that tuning these parameters can further improve the performance.  
\begin{table}[h!]
    \caption{Test pAUC for SOPA-s on different $\gamma_0$.}
    \label{tab:sopa-s-beta}
    \centering
    \resizebox{1.0\textwidth}{!}{
    \begin{tabular}{llllllll}
       Dataset& Metric & $\gamma_0=1.0$& $\gamma_0=0.9$ & $\gamma_0=0.7$ & $\gamma_0=0.5$ & $\gamma_0=0.3$ & $\gamma_0=0.1$\\
        \hline
        CIFAR-10 & FPR$\le$0.3 &0.8721(0.0049)& 0.8691(0.0036)& 0.8682(0.0048)& 0.8697(0.0032)& 0.8674(0.0045)& \textbf{0.8725(0.0012)}\\
         & FPR$\le$0.5 & 0.8989(0.0051)& 0.8961(0.0036)& 0.8946(0.0040)& 0.8980(0.0037)& 0.8947(0.0028)&\textbf{0.8996(0.0021)}\\
        \hline
        CIFAR-100 & FPR$\le$0.3 & 0.7464(0.0012)& 0.7460(0.0068)& 0.7482(0.0031)& 0.7508(0.0038)& 0.7494(0.0048)& \textbf{0.7514(0.0018)}\\
         & FPR$\le$0.5 & 0.7888(0.0016)& 0.7877(0.0053)& 0.7936(0.0040)& 0.7961(0.0063)& 0.7922(0.0059)&\textbf{0.7954(0.0015)}\\
         \hline
        moltox21 & FPR$\le$0.3 & 0.7242(0.0170)& 0.7288(0.0125)& 0.7274(0.0087)& 0.7274(0.0062)& 0.7158(0.0069)& \textbf{0.7340(0.0079)}\\
         & FPR$\le$0.5 & 0.7245(0.0194)& 0.7266(0.0111)& 0.7280(0.0066)& 0.7358(0.0079)& 0.7249(0.0090)&\textbf{0.7360(0.0123)}\\
        \hline
        molmuv & FPR$\le$0.3 & 0.8692(0.0116)& 0.8376(0.0340)& 0.8642(0.0214)& \textbf{0.8735(0.0070)}& 0.8732(0.0104)& 0.8496(0.0392)\\
         & FPR$\le$0.5 & 0.8773(0.0186)& 0.8616(0.0355)& 0.8747(0.0053)& \textbf{0.8996(0.0228)}& 0.8918(0.0232)&0.8622(0.0415)\\
    \end{tabular}}
\end{table}
\newpage
\begin{table}[h!]
    \caption{Test pAUC for SOTA-s on different $\gamma_0$ and $\gamma_1$.}
    \label{tab:sota-s-beta}
    \centering
    \resizebox{1.0\textwidth}{!}{
    \begin{tabular}{llllllll}
     CIFAR-10&(0.4, 0.6) & $\gamma_1=1.0$& $\gamma_1=0.9$ & $\gamma_1=0.7$ & $\gamma_1=0.5$ & $\gamma_1=0.3$ & $\gamma_1=0.1$\\
        \hline
        &$\gamma_0=1.0$ & 0.5731(0.0069) & 0.5744(0.0083) & 0.5668(0.0048) & 0.5686(0.0089) & 0.5678(0.0070) & 0.5739(0.0135) \\
        &$\gamma_0=0.9$ & 0.5802(0.0081) & \textbf{0.5839(0.0116)} & 0.5693(0.0072) & 0.5743(0.0131) & 0.5688(0.0048) & 0.5760(0.0095) \\
        &$\gamma_0=0.7$ & 0.5755(0.0098) & 0.5715(0.0032) & 0.5684(0.0131) & 0.5797(0.0036) & 0.5718(0.0095) & 0.5695(0.0111) \\
        &$\gamma_0=0.5$ & 0.5743(0.0058) & 0.5706(0.0025) & 0.5641(0.0094) & 0.5819(0.0077) & 0.5725(0.0109) & 0.5739(0.0068) \\
        &$\gamma_0=0.3$ & 0.5767(0.0121) & 0.5638(0.0144) & 0.5617(0.0075) & 0.5644(0.0099) & 0.5595(0.0028) & 0.5768(0.0100) \\
        &$\gamma_0=0.1$ & 0.5608(0.0075) & 0.5689(0.0134) & 0.5699(0.0096) & 0.5685(0.0056) & 0.5583(0.0101) & 0.5756(0.0135) \\

         \midrule
         CIFAR-10&(0.5, 0.5) & $\gamma_1=1.0$& $\gamma_1=0.9$ & $\gamma_1=0.7$ & $\gamma_1=0.5$ & $\gamma_1=0.3$ & $\gamma_1=0.1$\\
        \hline
        &$\gamma_0=1.0$ & 0.7022(0.0054) & 0.6999(0.0076) & 0.6922(0.0076) & 0.6969(0.0087) & 0.6982(0.0046) & 0.6981(0.0074) \\
        &$\gamma_0=0.9$ & \textbf{0.7051(0.0083)} & 0.7047(0.0072) & 0.6987(0.0034) & 0.7008(0.0103) & 0.6964(0.0048) & 0.7003(0.0064) \\
        &$\gamma_0=0.7$ & 0.7027(0.0046) & 0.6999(0.0031) & 0.6959(0.0099) & 0.7043(0.0056) & 0.6988(0.0099) & 0.6914(0.0110) \\
        &$\gamma_0=0.5$ & 0.7012(0.0050) & 0.6988(0.0008) & 0.6944(0.0067) & 0.7044(0.0067) & 0.7014(0.0047) & 0.7031(0.0088) \\
        &$\gamma_0=0.3$ & 0.7024(0.0099) & 0.6940(0.0085) & 0.6901(0.0060) & 0.6981(0.0047) & 0.6908(0.0025) & 0.6977(0.0129) \\
        &$\gamma_0=0.1$ & 0.6934(0.0068) & 0.7009(0.0060) & 0.6972(0.0088) & 0.6946(0.0037) & 0.6891(0.0098) & 0.7027(0.0082) \\

         \hline
         
     CIFAR-100&(0.4, 0.6) & $\gamma_1=1.0$& $\gamma_1=0.9$ & $\gamma_1=0.7$ & $\gamma_1=0.5$ & $\gamma_1=0.3$ & $\gamma_1=0.1$\\
        \hline
        &$\gamma_0=1.0$ & 0.2710(0.0144) & 0.2650(0.0052) & 0.2528(0.0061) & 0.2573(0.0129) & 0.2587(0.0096) & 0.2613(0.0046) \\
        &$\gamma_0=0.9$ & 0.2666(0.0144) & 0.2708(0.0055) & 0.2698(0.0117) & 0.2586(0.0116) & 0.2562(0.0032) & 0.2663(0.0056) \\
        &$\gamma_0=0.7$ & 0.2634(0.0086) & 0.2624(0.0065) & 0.2577(0.0045) & 0.2664(0.0051) & 0.2634(0.0090) & 0.2626(0.0144) \\
        &$\gamma_0=0.5$ & 0.2633(0.0010) & 0.2591(0.0026) & 0.2552(0.0068) & 0.2530(0.0069) & 0.2656(0.0089) & 0.2594(0.0052) \\
        &$\gamma_0=0.3$ & 0.2572(0.0063) & 0.2542(0.0076) & 0.2517(0.0111) & 0.2599(0.0184) & 0.2580(0.0191) & 0.2631(0.0081) \\
        &$\gamma_0=0.1$ & 0.2532(0.0050) & \textbf{0.2745(0.0029)} & 0.2529(0.0060) & 0.2573(0.0115) & 0.2568(0.0029) & 0.2671(0.0187) \\

         \midrule
         CIFAR-100&(0.5, 0.5) & $\gamma_1=1.0$& $\gamma_1=0.9$ & $\gamma_1=0.7$ & $\gamma_1=0.5$ & $\gamma_1=0.3$ & $\gamma_1=0.1$\\
        \hline
        &$\gamma_0=1.0$& 0.4489(0.0122) & 0.4454(0.0084) & 0.4337(0.0097) & 0.4418(0.0083) & 0.4361(0.0093) & 0.4405(0.0063) \\
        &$\gamma_0=0.9$& 0.4416(0.0115) & \textbf{0.4528(0.0069)} & 0.4494(0.0073) & 0.4449(0.0171) & 0.4337(0.0046) & 0.4421(0.0058) \\
        &$\gamma_0=0.7$& 0.4426(0.0119) & 0.4426(0.0098) & 0.4367(0.0092) & 0.4414(0.0072) & 0.4455(0.0046) & 0.4430(0.0108) \\
        &$\gamma_0=0.5$& 0.4404(0.0058) & 0.4388(0.0085) & 0.4341(0.0066) & 0.4332(0.0079) & 0.4426(0.0107) & 0.4421(0.0057) \\
        &$\gamma_0=0.3$& 0.4384(0.0090) & 0.4313(0.0081) & 0.4377(0.0104) & 0.4392(0.0185) & 0.4401(0.0095) & 0.4384(0.0125) \\
        &$\gamma_0=0.1$& 0.4325(0.0006) & 0.4500(0.0015) & 0.4361(0.0085) & 0.4316(0.0083) & 0.4355(0.0053) & 0.4461(0.0213) \\
         \hline

    moltox21 &(0.4, 0.6)  & $\gamma_1=1.0$& $\gamma_1=0.9$ & $\gamma_1=0.7$ & $\gamma_1=0.5$ & $\gamma_1=0.3$ & $\gamma_1=0.1$\\
    \hline
    &$\gamma_0=1.0$ & 0.0769(0.0403) & 0.0800(0.0345) & 0.0594(0.0377) & 0.0409(0.0241) & 0.0774(0.0342) & 0.0607(0.0139) \\
    &$\gamma_0=0.9$ & 0.0668(0.0304) & 0.0733(0.0198) & 0.0676(0.0357) & 0.0761(0.0316) & 0.0737(0.0245) & 0.0588(0.0249) \\
    &$\gamma_0=0.7$ & 0.0657(0.0280) & 0.0849(0.0451) & 0.0840(0.0162) & 0.0807(0.0253) & 0.0665(0.0171) & 0.0573(0.0262) \\
    &$\gamma_0=0.5$ & 0.0795(0.0537) & 0.0513(0.0140) & 0.0704(0.0301) & 0.0367(0.0164) & 0.0754(0.0190) & 0.0676(0.0350) \\
    &$\gamma_0=0.3$ & \textbf{0.1009(0.0133)} & 0.0695(0.0256) & 0.0806(0.0289) & 0.0805(0.0358) & 0.0943(0.0453) & 0.0546(0.0409) \\
    &$\gamma_0=0.1$ & 0.0746(0.0473) & 0.0610(0.0224) & 0.0694(0.0303) & 0.0578(0.0458) & 0.0768(0.0574) & 0.0485(0.0254) \\
  
\hline

    moltox21 &(0.5, 0.5) & $\gamma_1=1.0$& $\gamma_1=0.9$ & $\gamma_1=0.7$ & $\gamma_1=0.5$ & $\gamma_1=0.3$ & $\gamma_1=0.1$\\
    \hline
    &$\gamma_0=1.0$ & 0.2474(0.0154) & 0.2489(0.0282) & 0.2355(0.0508) & 0.2127(0.0326) & \textbf{0.2798(0.0640)} & 0.2254(0.0223) \\
    &$\gamma_0=0.9$ & 0.2483(0.0214) & 0.2476(0.0222) & 0.2373(0.0552) & 0.2457(0.0296) & 0.2341(0.0137) & 0.2198(0.0233) \\
    &$\gamma_0=0.7$ & 0.2306(0.0286) & 0.2354(0.0360) & 0.2531(0.0217) & 0.2546(0.0174) & 0.2415(0.0181) & 0.2233(0.0225) \\
    &$\gamma_0=0.5$ & 0.2536(0.0531) & 0.2251(0.0163) & 0.2276(0.0241) & 0.2153(0.0196) & 0.2488(0.0121) & 0.2230(0.0228) \\
    &$\gamma_0=0.3$ & 0.2773(0.0207) & 0.2482(0.0256) & 0.2434(0.0224) & 0.2603(0.0323) & 0.2574(0.0455) & 0.2208(0.0494) \\
    &$\gamma_0=0.1$ & 0.2368(0.0193) & 0.2395(0.0267) & 0.2323(0.0239) & 0.1973(0.0383) & 0.2543(0.0611) & 0.2121(0.0369) \\

         \hline
       molmuv&(0.4, 0.6) & $\gamma_1=1.0$& $\gamma_1=0.9$ & $\gamma_1=0.7$ & $\gamma_1=0.5$ & $\gamma_1=0.3$ & $\gamma_1=0.1$\\
        \hline
         &$\gamma_0=1.0$ &0.5021(0.1222) & 0.4508(0.1263)& 0.3878(0.0776)& 0.5399(0.1575)& 0.3972(0.0768)& 0.3904(0.1404)\\
         &$\gamma_0=0.9$ &0.5651(0.0833) & 0.4524(0.0290)& \textbf{0.6060(0.2187)}& 0.5711(0.1495)& 0.4250(0.0543)& 0.4313(0.1174)\\
         &$\gamma_0=0.7$ & 0.4120(0.1026)& 0.4338(0.0407)&0.5375(0.1339) &0.5765(0.1719) &0.3251(0.0889) &0.4744(0.1379)\\
         &$\gamma_0=0.5$ &0.5432(0.1177) &0.4934(0.1065) &0.4812(0.0744) &0.3890(0.0976) &0.4507(0.0387) &0.3859(0.0488)\\
         &$\gamma_0=0.3$ &0.4680(0.1079) &0.4559(0.1037) &0.4811(0.1237) &0.4221(0.1322) &0.4297(0.2263) &0.3828(0.0606)\\
         &$\gamma_0=0.1$ &0.4777(0.0591) &0.4264(0.1414) &0.5322(0.1342) &0.4122(0.0260) &0.3836(0.0368) &0.3903(0.0325)\\
         \midrule
         molmuv&(0.5, 0.5) & $\gamma_1=1.0$& $\gamma_1=0.9$ & $\gamma_1=0.7$ & $\gamma_1=0.5$ & $\gamma_1=0.3$ & $\gamma_1=0.1$\\
        \hline
         &$\gamma_0=1.0$ &0.6780(0.0880) & 0.6384(0.0964)& 0.6175(0.0447)& 0.7135(0.1074)& 0.6188(0.0483)& 0.6152(0.0965)\\
         &$\gamma_0=0.9$ &0.7405(0.0621) & 0.6399(0.0276)& \textbf{0.7481(0.1464)}& 0.7230(0.1020)& 0.6136(0.0321)& 0.6277(0.0854)\\
         &$\gamma_0=0.7$ & 0.6201(0.0698)& 0.6445(0.0430)&0.7007(0.1013) &0.7272(0.1178) &0.5588(0.0482) &0.6496(0.1030)\\
         &$\gamma_0=0.5$ &0.7071(0.0869) &0.6854(0.0853) &0.6542(0.0646) &0.5867(0.0657) &0.6495(0.0470) &0.5996(0.0431)\\
         &$\gamma_0=0.3$ &0.6626(0.0737) &0.6561(0.0778) &0.6658(0.0910) &0.6408(0.0825) &0.6264(0.1791) &0.5801(0.0399)\\
         &$\gamma_0=0.1$ &0.6643(0.0562) &0.6362(0.1103) &0.7142(0.0904) &0.6271(0.0188) &0.5746(0.0283) &0.5877(0.0449)\\
    \end{tabular}}
\end{table}

\section{Proofs}
We next present several lemmas. The first lemma is straightforward. 
\begin{lemma}\label{lem:1}
For $\hat L_{kl}(\cdot; \lambda)$,  when $\lambda=0$ it reduces to the maximal value of $\{\ell_1(\cdot), \ldots, \ell_n(\cdot)\}$, i.e., $\hat L_{kl}(\cdot, 0) = \max_i\ell_i(\cdot)$; and when $\lambda=\infty$, it reduces to the average value of $\{\ell_1, \ldots, \ell_n\}$, i.e., $\hat L_{kl}(\cdot; \infty) = \frac{1}{n}\sum_{i=1}^n\ell_i(\cdot)$. 
\end{lemma}

\begin{lemma}\label{lem:2}[Lemma 1~\cite{minsumtopk2003}]
Assume $\gamma = k/n$ for some integer $k\in[n]$, we have
    $\hat L_{cvar}(\cdot;\gamma) = \min_{s} s + \frac{1}{n\gamma}\sum_{i=1}^n[\ell_i(\cdot)-s]_+$.
\end{lemma}
A major difference between $\hat L_{kl}(\cdot; \lambda)$ and $\hat L_{cvar}(\cdot; \gamma)$ that has an impact on optimization is that $\hat L_{cvar}(\cdot; \gamma)$ is a non-smooth function, and $\hat L_{kl}(\cdot; \lambda)$ is a smooth function for $\lambda>0$ when $\ell_i(\cdot)$ is smooth and bounded as indicated by the following lemma. 
\begin{lemma}\label{lem:3}
If $\ell_i(\cdot)$ is a smooth function, and has a bounded value and bounded gradient for a bounded input, then for $\lambda>0$ the function $\hat L_{kl}(\cdot; \lambda)$ is also a smooth and bounded function. 
\end{lemma}

\begin{lemma}\label{lem:tpauc}
When $\phi_c(t)=\I(0\leq t\leq 1/\beta)$ and $\phi'_c(t)=\I(0\leq t\leq 1/\alpha)$ with $K_2=n_-\beta$ and $K_1=n_+\alpha$ being integers,  if $\ell(\cdot)$ is monotonically decreasing for $\ell(\cdot)>0$, we can show that 
\begin{align*}
&F(\w; \phi_c, \phi_c')=\frac{1}{K_1K_2}\sum_{\x_i\in\S^\uparrow_+[1,K_1]}\sum_{\x_j\in\S^\downarrow_-[1,K_2]}L(\w; \x_i, \x_j),
\end{align*}
which is also equivalent to 
\begin{align*}
   &F(\w; \phi_c, \phi_c')=\min_{s'\in\R, \s\in\R^{n_+}}s' +\frac{1}{n_+\alpha}\sum_{\x_i\in\S_+}(s_i+\frac{1}{\beta}\psi_i(\w; s_i) - s')_+.
\end{align*}
\end{lemma}

\subsection{Proof of Lemma~\ref{lem:3}}
It is easy to see if $\ell(\w)\in[0, C]$ is bounded and smooth, we have $\exp(\frac{\ell(\w)}{\lambda})$ is bounded and smooth due to its second order gradient is upper bounded. Then $\log\E_i\exp(\frac{\ell_i(\w)}{\lambda})$ is bounded. Its smoothness due to that is a composition of $f=\log(\cdot)$ and $g=\E_i\exp(\frac{\ell_i(\w)}{\lambda})\in[1, C']$ and both $f, g$ are smooth and Lipschitz continuous for their inputs.

\subsection{Proof of Lemma~\ref{lem:tpauc}}
\begin{proof}
First, following Lemma~\ref{lem:0}, we have $F(\w; \phi, \phi')$ is equivalent to $\frac{1}{K_1}\sum_{i=1}^{K_1}\hat L_{\phi}(\x_{\pi_i},\w)$, where $\pi_i$ denote the index of the positive example whose $\hat L_{\phi}(\x_{\pi_i},\w)$ is the $i$-th largest among all positive examples.  We prove that this is equivalent to $\frac{1}{K_1}\sum_{\x_i\in\S^\uparrow_+[1,K_1]} \hat L_{\phi}(\x_i, \w)$. To this end, we just need to show that if $h_\w(\x)\geq h_\w(\x')$ then $\hat L_\phi(\x
', \w)\geq \hat L_\phi(\x, \w)$, which is true due to   $\hat L_\phi(\x_i, \w)=\frac{1}{K_2}\sum_{\x_j\in\S^\downarrow_-[1,K_2]}\ell(h_\w(\x_i) - h_\w(\x_j))$ and $\ell$ is monotonically decreasing function. The second equation in the lemma is applying Lemma~\ref{lem:2} twice and by noting that $(\min_x f(x) -s)_+=\min_x (f(x)-s)_+$. 
\end{proof}

\subsection{Proof of Theorem~\ref{thm:cvar}}
\begin{proof}
Let us consider for a particular $\x_i\in\S_+$. When $\phi(\cdot)=\phi_c(\cdot)=\I(\cdot\in(0, 1/\beta])$, then $\hat L(\w; \x_i)$ becomes the CVaR estimator, i.e.,  the average of top $K=n_-\beta$ losses of $\ell(h_\w(\x_i) - h_\w(\x_j))$ among $\x_j\in\S_-$. Since $\ell(\cdot)$ is monotonically decreasing when $\ell(\cdot))>0$, the top $K=n_-\beta$ losses of $\ell(h_\w(\x_i) - h_\w(\x_j))$ among all $\x_j\in\S_-$  correspond to negative samples with top $K$ prediction scores. Hence, $\hat L_{\phi_c}(\w; \x_i) = \frac{1}{K}\sum_{\x_j\in\S^\downarrow_-[1,K]}\ell(h_\w(\x_i) - h_\w(\x_j))$. Then the equivalent problem in~(\ref{eqn:opauccvar}) follows from Lemma~\ref{lem:2}. 
\end{proof}

\subsection{Proof of Theorem~\ref{thm:klform}}
\begin{proof}
When choosing $\phi(\cdot)=\phi_{kl}(\cdot)$, the $\hat L_\phi(\w; \x_i)$ in  problem~(\ref{eqn:pauces}) becomes
\begin{align*}
 \max_{\p\in\Delta} \sum_{\x_j\in\S_-} p_jL(\w; \x_i, \x_j)- \lambda \sum_jp_j\log (np_j).
\end{align*}
With Karush-Kuhn-Tucker(KKT) conditions, it is not difficult to have $p_j^\star=\frac{\exp(L(\w;\x_i,\x_j)/\lambda}{\sum_j\exp(L(\w;\x_i,\x_j)/\lambda}$. By plugging in this back we obtain the claimed  objective~(\ref{eqn:pauceskl}). When $\lambda=0$, the above becomes the maximal one among $\{L(\w; \x_i, \x_j), \x_j\in\S_+\}$ for each $\x_i$. Then the object is $\frac{1}{n_+}\max_{\x_j\in\S_-}L(\w; \x_i, \x_j)$,  which is the surrogate of pAUC with FPR$\leq 1/n_-$. When $\lambda=\infty$, the above becomes the average of  $\{L(\w; \x_i, \x_j), \x_j\in\S_+\}$, which gives the standard surrogate of AUC. 
\end{proof}

\subsection{Proof of Lemma~\ref{lem:weak}}
First note that $F(\w, \s)= \frac{1}{n_+}\sum_{\x_i\in\S_+} \left(s_i  +  \frac{1}{\beta} g_i(\w, s_i)\right)$, and $g_i(\w, s_i) = \frac{1}{n_-} \sum_{\x_j\in \S_-}(L(\w; \x_i, \x_j) - s_i)_+$. 
We prove that $(L(\w; \x_i, \x_j) - s_i)_+$ is weakly convex in terms of $(\w, \s_i)$, i.e. there exists $\rho>0$ such that $(L(\w; \x_i, \x_j) - s_i)_++\frac{\rho}{2}\|\w\|^2 + \frac{\rho}{2}s_i^2$ is jointly convex in terms of $\w, \s$. To this end, let $\psi(\cdot) = [\cdot]_+$ which is convex and Lipchitz continuous, and $q(\w, \s_i) =L(\w; \x_i, \x_j) - s_i $, which is $L_s$-smooth function with respect to $(\w, s_i)$ due to that $L(\w; \x_i, \x_j)$ is $L_s$-smooth function. Then for any $\omega\in  \phi'(\psi(\w', s'_i))$ we have
\begin{align*}
    &\psi(q(\w, s_i))\geq \psi(q(\w', s'_i)) + \omega(q(\w, s_i) - q(\w', s'_i))\\
    &\geq  \psi(q(\w', s'_i)) + \omega (\nabla q(\w', s'_i) - \frac{L_s}{2}(\|\w-\w'\|^2+ |s_i-s'_i|^2))\\
    &\geq  \psi(q(\w', s'_i)) + \partial\psi(q(\w', s'_i)) - \frac{L_s}{2}(\|\w-\w'\|^2+ |s_i-s'_i|^2)
\end{align*}
where we use $0\leq \omega\leq 1$. The above inequality implies that $[L(\w; \x_i, \x_j) - s_i]_+$ is $L_s$-weakly convex in terms of $(\w, s_i)$~\cite{sgdweakly18}, i.e.,  $\frac{1}{n_-} \sum_{\x_j\in \S_-}\left\{(L(\w; \x_i, \x_j) - s_i)_+ +  \frac{L_s}{2}(\|\w\|^2+ |s_i|^2)\right\}$ is convex. As a result $\frac{1}{n_-} \sum_{\x_j\in \S_-}(L(\w; \x_i, \x_j) - s_i)_+ +  \frac{L_s}{2}(\|\w\|^2+ \|s_i\|^2)$ is jointly convex  in $(\w, s_i)$. Then $F(\w, \s)+ \frac{L_s}{2\beta}(\|\w\|^2+\sum_i|s_i|^2)$ is jointly convex in terms of $(\w, \s)$. 

\subsection{Proof of Theorem~\ref{thm:sopa}}
Let $F(\w, \s)$ denote the objective function, and let $\v=(\w, \s)$. Define $F_{1/\hrho}(\v)=\min_{\u}F(\u) + \frac{\hrho}{2}\|\u - \v\|^2$ for some $\hrho>\rho$ and the minimizer is denoted by $\text{prox}_{F/\hrho}(\v)$.   Let $\vh_t =\text{prox}_{F/\hrho}(\v_t)$.  Define $\|\v - \v'\|^2 = \|\w- \w'\|^2 +  \|\s-\s'\|^2$. 
\begin{align*}
&\E_t[F_{1/\hrho}(\v_{t+1})] = \E_t[F(\vh_t) +\frac{\hrho}{2}\|\v_{t+1} - \vh_t\|^2]\\
&\leq F(\vh_t) + \frac{\hrho}{2}\E_t[\|\w_t - \eta_1\nabla_{\w} F(\w_t, \s_t, \xi_t) - \wh_t\|^2+\|\s_{t+1} - \hat\s_t\|^2]\\
&\leq F(\vh_t) + \frac{\hrho}{2}\E_t[\|\w_t - \eta_1\nabla_{\w} F(\w_t, \s_t, \xi_t) - \wh_t\|^2] + \frac{\hrho}{2}\E_t[\|\s_{t+1}-\hat\s_t\|^2] \\
& \leq F(\xh_t) + \frac{\hrho}{2}\|\w_t- \wh_t\|^2 + \hrho\eta_t \E_t[(\wh_t - \w_t)^{\top}\nabla_{\w} F(\w_t, \s_t)] + \frac{\hrho \eta_1^2G^2}{2} + \frac{\hrho}{2}\E_t[\|\s_{t+1}-\hat\s_t\|^2]\\
\end{align*}
where we assume $\E[\|\nabla_{\w} F(\w_t, \s_t, \xi_t)\|^2]\leq G^2 = \frac{C^2}{\beta^2}$.  According to the analysis of stochastic coordinate descent method, for any $\s=(s_1,\ldots, s_{n_+})$ we have 
\begin{align*}
    &2\eta_2(s_{t,i}-s_i)^{\top}\nabla_{s_i} F(\w_t, \s_t; \xi_t)\leq \eta_2^2\|\nabla_{s_i} F(\bar\w_t, \u_t; \xi_t)\|^2 + (\|s_i- s_{t,i}\|^2 -\|s_i- s_{t+1,i}\|^2)\\
\end{align*}
Summing the above inequality over $i\in\B_+$, we have
\begin{align*}
    &2\eta_2\sum_{i\in\B_+}(s_{t,i}-s_i)^{\top}\nabla_{s_i} F(\w_t, \s_t; \xi_t)\leq \eta_2^2\sum_{i\in\B_+}\|\nabla_{s_i} F(\bar\w_t, \u_t; \xi_t)\|^2 + (\|\s- \s_{t}\|^2 -\|\s- \s_{t+1}\|^2)\\
\end{align*}
Taking expectation and re-arrange, we have
\begin{align*}
\E[\frac{1}{2}\|\s_{t+1}-\widehat\s_t\|^2]\leq \E[\frac{1}{2}\|\s_t - \widehat\s_t\|^2  + \eta_2 \frac{B_+}{n_+}(\widehat\s_t - \s_t)^{\top}\nabla_{\s}F(\w_t, \s_t) + \frac{\eta_2^2 B_+C_2^2}{2} ]
\end{align*}
where we use the fact $\E[|\nabla_{s_i} F(\w, \s)|^2]\leq C_2^2 = \frac{1}{n_+^2}(1+1/\beta)^2$. Let $\eta_2B_+/n_+=\eta_1$, we have
\begin{align*}
&\E_t[F_{1/\hrho}(\v_{t+1})] \\
& \leq F(\xh_t) + \frac{\hrho}{2}\|\w_t- \wh_t\|^2 + \hrho\eta_1\E_t[(\wh_t - \w_t)^{\top}\nabla_{\w} F(\w_t, \s_t)]  + \frac{\hrho}{2}\E_t[\|\s_{t+1}- \widehat\s_t)]+ \frac{\hrho \eta_1^2G^2}{2}\\
&\leq  F(\vh_t) + \frac{\hrho}{2}\|\v_t- \vh_t\|^2 + \hrho\eta_1 \E_t[(\vh_t - \v_t)^{\top}\partial F(\w_t, \s_t)] + \frac{\hrho \eta_1^2(G^2+n_+^2C_2^2/B_+)}{2}\\
&\leq F_{1/\hrho}(\v_t) + \hrho\eta_1 (F(\vh_t) - F(\v_t) + \frac{\rho}{2}\|\v_t - \vh_t\|^2)+ \frac{\hrho \eta_1^2(G^2+n_+^2C_2^2/B_+)}{2}\\
\end{align*}

As  a result, we have
\begin{align*}
 \hrho\eta_1 (F(\v_t) - F(\vh_t) - \rho\|\v_t - \vh_t\|^2)\leq F_{1/\hrho}(\v_t) - \E_t[F_{1/\hrho}(\v_{t+1})]   + \frac{\hrho \eta_1^2 (G^2+n_+^2C_2^2/B_+)}{2} 
\end{align*}
Since we have
\begin{align*}
&F(\v_t) - F(\vh_t) - \rho\|\v_t - \vh_t\|^2  = (F(\v_t) + \hrho\|\v_t - \v_t\|^2) -( F(\vh_t)  +  \hrho\|\vh_t - \v_t\|^2) + (\hrho - \rho)\|\v_t - \vh_t\|^2 \\
&\geq \frac{(2\hrho - \rho)}{2}\|\vh_t - \v_t\|^2  + (\hrho - \rho)\|\v_t - \vh_t\|^2 = (2\hrho - 3/2\rho)\|\v_t - \vh_t\|^2 = \frac{(2\hrho - 3/2\rho)}{\hrho^2}\|\nabla F_{1/\hrho}(\v_t)\|^2
\end{align*}
Let $\hat\rho=3\rho/2$. As a result, we have
\begin{align*}
\frac{1}{T}\sum_{t=1}^T\|\nabla F_{1/\hrho}(\v_t)\|^2
&\leq  \frac{(F_{1/\hrho}(\v_1) - \min F)}{\eta_1 T}  + \frac{\hrho \eta_1 (G^2+n^2_+C_2^2/B_+)}{2} \\
&\leq  \frac{(F_{1/\hrho}(\v_1) - \min F)}{\eta_1 T}  + \frac{\hrho \eta_1 (C^2/\beta^2+(1+1/\beta)^2/B_+)}{2} 
\end{align*}
By setting $\eta_1 = O(\beta\epsilon^2)$ and $T=O(\frac{1}{\epsilon^2\eta_1})=O(1/(\beta\epsilon^4))$ we have $\E[\|\nabla F_{1/\hrho}(\v_\tau)\|^2]\leq \epsilon^2$  for a randomly selected $\tau\in[T]$. 

\subsection{Proof of Theorem~\ref{thm:2}}
Note that the SOPA-s algorithm is just a special case of the SOX algorithm~\citep{wangsox} for the more general problem $\min_\w \frac{1}{n}\sum_{\z_i\in\D} f_i(g_i(\w))$ and the convergence proof just follows the proof of Theorem 1 in \citet{wangsox}.

\subsection{Proof of Theorem~\ref{thm:alg3}}
We consider the following problem: 
\begin{align}\label{eq:prob_general}
	\min_\w F(\w),\quad F(\w) = f_1\left(\frac{1}{n}\sum_{i\in\S} f_2(g_i(\w))\right). 
\end{align}
\begin{lemma}\label{lem:F_smooth_3}
If $g_i$ is $C_g$-Lipschitz, $L_g$-smooth and $f_1, f_2$ are $C_f$-Lipschitz, $L_g$-smooth, $F$ in \eqref{eq:prob_general} is $L_F$-smooth and $L_F =L_f C_f^2 C_g^2 + C_f^2L_g + C_f C_g L_f$.
\end{lemma}	
\begin{proof}
Based on the definition of $F$, we have
\begin{align*}
	  & \Norm{\nabla F(\w) - \nabla F(\w')}\\ &= \left\|\nabla f_1\left(\frac{1}{n}\sum_{i\in\S}f_2(g_i(\w))\right)\left(\frac{1}{n}\sum_{i\in\S}\nabla f_2(g_i(\w))\nabla g_i(\w)\right) \right.\\
	& \quad\quad\quad \left.- \nabla f_1\left(\frac{1}{n}\sum_{i\in\S}f_2(g_i(\w'))\right)\left(\frac{1}{n}\sum_{i\in\S}\nabla f_2(g_i(\w'))\nabla g_i(\w')\right)\right\|\\
	& \leq L_f\Norm{\frac{1}{n}\sum_{i\in\S}\nabla f_2(g_i(\w))\nabla g_i(\w)}\Norm{\frac{1}{n}\sum_{i\in\S}(f_2(g_i(\w)) -f_2(g_i(\w')))}\\
	& \quad\quad\quad +\Norm{\nabla f_1\left(\frac{1}{n}\sum_{i\in\S}f_2(g_i(\w'))\right)}\Norm{\frac{1}{n}\sum_{i\in\S}(\nabla f_2(g_i(\w))\nabla g_i(\w) - \nabla f_2(g_i(\w'))\nabla g_i(\w') )}.
\end{align*}
We can show that $\Norm{\frac{1}{n}\sum_{i\in\S}(\nabla f_2(g_i(\w))\nabla g_i(\w) - \nabla f_2(g_i(\w'))\nabla g_i(\w') )}\leq (C_f L_g + C_g L_f)\Norm{\w-\w'}$. Thus,
\begin{align*}
	\Norm{\nabla F(\w) - \nabla F(\w')} & \leq L_f C_f^2 C_g^2\Norm{\w-\w'} + C_f(C_f L_g + C_g L_f)\Norm{\w-\w'} \\
	& = (L_f C_f^2 C_g^2 + C_f^2L_g + C_f C_g L_f)\Norm{\w-\w'}.
\end{align*}
\end{proof}
We propose SOTA-s to solve \eqref{eq:prob_general}.
\begin{lemma}\label{lem:nonconvex_starter}
Consider a sequence $\w_{t+1} = \w_t - \eta \m_t$ and the $L_F$-smooth function $F$ and the step size $\eta L_F\leq 1/2$. 
\begin{align}\label{eq:nonconvex_starter}
	F(\w_{t+1}) & \leq F(\w_t) + \frac{\eta}{2}\Norm{\Delta_t}^2 - \frac{\eta}{2}\Norm{\nabla F(\w_t)}^2 - \frac{\eta}{4}\Norm{\m_t}^2,
\end{align}	
where $\Delta_t \coloneqq \m_t - \nabla F(\w_t)$
\end{lemma}	

\def \bg {\mathbf{g}}

\begin{lemma}\label{lem:3_level_grad_recursion}
	For the gradient estimator $\m_t$ in SOTA-s and $\Delta_t = \m_t - \nabla F(\w_t)$, 
	\begin{align}\nonumber
		\E\left[\Norm{\Delta_{t+1}}^2\right] &\leq (1-\gamma_2)\E\left[\Norm{\Delta_t}^2\right] + \frac{2L_F^2\eta^2}{\gamma_2} \E\left[\Norm{\m_t}^2\right] + 10\gamma_2 C_f^2C_1^2L_f^2 \E\left[\Norm{\Psi_{t+1}}^2\right]\\\label{eq:3_level_grad_recursion}
		& \quad + 20\gamma_2 C_f^2L_f^2C_1^2 \E\left[\frac{1}{n}\Norm{\Xi_{t+1}}^2\right] + 20\gamma_2 C_f^2L_f^2C_1^2 \E\left[\frac{1}{n}\Norm{\u_{t+1} - \u_t}^2\right] + 2\gamma_2^2C_f^4\left(\frac{\zeta^2}{B_-} + \frac{C_g^2}{B_+}\right),
	\end{align}
	where we denote $\Delta_t \coloneqq \m_{t+1} - \nabla F(\w_t)$, $\Xi_t \coloneqq \u_{t+1} - \bg(\w_t)$ and $\Psi_t\coloneqq v_{t+1} - \frac{1}{n}\sum_{i\in\S} f_2(g_i(\w_t))$.
\end{lemma}	
\begin{proof}
	Based on the update rule $\m_{t+1} = (1-\gamma_2)\m_t + \gamma_2 G(\w_{t+1})$, we have
	\begin{align*}
		& \Norm{\m_{t+1}  - \nabla F(\w_{t+1})}^2\\
		& = \Norm{(1-\gamma_2)\m_t + \gamma_2 \frac{1}{B_+}\sum_{i\in\B_+} \nabla f_1(v_{t+1})\nabla f_2(u_{t}^i)\nabla g_i (\w_{t+1};\B_-) - \nabla F(\w_{t+1})}^2\\
		&= \left\|\underbrace{(1-\gamma_2) (\m_t - \nabla F(\w_t))}_{\tcircle{a}} + \underbrace{(1-\gamma_2)(\nabla F(\w_t) - \nabla F(\w_{t+1}))}_{\tcircle{b}} + \gamma_2 \frac{1}{B_+}\sum_{i\in\B_+}\nabla f_1(v_{t+1})\nabla f_2(u_{t}^i)\nabla g_i (\w_{t+1};\B_-) \right. \\
		& \quad\quad \left. -\gamma_2 \frac{1}{B_+}\sum_{i\in\B_+}\nabla f_1\left(\frac{1}{n}\sum_{i\in\S} f_2(g_i(\w_{t+1}))\right)\nabla f_2(g_i(\w_{t+1}))\nabla g_i (\w_{t+1};\B_-)\right.\\
		& \quad\quad \left. +\underbrace{\gamma_2 \left(\frac{1}{B_+}\sum_{i\in\B_+}\nabla f_1\left(\frac{1}{n}\sum_{i\in\S} f_2(g_i(\w_{t+1}))\right) \nabla f_2(g_i(\w_{t+1}))\nabla g_i (\w_{t+1};\B_-) - \nabla F(\w_{t+1}) \right)}_{\tcircle{d}}\right\|^2.
	\end{align*}
We define that
\begin{align*}
	\tcircle{c} &\coloneqq \gamma_2 \left(\frac{1}{B_+}\sum_{i\in\B_+}\nabla f_1(v_{t+1})\nabla f_2(u_{t}^i)\nabla g_i (\w_{t+1};\B_-)\right.\\
	&\left.\quad\quad\quad - \frac{1}{B_+}\sum_{i\in\B_+}\nabla f_1\left(\frac{1}{n}\sum_{i\in\S} f_2(g_i(\w_{t+1}))\right)\nabla f_2(g_i(\w_{t+1}))\nabla g_i (\w_{t+1};\B_-)\right).
\end{align*}
We define that $\Delta_t \coloneqq \m_t  - \nabla F(\w_t)$. Note that $\E\left[\inner{\tcircle{a}}{\tcircle{d}}\right] = \E\left[\inner{\tcircle{b}}{\tcircle{d}}\right] = 0$. Then,
\begin{align*}
\E_t\left[\Norm{\tcircle{a} + \tcircle{b} + \tcircle{c} + \tcircle{d}}^2\right] &= \Norm{\tcircle{a}}^2 + \Norm{\tcircle{b}}^2 + \E_t\left[\Norm{\tcircle{c}}^2\right] + \E_t\left[\Norm{\tcircle{d}}^2\right] + 2\inner{\tcircle{a}}{\tcircle{b}}\\
& \quad\quad\quad + 2\E_t\left[\inner{\tcircle{a}}{\tcircle{c}}\right] + 2\E_t\left[\inner{\tcircle{b}}{\tcircle{c}}\right] + 2\E_t\left[\inner{\tcircle{c}}{\tcircle{d}}\right].
\end{align*}
Based on the Young's inequality for products, we have $2\inner{\a}{\b}\leq \frac{\Norm{\a}^2c}{2} + \frac{2\Norm{\b}^2}{c}$ for $c>0$.
\begin{align*}
    & \E_t\left[\Norm{\tcircle{a} + \tcircle{b} + \tcircle{c} + \tcircle{d}}^2\right] \\
    & \leq (1+\gamma_2)\Norm{\tcircle{a}}^2 + 2(1+1/\gamma_2)\Norm{\tcircle{b}}^2 + \frac{2+3\gamma_2}{\gamma_2}\E_t\left[\Norm{\tcircle{c}}^2\right] + 2\E_t\left[\Norm{\tcircle{d}}^2\right].
\end{align*}
Thus, we have
\begin{align}\label{eq:3_level_grad_single_iter}
	\E_t[\Norm{\Delta_{t+1}}^2] &\leq (1-\gamma_2)\Norm{\Delta_t}^2 + \frac{2(1+\gamma_2)}{\gamma_2}\Norm{\tcircle{b}}^2 + \frac{5}{\gamma_2}\E_t[\Norm{\tcircle{c}}^2] + 2\E_t[\Norm{\tcircle{d}}^2].
\end{align}
Moreover, we have
\begin{align}\label{eq:3_level_bound_B2}
	\Norm{\tcircle{b}}^2 &= (1-\gamma_2)^2\Norm{\nabla F(\w_t) - \nabla F(\w_{t+1})}^2\leq (1-\gamma_2)^2\eta^2 L_F^2\Norm{\m_t}^2.
\end{align}
On the other hand,
\begin{align*}
& \Norm{\tcircle{c}}^2 \leq \frac{2\gamma_2^2}{B_+} \sum_{i\in\B_+} \Norm{\nabla g_i(\w_{t+1};\B_-)}^2\Norm{\nabla f_1(v_{t+1}) - \nabla f_1\left(\frac{1}{n}\sum_{i\in\S}f_2(g_i(\w_{t+1}))\right)}^2\Norm{\nabla f_2(u_{t}^i)}^2 \\
& + \frac{2\gamma_2^2}{B_+} \sum_{i\in\B_+}\Norm{\nabla g_i(\w_{t+1};\B_-)}^2\Norm{\nabla f_1\left(\frac{1}{n}\sum_{i\in\S}f_2(g_i(\w_{t+1}))\right)}^2\Norm{\nabla f_2(u_{t}^i) - \nabla f_2(g_i(\w_{t+1}))}^2\\
& \leq \frac{2\gamma_2^2 L_f^2C_f^2}{B_+} \sum_{i\in\B_+} \Norm{\nabla g_i(\w_{t+1};\B_-)}^2\Norm{v_{t+1} - \frac{1}{n}\sum_{i\in\S} f_2(g_i(\w_{t+1}))}^2\\
& \quad\quad\quad + \frac{2\gamma_2^2 L_f^2C_f^2}{B_+}\sum_{i\in\B_+} \Norm{\nabla g_i(\w_{t+1};\B_-)}^2 \Norm{u_{t}^i - g_i(\w_{t+1})}^2
\end{align*}
Due to $\E_{\B_-}\left[\Norm{\nabla g_i(\w_{t+1};\B_-)}^2\right]\leq C_g^2 + \zeta^2/\B_-\coloneqq C_1^2$, 
we have
\begin{align*}
\E_t\left[\Norm{\tcircle{c}}^2\right]&\leq \frac{2\gamma_2^2L_f^2 C_f^2C_1^2}{n}\sum_{i\in\S} \E_t\left[\Norm{v_{t+1} - \frac{1}{n}\sum_{i\in\S} f_2(g_i(\w_{t+1}))}^2 + \Norm{u_{t}^i - g_i(\w_{t+1})}^2\right]    
\end{align*}
Besides, we also have
\begin{align}\nonumber
	& \E_t\left[\Norm{\tcircle{d}}^2\right] \\\nonumber
	&\leq \gamma_2^2\E_t\left[\Norm{\frac{1}{B_+}\sum_{i\in\B_+}\nabla f_1\left(\frac{1}{n}\sum_{i\in\S} f_2(g_i(\w_{t+1}))\right)\nabla f_2(g_i(\w_{t+1}))\nabla g_i (\w_{t+1};\B_-) - \nabla F(\w_{t+1})}^2\right]\\\nonumber
	& = \gamma_2^2C_f^2 \E_t\left[\Norm{\frac{1}{B_+}\sum_{i\in\B_+}\left(\nabla f_2(g_i(\w_{t+1}))\nabla g_i (\w_{t+1};\B_-) - \nabla f_2(g_i(\w_{t+1}))\nabla g_i (\w_{t+1})\right)}^2\right]\\\nonumber
	& +  \gamma_2^2C_f^2 \E_t\left[\Norm{\frac{1}{B_+}\sum_{i\in\B_+}\nabla f_2(g_i(\w_{t+1}))\nabla g_i (\w_{t+1};\B_-) - \frac{1}{n}\sum_{i\in\S}\nabla f_2(g_i(\w_{t+1}))\nabla g_i (\w_{t+1};\B_-)}^2\right]\\\label{eq:3_level_bound_A4}
	& \leq \frac{\gamma_2^2 C_f^4\zeta^2}{B_-} + \frac{\gamma_2^2C_f^4C_g^2}{B_+}.  
\end{align}
Then,
\begin{align*}
	\E\left[\Norm{\Delta_{t+1}}^2\right] &\leq (1-\gamma_2)\E\left[\Norm{\Delta_t}^2\right] + \frac{2L_F^2\eta^2}{\gamma_2} \E\left[\Norm{\m_t}^2\right] + 10\gamma_2 C_f^2C_1^2L_f^2 \E\left[\Norm{\Psi_{t+1}}^2\right]\\
	& \quad\quad  + 20\gamma_2 C_f^2L_f^2C_1^2 \E\left[\frac{1}{n}\Norm{\Xi_{t+1}}^2\right]  + 20\gamma_2 C_f^2L_f^2C_1^2 \E\left[\frac{1}{n}\Norm{\u_{t+1} - \u_t}^2\right]+ 2\gamma_2^2C_f^4\left(\frac{\zeta^2}{B_-} + \frac{C_g^2}{B_+}\right),
\end{align*}
where we denote $\Xi_t \coloneqq \u_t - \bg(\w_t)$ and $\Psi_t\coloneqq v_t - \frac{1}{n}\sum_{i\in\S} f_2(g_i(\w_t))$.
\end{proof}	

\begin{lemma}\label{lem:2_level_fval_recursion}
For the function value estimator $v_{t+1}$ in SOTA-s and $\Psi_t\coloneqq v_t - \frac{1}{n}\sum_{i\in\S} f_2(g_i(\w_t))$, 
\begin{align}\label{eq:2_level_fval_recursion}
	\E_t\left[\Norm{\Psi_{t+1}}^2\right] &\leq (1-\gamma_1)\Norm{\Psi_t}^2 + \frac{2C_f^2C_g^2\eta^2}{\gamma_1}\Norm{\m_t}^2 +5\gamma_1 C_f^2\E_t\left[\frac{1}{n}\Norm{\Xi_{t+1}}^2\right] \\\nonumber
	&+ \frac{10\gamma_1 C_f^2}{n}\E_t\left[\Norm{\u_{t} - \u_{t+1}}^2\right]+ \frac{2\gamma_1^2C_f^2}{B_+} .
\end{align}
\end{lemma}
\begin{proof}
	According to the update of $v$ in SOTA-s, we have
	\begin{align*}
		& \Norm{v_{t+1} - \frac{1}{n}\sum_{i\in\S} f_2(g_i(\w_{t+1}))}^2 = \Norm{ (1-\gamma_1)v_t + \gamma_1\frac{1}{B_+}\sum_{i\in\B_+} f_2(u_{t}^i) -  \frac{1}{n}\sum_{i\in\S} f_2(g_i(\w_{t+1}))}^2\\
		& = \left\|(1-\gamma_1)\left(v_t - \frac{1}{n}\sum_{i\in\S} f_2(g_i(\w_t))\right) + (1-\gamma_1) \left(\frac{1}{n}\sum_{i\in\S}\left(f_2(g_i(\w_t)) - f_2(g_i(\w_{t+1}))\right)\right)\right.\\
		& \quad\quad\quad  \left. + \gamma_1 \left(\frac{1}{B_+}\sum_{i\in\B_+} f_2(u_{t}^i) - \frac{1}{B_+}\sum_{i\in\B_+} f_2(g_i(\w_{t+1}))\right)  + \gamma_1 \left(\frac{1}{B_+}\sum_{i\in\B_+} f_2(g_i(\w_{t+1})) - \frac{1}{n}\sum_{i\in\S} f_2(g_i(\w_{t+1}))\right)\right\|^2.
	\end{align*}
Denoting $\Psi_t \coloneqq \Norm{v_t - \frac{1}{n}\sum_{i\in\S} f_2(g_i(\w_t))}^2 $, we have
\begin{align*}
	& \E_t\left[\Norm{\Psi_{t+1}}^2\right] \\\nonumber
	& \leq (1-\gamma_1)\Norm{\Psi_t}^2 + \frac{2(1-\gamma_1)C_f^2 C_g^2\eta^2}{\gamma_1}\Norm{\m_t}^2 + \frac{(2+3\beta)C_f^2}{\gamma_1}\gamma_1^2\E_t \left[\frac{1}{B_+}\sum_{i\in\B_+}\Norm{u_{t}^i - g_i(\w_{t+1})}^2\right] + \frac{2\gamma_1^2C_f^2}{B_+} \\
	& \leq (1-\gamma_1)\Norm{\Psi_t}^2 + \frac{2C_f^2C_g^2\eta^2}{\gamma_1}\Norm{\m_t}^2 + \frac{5\gamma_1 C_f^2}{n}\E_t\left[\Norm{\u_{t} - \bg(\w_{t+1})}^2\right]+ \frac{2\gamma_1^2C_f^2}{B_+}\\
		& \leq (1-\gamma_1)\Norm{\Psi_t}^2 + \frac{2C_f^2C_g^2\eta^2}{\gamma_1}\Norm{\m_t}^2 + \frac{10\gamma_1 C_f^2}{n}\E_t\left[\Norm{\u_{t+1} - \bg(\w_{t+1})}^2\right]+\frac{10\gamma_1 C_f^2}{n}\E_t\left[\Norm{\u_{t} - \u_{t+1}}^2\right]+ \frac{2\gamma_1^2C_f^2}{B_+} .
\end{align*}
\end{proof}

\begin{proof}[Proof of Theorem~\ref{thm:alg3}]
Based on \eqref{eq:nonconvex_starter}, we have 
\begin{align}\label{eq:sopu3_first_eq}
	\E\left[F(\w_{t+1}) - F_{\inf}\right] &\leq \E\left[F(\w_t)- F_{\inf}\right] + \frac{\eta}{2}\E\left[\Norm{\Delta_t}^2\right] - \frac{\eta}{2}\E\left[\Norm{\nabla F(\w_t)}^2\right] - \frac{\eta}{4}\E\left[\Norm{\m_t}^2\right]
\end{align}
Re-arranging the terms and telescoping \eqref{eq:sopu3_first_eq} from $t=1$ to $T$ leads to
\begin{align}\label{eq:telescoped}
\frac{1}{T}\sum_{t=1}^T \E\left[\Norm{\nabla F(\w_t)}^2\right] & \leq \frac{2\E\left[F(\w_1) - F_{\inf}\right]}{\eta T} + \underbrace{\frac{1}{T}\sum_{t=1}^T \E\left[\Norm{\Delta_t}^2\right]}_{\coloneqq \tcircle{e}} - \frac{1}{2T}\sum_{t=1}^T \E\left[\Norm{\m_t}^2\right].    
\end{align}
Based on \eqref{eq:3_level_grad_recursion}, the term $\tcircle{e}$ can be upper bounded as
\begin{align*}
\frac{1}{T}\sum_{t=1}^T\E\left[\Norm{\Delta_t}^2\right] &\leq \frac{\E\left[\Norm{\Delta_1}^2\right]}{\gamma_2 T} + \frac{2L_F^2\eta^2}{\gamma_2^2}\frac{1}{T}\sum_{t=1}^T\E\left[\Norm{\m_t}^2\right] + 2\gamma_2 C_f^4 \left(\frac{\zeta^2}{B_-} + \frac{C_g^2}{B_+}\right) \\
&+ 20 C_f^2 C_1^2 L_f^2\sum_{t=1}^T\frac{1}{n}\E[\|\u_{t+1} - \u_t\|^2] + 20 C_f^2 C_1^2 L_f^2 \left(\underbrace{\frac{1}{T}\sum_{t=1}^T \E\left[\Norm{\Psi_{t+1}}^2\right]}_{\coloneqq \tcircle{f}} + \underbrace{\frac{1}{T}\sum_{t=1}^T\E\left[\frac{1}{n}\Norm{\Xi_{t+1}}^2\right]}_{\coloneqq \tcircle{g}}\right).
\end{align*}
Based on Lemma 2 in \citet{wangsox}, the term $\tcircle{g}$ can be upper bounded as
\begin{align*}
\frac{1}{T}\sum_{t=1}^T \E\left[\frac{1}{n}\Norm{\Xi_{t+1}}^2\right] & \leq \frac{4n \E\left[\frac{1}{n}\Norm{\Xi_2}^2\right]}{\gamma_0 B_+ T} + 20 C_g^2\left(\frac{n\eta}{\gamma_0 B_+}\right)^2\frac{1}{T}\sum_{t=1}^T\E\left[\Norm{\m_{t+1}}^2\right]\\\nonumber
&+ \frac{8\gamma_0 \sigma^2}{B_-} -\frac{1}{\gamma_0 B_+}\sum_{t=1}^T\E\left[\Norm{\u^{t+1} - \u^t}^2\right].
\end{align*}
With $1/(\gamma_0 B_+)\geq \max(10C_f^2/n, 1/n)$, based on \eqref{eq:2_level_fval_recursion}, the term $\tcircle{f}$ can be bounded as
\begin{align*}
&\frac{1}{T}\sum_{t=1}^T\E\left[\Norm{\Psi_{t+1}}^2\right] \leq \frac{\E\left[\Norm{\Psi_2}^2\right]}{\gamma_1 T} + \frac{2C_f^2 C_g^2\eta^2}{\gamma_1^2}\frac{1}{T}\sum_{t=1}^T\E\left[\Norm{\m_{t+1}}^2\right] + 5 C_f^2\frac{1}{T}\sum_{t=1}^T \E\left[\frac{1}{n}\Norm{\Xi_{t+2}}^2\right] + \frac{2\gamma_1 C_f^2}{B_+} \\
& \leq \frac{\E\left[\Norm{\Psi_2}^2\right]}{\gamma_1 T} + \frac{2C_f^2 C_g^2\eta^2}{\gamma_1^2}\frac{1}{T}\sum_{t=1}^T\E\left[\Norm{\m_{t+1}}^2\right] + \frac{2\gamma_1 C_f^2}{B_+} + \frac{20C_f^2 n \E\left[\frac{1}{n}\Norm{\Xi_3}^2\right]}{\gamma_0 B_+ T} \\
& \quad\quad\quad\quad + 100C_f^2 C_g^2\left(\frac{n\eta}{\gamma_0 B_+}\right)^2 \frac{1}{T}\sum_{t=1}^T\E\left[\Norm{\m_{t+2}}^2\right] + \frac{40C_f^2 \gamma_0 \sigma^2}{B_-}.
\end{align*}
Plug the upper bounds of $\tcircle{f}$ and $\tcircle{g}$ into \eqref{eq:telescoped}.
\begin{align*}
& \frac{1}{T}\sum_{t=1}^T\E\left[\Norm{\nabla F(\w_t)}^2\right] \\
&\leq \frac{2\E\left[F(\w_1) - F_{\inf}\right]}{\eta T} + \frac{\E\left[\Norm{\Delta_1}^2\right]}{\gamma_2 T} + 2\gamma_2 C_f^4\left(\frac{\zeta^2}{B_-}+ \frac{C_g^2}{B_+}\right) \\
&+ \frac{40 n C_f^2 C_1^2 L_f^2\E\left[\frac{1}{n}\Norm{\Xi_2}^2\right]}{\gamma_0 B_+ T} - \left(\frac{1}{2}-\frac{2L_F^2\eta^2}{\gamma_2^2}\right)\frac{1}{T}\sum_{t=1}^T\E\left[\Norm{\m_t}^2\right] \\
& + 200C_f^2 C_g^2 C_1^2 L_f^2 \left(\frac{n\eta}{\gamma_0 B_+}\right)^2 \frac{1}{T}\sum_{t=1}^T \E\left[\Norm{\m_{t+1}}^2\right] + \frac{80\gamma_0 C_f^2 C_1^2 L_f^2\sigma^2}{B_-} \\
& + \frac{10C_f^2 C_1^2 L_f^2 \E\left[\Norm{\Psi_2}^2\right]}{\gamma_1 T} + \frac{20 C_f^4 C_g^2 C_1^2 L_f^2\eta^2}{\gamma_1^2} \frac{1}{T}\sum_{t=1}^T\E\left[\Norm{\m_{t+1}}^2\right] + \frac{20\gamma_1 C_f^4 C_1^2 L_f^2}{B_+} \\
& + \frac{200 n C_f^4 C_1^2 L_f^2\E\left[\frac{1}{n}\Norm{\Xi_3}^2\right]}{\gamma_0 B_+ T} + 1000C_f^4 C_1^2 C_g^2\left(\frac{n\eta}{\gamma_0 B_+}\right)^2 \frac{1}{T}\sum_{t=1}^T \E\left[\Norm{\m_{t+2}}^2\right] + \frac{400 \gamma_0 C_f^4C_1^2 L_f^2\sigma^2}{B_-}.
\end{align*}
If we choose $\eta\leq \min\left\{\frac{\gamma_2}{4L_F},\frac{\gamma_0 B_+}{40 n C_f C_g C_1\sqrt{L_f^2 + 5C_f^2}},\frac{\gamma_1}{15 C_f^2C_g C_1 L_f}\right\}$, we have 
\begin{align*}
& - \left(\frac{1}{2}-\frac{2L_F^2\eta^2}{\gamma_2^2}\right)\frac{1}{T}\sum_{t=1}^T\E\left[\Norm{\m_t}^2\right] + 200C_f^2 C_g^2 C_1^2 L_f^2 \left(\frac{n\eta}{\gamma_0 B_+}\right)^2 \frac{1}{T}\sum_{t=1}^T \E\left[\Norm{\m_{t+1}}^2\right] \\
& \quad\quad + \frac{20 C_f^4 C_g^2 C_1^2 L_f^2\eta^2}{\gamma_1^2} \frac{1}{T}\sum_{t=1}^T\E\left[\Norm{\m_{t+1}}^2\right] + 1000C_f^4 C_1^2 C_g^2\left(\frac{n\eta}{\gamma_0 B_+}\right)^2 \frac{1}{T}\sum_{t=1}^T \E\left[\Norm{\m_{t+2}}^2\right] \\
& \leq \frac{\E\left[\Norm{\m_{T+1}}^2\right] + \E\left[\Norm{\m_{T+2}}^2\right]}{8T}.
\end{align*}
Besides, Lemma~2 in \citet{wangsox} and Lemma~\ref{lem:2_level_fval_recursion} imply that
\begin{align*}
\E\left[\frac{1}{n}\Norm{\Xi_2}^2\right] & \leq \E\left[\frac{1}{n}\Norm{\Xi_1}^2\right] + \frac{5n\eta^2 C_g^2}{\gamma_0 B_+}\E\left[\Norm{\m_1}^2\right] + \frac{2\gamma_0^2\sigma^2 B_+}{nB_-}\\
& \leq \E\left[\frac{1}{n}\Norm{\Xi_1}^2\right] + \frac{\eta C_g}{8C_f C_1 \sqrt{L_f^2 + 5C_f^2}} \E\left[\Norm{\m_1}^2\right] + \frac{2\gamma_0^2\sigma^2 B_+}{nB_-},\\
\E\left[\frac{1}{n}\Norm{\Xi_3}^2\right] 
&\leq \E\left[\frac{1}{n}\Norm{\Xi_1}^2\right] + \frac{\eta C_g}{8C_f C_1 \sqrt{L_f^2 + 5C_f^2}}\left(\E\left[\Norm{\m_1}^2 + \Norm{\m_2}^2\right]\right) +  \frac{4\gamma_0^2\sigma^2 B_+}{nB_-},\\
\E\left[\Norm{\Psi_2}^2\right] &\leq \E\left[\Norm{\Psi_1}^2\right] + \frac{2C_g\eta}{15C_1 L_f}\Norm{\m_1}^2 + 5\gamma_1 C_f^2 \E\left[\frac{1}{n}\Norm{\Xi_{t+1}}^2\right] + \frac{2\gamma_1^2C_f^2}{B_+}.
\end{align*}
If we initialize $\m_1 = 0$, then $\E\left[\Norm{\m_t}^2\right]\leq C_f^2 C_1^2$ for any $t\geq 1$. We define that $\Lambda_{F,1} = \E\left[F(\w_1) - F_{\inf}\right]<+\infty$, $\Lambda_{\Delta,1} = \E\left[\Norm{\Delta_1}^2\right]<+\infty$, $\Lambda_{\Xi,2} = \E\left[\frac{1}{n}\Norm{\Xi_2}^2\right]<+\infty$, $\Lambda_{\Xi,3} = \E\left[\frac{1}{n}\Norm{\Xi_2}^2\right]<+\infty$, $\Lambda_{\Psi}^2 = \E\left[\Norm{\Psi_2}^2\right]<+\infty$. Then,
\begin{align*}
& \frac{1}{T}\sum_{t=1}^T\E\left[\Norm{\nabla F(\w_t)}^2\right] \\   
& \leq \frac{2\Lambda_{F,1}}{\eta T} + \frac{\Lambda_{\Delta,1}}{\gamma_2 T} + \frac{40 n C_f^2 C_1^2 L_f^2 \Lambda_{\Xi,2}}{\gamma_0 B_+ T} + \frac{10C_f^2 C_1^2 L_f^2\Lambda_\Psi^2}{\gamma_1 T} + \frac{200nC_f^4 C_1^2 L_f^2 \Lambda_{\Xi,3}}{\gamma_0 B_+T}\\
& \quad\quad\quad +2\beta C_f^4\left(\frac{\zeta^2}{B_-} + \frac{C_g^2}{B_+}\right) + \frac{80\gamma_0 C_f^2C_1^2L_f^2\sigma^2}{B_-} + \frac{20\gamma_1 C_f^4 C_1^2L_f^2}{B_+} + \frac{400\gamma_0 C_f^4 C_1^2 L_f^2\sigma^2}{B_-} +\frac{C_f^2C_1^2}{4T}. 
\end{align*}
Set $\gamma_0 = \frac{B_-\epsilon^2}{400C_f^2C_1^2L_f^2\sigma^2(1+5C_f^2)}$, $\gamma_1 =\frac{B_+\epsilon^2}{200C_f^4C_1^2L_f^2}$, $\gamma_2 = \frac{\min\left\{B_-,B_+\right\}\epsilon^2}{20C_f^4 (\zeta^2 + C_g^2)}$, and
\begin{align*}
& \eta\leq \min\left\{\frac{\gamma_2}{4L_F},\frac{\gamma_0 B_+}{40 n C_f C_g C_1\sqrt{L_f^2 + 5C_f^2}},\frac{\gamma_1}{15 C_f^2C_g C_1 L_f}\right\},\\
& T = \max\left\{\frac{1600\Lambda_{F,1} L_FC_f^4(\zeta^2+C_g^2)}{\min\{B_-,B_+\}\epsilon^4}, \frac{320000 n\Lambda_{F,1}C_f^3C_gC_1^3L_f^2(1+5C_f^2)\sqrt{L_f^2 + 5C_f^2}\sigma^2}{B_-B_+\epsilon^4},\right.\\
& \quad\quad\quad\quad\quad \left. \frac{60000\Lambda_{F,1} C_f^6 C_g C_1^3L_f^3}{B_+\epsilon^4}, \frac{200C_f^4(\zeta^2 + C_g^2)\Lambda_{\Delta,1}}{\min\{B_-,B_+\}\epsilon^4}, \frac{160000 n C_f^4C_1^4L_f^4\sigma^2(1+5C_f^2)\Lambda_\Xi^2}{B_-B_+\epsilon^4},\right.\\
&\quad\quad\quad\quad\quad \left.\frac{20000C_f^6 C_1^4L_f^4\Lambda_\Psi^2}{B_+\epsilon^4}, \frac{800000n C_f^6 C_1^4 L_f^4(1+5C_f^2) \Lambda_{\Xi,3} \sigma^2}{B_-B_+\epsilon^4} \right\}.
\end{align*}
Then, we have $\frac{1}{T}\sum_{t=1}^T\E\left[\Norm{\nabla F(\w_t)}^2\right]\leq \epsilon^2$.

\end{proof}

\section{Optimization of CVaR-estimator of TPAUC}
We have the following estimator
\begin{align*}
    &F(\w)= \min_{\pi\in\R}\pi +\frac{1}{n_+\alpha}\sum_{\x_i\in\S_+}(\min_{s_i}s_i+\frac{1}{\beta}\psi_i(\w; s_i) - \pi)_+,\quad \psi_i(\w; s_i) = \frac{1}{n_-}\sum_{\x_j\in\S_-}(L(\w;\x_i,\x_j) - s_i)_+.
\end{align*}
It is not difficult to show that the above estimator is equivalent to 
\begin{align*}
    &F(\w)= \min_{\pi\in\R, \s\in\R^{n_+}}\pi +\frac{1}{n_+\alpha}\sum_{\x_i\in\S_+}(s_i+\frac{1}{\beta}\psi_i(\w; s_i) - \pi)_+.
\end{align*}
The reason is that $(\min_{x}f(x)-s)_+=\min_{x}(f(x)-s)_+$.  Using the conjugate of $[\cdot]_+$,  we have
\begin{align*}
    \min_{\w, \s\in\R^{n_+}, \pi\in\R}\max_{\u\in[0,1]^{n_+}} 
    \underbrace{\pi+\frac{1}{n_+\alpha}\sum_{\x_i\in\S_+}u_i(s_i+\frac{1}{\beta}\psi_i(\w; s_i) - \pi)}\limits_{F(\w, \s, \pi, \u)}.
\end{align*}
Define $F_i(\w, s_i, \pi,u_i) = \pi + \frac{1}{\alpha} u_i(s_i + \frac{1}{\beta}\psi_i(\w;s_i)-\pi)$ such that $F(\w, \s, \pi, \u) = \frac{1}{n_+}\sum_{\x_i\in \S_+} F_i(\w, s_i, \pi,u_i)$. Based on a minibatch $\B_- \subseteq \S_-$, we can estimate $F_i(\w, s_i, \pi,u_i)$ by $F_i(\w, s_i, \pi,u_i;\B_-) \coloneqq \pi + \frac{1}{\alpha} u_i(s_i + \frac{1}{\beta}\psi_i(\w;s_i;\B_-)-\pi)$ and $\psi_i(\w; s_i; \B_-) \coloneqq \frac{1}{B_-}\sum_{\x_j\in\B_-}(L(\w;\x_i,\x_j) - s_i)_+$. We consider the function $F(\w, \s, \pi, \u)$, which can be proved to be weakly convex w.r.t. $(\w,\s,\pi)$ and concave w.r.t. $\u$. Hence, we can use the stagewise proximal point method to solve the problem~\cite{rafique2018non}. At the $k$-th stage, we solve the following problem approximately: 
\begin{align*}
    \min_{\bar\w}\max_{\u\in[0,1]^{n_+}}F_k(\w, \s, \pi, \u) = F(\w, \s, \pi, \u) + \frac{1}{2\gamma}\|(\w,\s,\pi)-(\w^{k,0},\s^{k,0},\pi^{k,0})\|^2,
\end{align*}
where $\v^{k,0}$ is the initial value of $\v$ in the $k$-th stage. We will use stochastic primal-dual algorithm for solving $F_k$. However, for $\s$ we use stochastic coordinate gradient descent update, and for $\u$ we also use stochastic coordinate gradient ascent update. Let $\tilde{\bg}_\w^{k,t}$, $\tilde{\bg}_\s^{k,t}$, $\tilde{\bg}_{\pi}^{k,t}$, $\tilde{\bg}_\u^{k,t}$ denote stochastic estimators of  partial gradient of $F$ w.r.t. $\w, \s, \pi, \u$, respectively. We consider the following update:
\begin{equation}\label{eqn:update}
\begin{aligned}
    &\w^{k,t+1}= \arg\min_{\w}\left\{\w^{\top}\tilde{\bg}_\w^{k,t} + \frac{1}{2\eta_1}\|\w - \w^{k,t}\|^2 + \frac{1}{2\gamma}\|\w -\w^{k,0}\|^2\right\}\\
    &\s^{k,t+1}=\arg\min_{\s} \left\{\s^\top \tilde{\bg}_\s^{k,t} + \frac{1}{2\eta_2}\Norm{\s- \s^{k,t}}_2^2 + \frac{1}{2\gamma}\Norm{\s -\s^{k,0}}_2^2\right\},\\
    &\pi_{k,t+1}=\arg\min_{\pi} \left\{\pi \tilde{\bg}_\pi^{k,t} + \frac{1}{2\eta_3}|\pi - \pi^{k,t}|^2 + \frac{1}{2\gamma}|\pi - \pi^{k,0}|^2\right\}\\
    &\u^{k,t+1}=\arg\max_{\u\in[0,1]^{n_+}} \left\{\u^\top \tilde{\bg}_\u^{k,t} - \frac{1}{2\eta_4}\Norm{\u-\u^{k,t}}_2^2\right\},
\end{aligned}
\end{equation}
Note that different from~\cite{rafique2018non}, we use stochastic coordinate descent (ascent) to update $\s_{t+1}$ ($\u_{t+1}$). 
Next, we present the stochastic estimators of partial gradients. We define $U_i\in\R^{n_+\times 1}$ whose $i$-th coordinate is 1 while others are zero. Note that $\sum_{\x_i\in \S_+} U_i = I_{n_+}$, where $I_{n_+}$ is a size-$(n_+\times n_+)$ identity matrix.
\begin{align*}
    &\tilde{\bg}_\w^{k,t} = \frac{1}{B_+B_-\alpha\beta}\sum_{\x_i\in\B_+^{k,t}}\sum_{\x_j\in\B_-^{k,t}}u_i^{k,t}\I(L(\w^{k,t}; \x_i,\x_j) - s_i^{k,t}>0)\nabla L(\w^{k,t}; \x_i, \x_j),\\
    & \tilde{\bg}_\s^{k,t} = \frac{1}{\alpha B_+}\sum_{\x_i\in\B_+^{k,t}} U_i u_i^{k,t}(1 - \frac{1}{B_-\beta}\sum_{\x_j\in\B_-^{k,t}}\I(L(\w^{k,t}; \x_i,\x_j)-s_i^{k,t}>0)),\\
    &\tilde{\bg}_\pi^{k,t} = 1 - \frac{1}{B_+\alpha}\sum_{\x_i\in\B_+^{k,t}}u_i^{k,t},\\
     &\tilde{\bg}_\u^{k,t} = \frac{1}{\alpha B_+}\sum_{\x_i \in \B_+^{k,t}} U_i (s_i^{k,t} - \pi^{k,t} +\frac{1}{B_-\beta}\sum_{\x_j\in\B_-^{k,t}}(L(\w^{k,t}; \x_i,\x_j)-s_i^{k,t})_+).
\end{align*}
Assume $\max (|\pi^{k,t}|,|s_i^{k,t}|, L(\w^{k,t}; \x_i, \x_j), \|\nabla L(\w^{k,t}; \x_i, \x_j)\|)\leq C$. Then, we have
\begin{align*}
&\E[\tilde{\bg}_\w^{k,t}] = \nabla_1 F(\w^{k,t}, \s^{k,t},\pi^{k,t},\u^{k,t}), \quad \E[\|\tilde{\bg}_\w^{k,t}\|^2]\leq \frac{C}{\alpha^2\beta^2},\\
&\E[\tilde{\bg}_\s^{k,t}] = \nabla_2 F(\w^{k,t}, \s^{k,t},\pi^{k,t},\u^{k,t}),  \quad \E[\|\tilde{\bg}_\s^{k,t}\|^2]\leq \frac{1}{\alpha^2}(1+\frac{1}{\beta})^2,\\
&\E[\tilde{\bg}_\pi^{k,t}] =\nabla_3 F(\w^{k,t}, \s^{k,t},\pi^{k,t},\u^{k,t}),  \quad \E[\|\tilde{\bg}_\pi^{k,t}\|^2]\leq (1+\frac{1}{\alpha})^2,\\
&\E[\tilde{\bg}_\u^{k,t}] =\nabla_4F(\w^{k,t}, \s^{k,t},\pi^{k,t},\u^{k,t}), \quad \E[\|\tilde{\bg}_\u^{k,t}\|^2]\leq \frac{4C^2}{\alpha^2}(1+\frac{1}{\beta})^2.
\end{align*}

\begin{algorithm}[t]{\hspace*{-0.5in}}
    \centering
    \caption{SOTA}
    \label{alg:SOTA}
    \begin{algorithmic}[1]  
    \FOR{$k=0,\ldots,K-1$}
    \IF {$k=0$}
    \STATE Set $s_i^{0,0}=0, \pi^{0,0}=0, u_i^{0,0}=1$ and initialize $\w^{0,1}$
    \ELSE
    \STATE Let $\w^{k,0}=\w^{k-1}, \s^{k,0}=\s^{k-1}, \pi^{k,0}=\pi^{k-1}, \u^{k,0} = \u^{k-1}$
    \ENDIF
    \FOR {$t = 0,\ldots, T_k-1$}
    \STATE Sample $\B_+^{k,t}\subset\S_+$ and $\B_-^{k,t}\subset\S_-$
    \STATE Obtain $\w^{k,t+1}, \s^{k,t+1}, \pi^{k,t+1}, \u^{k,t+1}$ according to~(\ref{eqn:update})
    \ENDFOR
    \STATE Let  $\w^k, \s^k, \pi^k, \u^k$ be the average of $\w^{k,t}, \s^{k,t}, \pi^{k,t}, \u^{k,t}$, respectively
    \ENDFOR
    \end{algorithmic}
\end{algorithm}

Let $\v\coloneqq (\w, \s, \pi)$.  We first show that $F(\v, \u)$ is weakly convex in terms of $\v$ for any $\u$. 
\begin{lemma}
Under Assumption~\ref{ass:2}, then $F(\v, \u)$ is $\rho/(\alpha\beta)$-weakly convex in terms of $\v$ for any $\u$, where $\rho$ is the smoothness constant of $L(\w; \x_i, \x_j)$ w.r.t. $\w$. 
\end{lemma}
\begin{proof}
Following similar analysis of Lemma~\ref{lem:weak}, we can show that $F(\v, \u)+ \frac{\rho}{2\alpha\beta}\|\v\|^2 =  F(\v, \u)+ \frac{1 }{n_+\alpha}\sum_iu_i(\frac{\rho}{2\beta}\|\w\|^2 +\frac{\rho}{2\beta}|s_i|^2+\frac{\rho}{2\beta}|\pi|^2) + \frac{\rho }{2n_+\alpha\beta}\sum_i(1-u_i)(\|\w\|^2 +|s_i|^2+|\pi|^2) + \frac{(n_+-1)}{2n_+\alpha\beta}\|\s\|^2$ is jointly convex in terms $\w, \s, \pi$ for any $\u\in[0,1]$. Then $F(\v, \u)$ is $\rho'=\frac{\rho}{\alpha\beta}$-weakly convex in terms of $\v$ for any $\u$. 
\end{proof}

\begin{lemma}\label{lem:3p_ineq}
Consider the proximal gradient update
\begin{align*}
  \x_{t+1}=  \arg\min \x^{\top}G_t + \frac{1}{2\eta}\|\x- \x_t\|^2 + g(\x), 
\end{align*}
we have
\begin{align*}
    &(\x_{t}-\x)^{\top}G_t +g(\x_t) - g(\x) + \frac{1}{2\gamma}\|\x-\x_{t+1}\|^2\\
     &\leq \eta\|G_t\|^2 + \frac{1}{2\eta}(\|\x- \x_t\|^2 -\|\x- \x_{t+1}\|^2)+ g(\x_{t}) - g(\x_{t+1})  - \frac{1}{4\eta}\|\x_{t+1}- \x_t\|^2.
\end{align*}
\end{lemma}
\begin{proof}
Due to the update of $\x_{t+1}$,  we have
\begin{align*}
    &\x_{t+1}^{\top}G_t + \frac{1}{2\eta}\|\x_{t+1}- \x_t\|^2 + g(\x_{t+1}) + (\frac{1}{2\eta}+\frac{1}{2\gamma})\|\x-\x_{t+1}\|^2 \leq \x^{\top}G_t + \frac{1}{2\eta}\|\x- \x_t\|^2  + g(\x).
\end{align*}
As a result, we have
\begin{align*}
    &(\x_{t}-\x)^{\top}G_t +g(\x_t) - g(\x) + \frac{1}{2\gamma}\|\x-\x_{t+1}\|^2\\
    &\leq (\x_t - \x_{t+1})^{\top}G_t + \frac{1}{2\eta}(\|\x- \x_t\|^2 -\|\x- \x_{t+1}\|^2)+ g(\x_{t}) - g(\x_{t+1})  - \frac{1}{2\eta}\|\x_{t+1}- \x_t\|^2\\
     &\leq \eta\|G_t\|^2 + \frac{1}{2\eta}(\|\x- \x_t\|^2 -\|\x- \x_{t+1}\|^2)+ g(\x_{t}) - g(\x_{t+1})  - \frac{1}{4\eta}\|\x_{t+1}- \x_t\|^2.
\end{align*}
\end{proof}

\begin{thm}
Assume there exists $C>0$ such that $\max (|s'_t|,|s_{t,i}|, L(\w_t; \x_i, \x_j), \|\nabla L(\w_t; \x_i, \x_j)\|)\leq C$ at every stage. Let  $1/\gamma\geq \rho, \eta^k_1=\eta^k_2=\eta^k_3=\eta^k_4=\propto 1/k, T_k \propto n_+ k^2$. SOTA ensures that after $T=O(n_+/\epsilon^6)$ iterations we can find an $\epsilon$-nearly stationary solution for $\min_{\w, \s, s'}F(\w, \s, s')$. 
\end{thm}
\begin{proof}
Let $R_1(\w) \coloneqq \frac{1}{2\gamma}\|\w - \w^{k,0}\|^2$, $R_2(\s) \coloneqq \frac{1}{2\gamma}\|\s - \s^{k,0}\|^2$, and $R_3(\pi) =\frac{1}{2\gamma}|\pi - \pi^{k,0}|^2$.   
Apply Lemma~\ref{lem:3p_ineq} to $\w^{k,t+1}$.
\begin{align*}
    &(\w^{k,t}-\w)^{\top}\tilde{g}_\w^{k,t} +R_1(\w^t) - R_1(\w) \leq \eta_1\|\tilde{g}_\w^{k,t}\|^2 + \frac{1}{2\eta_1}(\|\w- \w^{k,t}\|^2 -\|\w- \w^{k,t+1}\|^2)+ g_1(\w^{k,t}) - g_1(\w^{k,t+1}). 
\end{align*}
Take expectation on both sides conditioned on the randomness that occurred before the $t$-th iteration in the $k$-th stage.
\begin{align*}
 & (\w^{k,t}-\w)^{\top}\nabla_\w F(\v^{k,t},\u^{k,t}) + R_1(\w^{k,t}) - R_1(\w) \\
 & \leq \eta_1 \E_t[\|\tilde{g}_\w^{k,t}\|^2] + \frac{1}{2\eta_1}(\|\w- \w^{k,t}\|^2 -\E_t[\|\w- \w^{k,t+1}\|^2])+ R_1(\w^{k,t}) - \E_t[R_1(\w^{k,t+1})].
\end{align*}
Similarly, apply Lemma~\ref{lem:3p_ineq} to $\s_i^{k,t+1}$ and $\pi^{k,t+1}$ and take the conditional expectations.
\begin{align*}
& (\s^{k,t} - \s)^\top \nabla_\s F(\v^{k,t},\u^{k,t}) + R_2(\s^{k,t}) - R_2(\s)\\
& \quad\quad \leq \eta_2 \E_t[\|\tilde{g}_\s^{k,t}\|_2^2] + \frac{1}{2\eta_2} (\|\s - \s^{k,t}\|^2 - \E_t[\|\s - \s^{k,t+1}\|^2]) + R_2(\s^{k,t}) - \E_t[R_2(\s^{k,t+1})],\\
& (\pi^{k,t} - \pi) \nabla_\pi F(\v^{k,t},\u^{k,t}) + R_3(\pi^{k,t}) - R_3(\pi) \\
& \quad\quad \leq \eta_3 \E_t[|\tilde{g}_\pi^{k,t}|_2^2] + \frac{1}{2\eta_2} (|\pi - \pi^{k,t}|^2 - \E_t[|\pi - \pi^{k,t+1}|^2]) + R_3(\pi^{k,t}) - \E_t[R_3(\pi^{k,t+1})].
\end{align*}
Note that $F_k(\v^{k,t}, \u^{k,t}) - F_k(\v, \u^{k,t}) = F(\v^{k,t}, \u^{k,t}) - F(\v, \u^{k,t}) + \frac{1}{2\gamma} \|\v^{k,t} - \v^{k,0}\|_2^2 - \frac{1}{2\gamma} \|\v - \v^{k,0}\|_2^2$. If $\frac{1}{\gamma} \geq \rho' = \frac{\rho}{\alpha\beta}$, we have $F_k(\v,\u)$ is convex w.r.t. $\v$ such that $F_k(\v^{k,t}, \u^{k,t}) - F_k(\v, \u^{k,t}) \leq (\v^{k,t} - \v)^\top \nabla_\v F(\v^{k,t},\u^{k,t}) + (\v^{k,t} - \v)^\top\nabla_\v R(\v^{k,t},\u^{k,t})$. where $R=(R_1,R_2,R_3)^\top$. Note that $(\v^{k,t} - \v)^\top\nabla_\v R(\v^{k,t},\u^{k,t}) = \frac{1}{\gamma}(\v^{k,t} - \v)^\top(\v^{k,t} - \v^{k,0}) \geq R(\v^{k,t}) - R(\v)$. Adding the above inequalities for $\w, \s, \pi$ together we have
\begin{align*}
    \E[\sum_{t=0}^{T_k - 1}(F_k(\v^{k,t}, \u^{k,t}) - F_k(\v, \u^{k,t}))] & \leq \eta_1C_1^2T + \eta_2 C_2^2T + \eta_3C_3^2T + \frac{1}{2\eta_1}\|\w- \w^{k,0}\|^2 + \frac{1}{2\eta_2}\|\s- \s^{k,0}\|^2 + \frac{1}{2\eta_3}|\pi- \pi^{k,0}|^2,
\end{align*}
where $C_1^2\coloneqq \frac{C}{\alpha^2\beta^2}$, $C_2^2 = \frac{1}{\alpha^2}(1+1/\beta)^2$, $C_3^2 \coloneqq (1+1/\alpha)^2$. Applying the same analysis to the update of $\u^{k,t}$, we have
\begin{align*}
&(\u - \u^{k,t})^{\top} \tilde{\bg}_\u^{k,t} \leq\eta_4\|\tilde{\bg}_\u^{k,t}\|^2 + \frac{1}{2\eta_4}(\|\u- \u^{k,t}\|^2 -\|\u- \u^{k,t+1}\|^2)
\end{align*}
We do not assume $\u$ is independent of the randomness in the updates of our algorithm. As a result, 
\begin{align*}
& (\u - \u^{k,t})^{\top}\nabla_4  F(\v_t, \u_t) =\eta_4\|\tilde{\bg}_\u^{k,t}\|^2 + \frac{1}{2\eta_4}(\|\u- \u^{k,t}\|^2 -\|\u- \u^{k,t+1}\|^2) + (\u-\u^{k,t})^{\top}(\nabla_4  F(\v_t, \u_t)-\tilde{\bg}_\u^{k,t}).
\end{align*}
Following previous analysis (e.g., Proposition A.1 in~\cite{rafique2018non}), for an auxiliary sequence $\{\tilde{\u}^{k,t}\}$ we have
\begin{align*}
&\E[(\u-\u^{k,t})^{\top}(\nabla_4 F(\v^{k,t}, \u^{k,t})-\tilde{\bg}_\u^{k,t})] \\
& \leq \eta_4\E[\|\nabla_4 F(\v^{k,t}, \u^{k,t})-\tilde{\bg}_\u^{k,t}\|^2] + \frac{1}{2\eta_4}(\E[\|\u- \tilde\u^{k,t}\|^2] -\E[\|\u- \tilde\u^{k,t+1}\|^2]).
\end{align*}
Hence for any $\u$ and $\tilde{\u}^{k,0} = \u^{k,0}$, $C_4^2 \coloneqq \frac{4C^2}{\alpha^2}(1+1/\beta)^2$, we have
\begin{align*}
\E[F_k(\v^{k,t},\u) - F_k(\v^{k,t},\u^{k,t})] \leq \E[\sum_{t=0}^{}(\u - \u^{k,t})^{\top}\nabla_4  F(\v_t, \u_t)] & \leq 2\eta_4 \E[\|\tilde{g}_\u^{k,t}\|^2] + \frac{1}{\eta_4}\|\u - \u^{k,0}\|^2 \\
& \leq 2\eta C_4^2 + \frac{1}{\eta_4}\|\u - \u^{k,0}\|^2.
\end{align*}
According to the initialization in Algorithm~\ref{alg:SOTA}, for any $\v = (\w,\s,\pi)$ and $\u$ we have
\begin{align*}
    \E[\sum_{t=0}^{T_k - 1}F_k(\v^{k,t},\u) - F_k(\v, \u^{k,t})] & \leq  \eta_1 C_1^2T + \eta_2 C_2^2T + \eta_3 C_3^2T + 2 \eta_4 C_4^2 \\
    &\quad\quad + \frac{1}{2\eta_1}\|\w^{k,0} - \w\|^2 + \frac{1}{2\eta_2}\|\s\|^2 +\frac{1}{2\eta_3}|\pi|^2 +  \frac{1}{\eta_4}\|\u - 1\|^2.
\end{align*}
Note that $\|\s\|^2 \leq n_+ C^2$ and $\|\u - 1\|^2 \leq n_+$. It remains to apply the analysis in~\citep[Theorem 4.1]{rafique2018non}  to derive the convergence for the Moreau envelope of $F(\w, \s, \pi)$ with a complexity in the order of $O(n_+/\epsilon^6)$ by setting $\eta_k\propto 1/k$ and $T_k\propto n_+ k^2$ and the total number of stages $K=O(1/\epsilon^2)$. 
\end{proof}

\end{document}